\documentclass{article} % For LaTeX2e
\usepackage[numbers]{natbib}
\usepackage[letterpaper,
    left=1in,
    right=1in,
    top=1in,
    bottom=1in
]{geometry}

\usepackage{times}
\usepackage{algorithm}
\usepackage{algpseudocode}
\usepackage{booktabs}
\usepackage{tabularx}
\usepackage{array}
\usepackage{multirow}
\usepackage[export]{adjustbox}

\usepackage{amsmath,amsfonts,bm}

\def\eqref#1{equation~\ref{#1}}
\def\1{\bm{1}}

\DeclareMathAlphabet{\mathsfit}{\encodingdefault}{\sfdefault}{m}{sl}
\SetMathAlphabet{\mathsfit}{bold}{\encodingdefault}{\sfdefault}{bx}{n}

\DeclareMathOperator*{\argmin}{arg\,min}

\usepackage{hyperref}
\usepackage{url}
\usepackage{amsthm}
\usepackage{graphicx}
\usepackage{subcaption}
\usepackage{array}
\allowdisplaybreaks
\title{Denoising as Projection: Constrained Optimization with Gradient-Guided Diffusion}

\author{
Runyu Zhang$^{1*}$, Jiawei Zhang$^{2*}$, Gioele Zardini$^{1}$,
Saurabh Amin$^{1}$, and Asuman Ozdaglar$^{1}$\\[0.4em]
{\small $^{1}$MIT,
\texttt{\{runyuzha,gzardini,amins,asuman\}@mit.edu}}\\
{\small $^{2}$University of Wisconsin--Madison,
\texttt{jzhang2924@wisc.edu}}
}
\usepackage{xcolor}
\definecolor{gznavy}{RGB}{0,0,128}

\newcommand{\cX}{\mathcal{X}}
\newcommand{\cM}{\mathcal{M}}

\newcommand{\cN}{\mathcal{N}}

\newcommand{\bE}{\mathbb{E}}
\newcommand{\bR}{\mathbb{R}}
\renewcommand{\r}{r}
\newcommand{\Proj}{\mathrm{Proj}}

\newcommand{\range}{\mathrm{range}}
\newcommand{\Aperp}{A^{\perp}}
\newcommand{\teta}{\widetilde{\eta}}

\newcommand{\cU}{\mathcal{U}}
\newcommand{\cA}{\mathcal{A}}
\newcommand{\cT}{\mathcal{T}}

\newcommand{\dist}{\mathrm{dist}}

\newcommand{\vol}{\mathrm{vol}}
\newcommand{\diam}{\mathrm{diam}}

\newcommand{\grad}{\mathrm{grad}}

\newtheorem{lemma}{Lemma}
\newtheorem{assump}{Assumption}
\newtheorem{theorem}{Theorem}
\begin{document}

\maketitle
\begingroup
\renewcommand{\thefootnote}{*}
\footnotetext{Corresponding authors.}
\endgroup
\begin{abstract}
Diffusion models are increasingly used not only for sampling from learned data distributions, but also for generating samples that optimize task-specific objectives. A common approach is to guide the reverse diffusion process using gradients of an external objective. However, when the data distribution is supported on a structured feasible set, such as a manifold or a constraint set, gradient guidance can move samples away from the learned data geometry. In this paper, we study a simple projected-gradient-guided diffusion update based on the observation that the Stein denoising operator can act as an approximate projection onto the data geometry. The proposed update incorporates the objective gradient inside the denoising step, yielding an inference-time method that uses only a pretrained denoiser and gradient evaluations. We analyze this update as an inexact projected-gradient method for constrained optimization over learned feasible geometries. Our theory covers three settings: linear manifolds, compact convex feasible sets, and compact Riemannian submanifolds. In all these settings, we prove descent and finite-time convergence guarantees. Numerical experiments support the theoretical interpretation and illustrate how the proposed update balances objective descent with preservation of the learned geometry.
\end{abstract}

\vspace{-15pt}
\section{Introduction}
\vspace{-5pt}
Diffusion models achieve state-of-the-art performance in image synthesis, inverse problems, scientific design, and decision-making~\cite{sohl2015deep,ho2020denoising,song2020score,song2020denoising}. Beyond realistic sampling, many applications require samples that optimize a task-specific objective while remaining consistent with a learned distribution, including inverse problems \cite{chung2022diffusion}, trajectory optimization \cite{janner2022planning}, molecular generation \cite{wu2022diffusion}, and control \cite{chi2025diffusion}.

A common strategy is to modify reverse diffusion with auxiliary information, as in classifier and classifier-free guidance~\cite{dhariwal2021diffusion,ho2022classifierfree}. Given a differentiable objective $f:\bR^n\to\bR$, its gradient can instead guide the reverse dynamics. This suggests viewing guided diffusion as solving
    $\min_{x\in\cX} f(x),$
where the support or high-density region of a learned distribution implicitly represents the feasible set $\cX\subseteq\bR^n$. The diffusion model then provides a prior over feasible designs, trajectories, or decisions, while $f$ specifies task performance. Although gradient guidance needs only inference-time evaluations of $\nabla f$~\cite{guo2024gradient}, directly injecting the gradient may disrupt the learned feasibility structure.

The difficulty is geometric. The set $\cX$ may be a low-dimensional manifold, a closed convex set, a set of physically admissible designs, or a family of dynamically feasible, collision-free trajectories. A gradient direction need not be tangent or feasible, so naive guidance may lower the objective while degrading sample quality, physical validity, or constraint satisfaction. This is especially consequential in diffusion-based robotic planning: learned trajectory distributions encode useful behavioral and dynamical priors, yet objective guidance can produce unsafe or infeasible trajectories without an additional feasibility mechanism~\cite{xiao2025safediffuser,shaoul2025multi,liang2025simultaneous,liang2026discrete}.

Existing approaches often add sampling, computation, or modeling complexity. For example, \citet{kharitenko2025landing} computes score Jacobians to construct geometric primitives, while \citet{guo2024gradient} uses repeated batch generation and objective linearization. Related denoiser-as-projection analysis proves recovery for linear inverse problems with quadratic objectives under an RIP-type condition, rather than general constrained optimization~\cite{leong2025recovery}. Manifold-aware guidance~\cite{chung2022improving,he2024manifold,chung2025cfg++,zirvi2025diffstategrad}, constrained diffusion~\cite{christopher2024constrained,li2025aligning,khalafi2024constrained,hadou2026constrained}, and reward alignment~\cite{clark2024direct,deng2024prdp,potaptchik2025tilt,pachebat2025iterative,kim2025testtime,azangulov2025adaptive} may instead require fine-tuning, repeated sampling, or auxiliary models. Appendix~\ref{app:geometry-preserving-guidance} provides detailed comparisons.

This raises a natural question: \emph{Can an update using one gradient evaluation and one pretrained Stein denoising call per step provide optimization guarantees without outer-loop generation, score Jacobians, or model fine-tuning?}

We answer affirmatively by applying the denoiser after the gradient update. The gradient step promotes objective descent, while denoising returns the iterate toward the feasible geometry. When the denoiser approximates the true projection, the update behaves as an inexact projected-gradient method.
\vspace{-5pt}
\paragraph{Contributions.}
First, we introduce \emph{Denoising-Corrected Gradient} (DCG) Guidance, which applies an objective-gradient step followed by one pretrained-denoiser call at each reverse step. DCG follows a single trajectory and requires no repeated batch generation, score-model adaptation~\cite{guo2024gradient}, or denoiser differentiation~\cite{kharitenko2025landing}.

Second, we establish quantitative bounds between the Stein posterior-mean denoiser and the relevant projection for linear-Gaussian distributions, compact convex sets, and compact smooth submanifolds. These bounds identify the denoiser as DCG's geometric correction and the reverse dynamics as a time-varying inexact projected-gradient method.

% Required packages:
% \usepackage{booktabs}
% \usepackage{multirow}
% \usepackage{array}
% \usepackage{amssymb}

% Required packages:
% \usepackage{booktabs}
% \usepackage{multirow}
% \usepackage{array}

\begin{table*}[t]
\centering
\caption{Comparison of theoretical scope and inference-time requirements.
Sublinear rates are best-iterate guarantees; constants and finite-noise or
denoising-error terms are omitted.}
\label{tab:theory-comparison}
\vspace{-5pt}
\footnotesize
\setlength{\tabcolsep}{2pt}
\renewcommand{\arraystretch}{0.5}

\begin{tabular}{@{}
>{\centering\arraybackslash}m{0.105\textwidth}
>{\centering\arraybackslash}m{0.135\textwidth}
>{\centering\arraybackslash}m{0.220\textwidth}
>{\centering\arraybackslash}m{0.125\textwidth}
>{\centering\arraybackslash}m{0.165\textwidth}
>{\centering\arraybackslash}m{0.170\textwidth}
@{}}
\toprule
\textbf{Method}
&
\textbf{Geometry}
&
\textbf{Guarantee}
&
\textbf{Rate}
&
\textbf{Score requirement}
&
\textbf{Sample requirement}
\\
\midrule

\cite{guo2024gradient}
&
Linear--Gaussian
&
Regularized optimality of the distributional mean
&
Geometric
&
Jacobian of score function
&
Multiple reverse-diffusion trajectories\\
\midrule

\cite{kharitenko2025landing}
&
Compact $C^3$ Riemannian manifold
&
Riemannian stationarity
&
$O(T^{-1/2})$
&
Jacobian of score function
&
No auxiliary inference-time samples
\\
\midrule

\multirow[c]{4}{0.105\textwidth}{%
\centering\arraybackslash\textbf{DCG (ours)}}
&
Linear--Gaussian
&
Global optimality; strongly convex $f$
&
Geometric
&
\multirow[c]{4}{0.165\textwidth}{%
\centering\arraybackslash
Don't require Jacobian}
&
\multirow[c]{4}{0.170\textwidth}{%
\centering\arraybackslash
Single reverse trajectory; no auxiliary samples}
\\

\cmidrule(lr){2-4}

&
\multirow[c]{2}{0.135\textwidth}{%
\centering\arraybackslash
Compact convex set}
&
Global optimality; strongly convex $f$
&
Geometric
&
&
\\

\cmidrule(lr){3-4}

&
&
Projected stationarity; smooth nonconvex $f$
&
$O(T^{-1/2})$
&
&
\\

\cmidrule(lr){2-4}

&
Compact $C^2$ Riemannian manifold
&
Riemannian stationarity; smooth nonconvex $f$
&
$O(T^{-1/2})$
&
&
\\

\bottomrule
\end{tabular}
\vspace{-5pt}
\end{table*}

Third, we prove finite-time guarantees tailored to each geometry and objective class. For linear-Gaussian data and strongly convex objectives, the projected objective error and normal component contract geometrically. For compact convex sets, we prove geometric convergence for strongly convex objectives and an $O(1/\sqrt{T})$ nonconvex best-iterate rate. For compact smooth submanifolds, we obtain the same best-iterate rate locally. This unified approximate-projection framework covers linear subspaces, possibly nonsmooth convex sets, and smooth manifolds. Table~\ref{tab:theory-comparison} compares our theory with the closest optimization-oriented approaches. % While linear-subspace and smooth-manifold models have received substantial attention in diffusion-based optimization, the corresponding optimization theory for general compact convex supports appears to be largely unexplored.
 %Our convex-set analysis therefore constitutes a distinct contribution, extending diffusion-based constrained optimization to feasible regions that may contain boundaries, corners, and other nonsmooth geometric features.

Finally, experiments on synthetic constrained optimization, motion planning, and D4RL diffusion-based reinforcement learning benchmarks (Appendix~\ref{app:numerics}) show that DCG favorably balances objective value and constraint satisfaction, outperforming standard gradient guidance and matching or improving related optimization-oriented diffusion baselines.
% \begin{table}[t]
% \centering
% \caption{Comparison of theoretical scope and guarantees.}
% \vspace{-5pt}
% \label{tab:theory-comparison}
% \small
% \renewcommand{\arraystretch}{1.2}
% \setlength{\tabcolsep}{2pt}
% \begin{tabular}{@{}
% >{\centering\arraybackslash}m{0.1\linewidth}
% >{\centering\arraybackslash}m{0.28\linewidth}
% >{\centering\arraybackslash}m{0.29\linewidth}
% >{\centering\arraybackslash}m{0.29\linewidth}
% @{}}
% \toprule
% & \textbf{Linear--Gaussian}
% & \textbf{Compact convex}
% & \textbf{Riemannian manifold} \\
% \midrule

% \cite{guo2024gradient}
% & Geometric convergence
% to a regularized solution for distributional mean
% & ---
% & ---
% \\

% \cite{kharitenko2025landing}
% & ---
% & ---
% & $O(\sqrt{T})$; requires
% $\nabla^2\log p_t(x)$ and $C^3$ manifold
% \\

% \textbf{DCG (ours)}
% & Geometric convergence for strongly convex objectives
% & Strongly convex: Geometric; Nonconvex Objective: $O(\sqrt{T})$
% & $O(\sqrt{T})$ without denoiser Jacobians for $C^2$ manifold
% \\

% \bottomrule
% \end{tabular}
% \end{table}
\vspace{-5pt}
\paragraph{Limitations.}
Our analysis assumes an exact Stein posterior-mean denoiser and learned data support that faithfully represents the feasible set. Projection-error terms can capture learned-score error, but we leave a full treatment of misspecification to future work. We also restrict the theory to deterministic DDIM dynamics under variance-exploding diffusion with an exponentially decaying noise schedule. Other samplers, parameterizations, and schedules require tracking their time-varying relaxation and denoising errors.
% \gz{Before rewriting myself I want to share some comments about the intro:
% \begin{enumerate}
%     \item We are missing a big robotics positioning opportunity: the intro cites Janner et al.\ and Chi et al., but the strongest current application pull for exactly this method is diffusion-based motion planning where feasibility is the pain point, e.g., Motion Planning Diffusion (Carvalho et al., IROS 2023), SafeDiffuser (CBF-based), MMD: multi-robot motion planning with diffusion + constraint-based search (Shaoul et al., ICLR 2025), and projected diffusion for continuous-space MAPF and simultaneous MRMP (Liang et al., 2024/2025, incl.\ Discrete-Guided Diffusion, 2025). [Probably by now you are aware of more methods]
%     Citing these would both modernize the motivation and set up the comparison in App.\ A: those methods need an \emph{explicit} projection oracle or search layer, whereas DCG's projection is the pretrained denoiser itself. This is precisely the gap the paper fills. We are missing an opportunity not saying it in the intro, and only in the appendix.
% \end{enumerate}}

\vspace{-5pt}
\section{Preliminaries}\label{sec:preliminaries}
\vspace{-5pt}
\paragraph{Constrained optimization over learned data geometry.}
We consider the constrained optimization problem
\vspace{-10pt}
\begin{align}\label{eq:constrained-optimization}
   \textstyle \min_{x\in\cX} f(x),
\end{align}
where $f:\bR^n\to\bR$ is differentiable and $\cX\subseteq\bR^n$ is the feasible geometry represented by the data distribution. The set $\cX$ may be a linear subspace, compact convex set, or smooth embedded submanifold. Given a diffusion model trained on samples supported on or near $\cX$, we use the pretrained model to approximately solve \eqref{eq:constrained-optimization} at inference time.

If the Euclidean projector $\pi_{\cX}$ were available, one could use projected gradient descent:~$x^{+}
    =
    \pi_{\cX}\bigl(x-\eta\nabla f(x)\bigr)$.
Here, however, the learned distribution represents $\cX$ only implicitly, so $\pi_{\cX}$ may be unavailable. Our premise is that diffusion denoising provides a learned approximate projection.
\vspace{-5pt}
\paragraph{Gaussian diffusion and Stein denoising.}
We use a variance-exploding forward process %Our experiments may therefore be implemented using standard
%variance-preserving parameterizations and DDIM samplers, while the theoretical
%analysis is expressed in the normalized coordinates of \eqref{eq:xtx0}.}
\begin{align*}
   \textstyle x_{t+1}
    =
    x_t+\sqrt{\alpha_t}w_t,
    ~~
    w_t\sim\cN(0,I_n),
\end{align*}
where $\alpha_t>0$ is the noise variance added at step $t$. Equivalently, \footnote{\textbf{Relation to variance-preserving diffusion:}
The additive form \eqref{eq:xtx0} simplifies the analysis. A standard variance-preserving process is
    $\widetilde x_t
    =
    \sqrt{\overline\alpha_t}x_0
    +
    \sqrt{1-\overline\alpha_t}\,w,
    ~~
    w\sim\cN(0,I_n).$
Rescaling by
   $ x_t
    :=
    \frac{\widetilde x_t}{\sqrt{\overline\alpha_t}},$
we obtain
    $x_t
    =
    x_0+\sigma_t w,
    ~~
    \sigma_t^2
    =
    \frac{1-\overline\alpha_t}{\overline\alpha_t}.$
therefore converts variance-preserving diffusion into the additive-noise notation used below. }
\begin{align}\label{eq:xtx0}
  \textstyle  x_t
    =
    x_0+\sigma_t\overline w_t,
    ~~
    \overline w_t\sim\cN(0,I_n),
    \qquad
    \sigma_t^2
    :=
    \sum_{s=0}^{t-1}\alpha_s.
\end{align}
Let $P_t$ be the marginal distribution of $x_t$, with score $\nabla\log P_t(x)$. The Tweedie--Miyasawa identity \cite{robbins1992empirical,miyasawa1961empirical} gives~
$ \mathbb E[x_0\mid x_t=x] = x+\sigma_t^2\nabla\log P_t(x).$~
We therefore define the Stein posterior-mean denoiser
\vspace{-10pt}
\begin{align}\label{eq:stein-denoiser-continuous}
    \quad\widehat\pi_{\sigma}(x)
    :=
    x+\sigma^2\nabla\log P_{\sigma}(x),
\end{align}
where $P_{\sigma}$ is the data distribution convolved with $\cN(0,\sigma^2I_n)$. At step $t$, write
\begin{align}\label{eq:pi_t}
    \widehat\pi_t(x)
    :=
    x+\sigma_t^2\nabla\log P_t(x).
\end{align}
In particular,
   $ \widehat\pi_t(x_t)
    =
    \mathbb E[x_0\mid x_t].$
In standard implementations, $\widehat\pi_t(x_t)$ is the posterior-mean or clean-sample prediction $\widehat x_0(x_t,t)$. Evaluating $\widehat\pi_t$ therefore requires no model or optimization subroutine beyond a standard sampler's denoising operation.

When clean data lie on a structured set $\cX$, $\widehat\pi_t$ moves noisy inputs toward that geometry. We quantify this behavior by comparing it with Euclidean projection onto linear or convex sets and nearest-point projection onto smooth submanifolds.

\paragraph{Deterministic backward diffusion.}

Following DDIM~\cite{song2020denoising}, we consider the deterministic reverse update
\vspace{-10pt}
\begin{align}
 x_{t-1}
&\textstyle=
\left(1-\frac{\sigma_{t-1}}{\sigma_t}\right)\mathbb E[x_0\mid x_t]
+
\frac{\sigma_{t-1}}{\sigma_t}x_t \notag\\
&\textstyle =\frac{\sigma_{t-1}}{\sigma_t}x_t
+
\left(1-\frac{\sigma_{t-1}}{\sigma_t}\right)\widehat\pi_t(x_t).\label{eq:ddim-backward}
\end{align}
Thus, each reverse step interpolates between the noisy sample $x_t$
and its denoised estimate $\widehat\pi_t(x_t)$. This update is also
structurally similar to \cite{delbracio2023inversion}.
\paragraph{Gradient-guided diffusion.}

To bias samples toward a differentiable objective $f:\bR^n\to\bR$, standard gradient guidance adds a descent direction to \eqref{eq:ddim-backward} (cf.~\cite{dhariwal2021diffusion,janner2022planning,chung2022diffusion}):
\begin{align}\label{eq:prelim-naive-guided-ddim}
\textstyle x_{t-1}
=
\frac{\sigma_{t-1}}{\sigma_t}x_t
+
\left(1-\frac{\sigma_{t-1}}{\sigma_t}\right)
\widehat\pi_t(x_t)-\eta_t\nabla f(x_t).
\end{align}
We call \eqref{eq:prelim-naive-guided-ddim} \emph{post-denoising gradient guidance} (PDG) because it applies the gradient correction outside the denoiser. PDG requires only evaluating $\nabla f$ and does not retrain the diffusion model.

If the data lie on a structured feasible set, however, the gradient may violate the learned geometry and cause off-manifold drift. We instead apply the Stein denoiser, viewed as an approximate projection, to a gradient-updated point. The next section introduces this \emph{denoising-corrected gradient (DCG) guidance} update.

\vspace{-5pt}
\section{Denoising-Corrected Gradient Guidance}
\vspace{-5pt}
% \gz{Agreed with the TODO. 
% Also note that the section introduces the update as \eqref{eq:guided-diffusion} but the analysis sections all reference \eqref{eq:projected-gradient-guidance-relaxed}, which is only defined at the very end of this section \emph{after} the forward reference in the paragraph above Algorithm~1 (``we analyze \eqref{eq:projected-gradient-guidance-relaxed}\dots'' appears before the equation exists). 
% I would suggest to reorder: state the general update, then immediately specialize to the geometric schedule, then give Algorithm~1 in the specialized form so pseudo-code and analyzed update match exactly.}

The Stein denoiser $\widehat\pi_t(x)=x+\sigma_t^2 \nabla\log P_t(x)$ estimates the clean sample's posterior mean. When the data lie on or near a constraint set $\cX$ or manifold $\cM$, it maps noisy points toward the feasible geometry and thus acts as a learned approximate projection~\cite{kharitenko2025landing,permenter2023interpreting,chung2022improving}.

Lemmas~\ref{lemma:pi-t}, \ref{lem:stein-convex-projection}, and \ref{lem:stein-projection-c2} formalize this interpretation for linear manifolds, compact convex sets, and compact $C^2$ embedded submanifolds, respectively. It suggests reversing the order used by PDG: apply the gradient step before denoising. Specifically, we consider
\begin{align}\label{eq:guided-diffusion}
\textstyle x_{t-1}
=
\frac{\sigma_{t-1}}{\sigma_t}x_t
+
\left(1-\frac{\sigma_{t-1}}{\sigma_t}\right)
\widehat\pi_t\bigl(x_t-\eta_t\nabla f(x_t)\bigr),
\end{align}
where $\eta_t>0$ is the guidance stepsize. 
% Equivalently, defining
% \begin{align*}
% \epsilon_t:=1-\frac{\sigma_{t-1}}{\sigma_t},
% \end{align*}
% we can write
% \begin{align*}
% x_{t-1}
% =
% (1-\epsilon_t)x_t
% +
% \epsilon_t\widehat\pi_t\bigl(x_t-\eta_t\nabla f(x_t)\bigr).
% \end{align*}

Unlike PDG, which shifts the denoised prediction, \eqref{eq:guided-diffusion} takes a gradient step before applying the Stein denoiser. When $\widehat\pi_t$ approximates the true projection, the update resembles relaxed projected gradient:
\vspace{-10pt}
\begin{align*}
x_{t-1}
\approx
(1-\epsilon_t)x_t
+
\epsilon_t\pi_\cX\bigl(x_t-\eta_t\nabla f(x_t)\bigr).
\end{align*}
This correspondence underlies our optimization analysis.

Each reverse step requires one evaluation of $\nabla f(x_t)$ and one application of the pretrained denoiser. DCG neither differentiates through $\widehat\pi_t$, evaluates $\nabla_x f(\widehat\pi_t(x))$, solves an auxiliary constrained problem, nor retrains the diffusion model. The following sections prove convergence when the Stein denoiser accurately approximates projection.

\begin{algorithm}[t]
\caption{Denoising-Corrected Gradient Guidance (DCG Guidance)}
\label{alg:projected-gradient-guidance}
\begin{algorithmic}[1]
\Require Initial noisy sample $x_T$, noise schedule $\{\sigma_t\}_{t=0}^{T}$, stepsizes $\{\eta_t\}_{t=1}^{T}$, denoisers $\{\widehat\pi_t\}_{t=1}^{T}$
\For{$t=T,T-1,\ldots,1$}
    \State Compute the gradient $\nabla f(x_t)$
    \State Set $z_t=x_t-\eta_t\nabla f(x_t)$
    \State Denoise/project $z_t$ using $\widehat\pi_t(z_t)$
    \State Update
        ~~$x_{t-1}
        =
        \frac{\sigma_{t-1}}{\sigma_t}x_t
        +
        \left(1-\frac{\sigma_{t-1}}{\sigma_t}\right)
        \widehat\pi_t(z_t)$
\EndFor
\State \Return $x_0$
\end{algorithmic}
\end{algorithm}

% In the following sections, we analyze \eqref{eq:guided-diffusion} as an optimization algorithm for the constrained problem $\min_{x\in\cX} f(x),$ where $\cX$ represents the feasible data geometry learned by the diffusion model. The central question is whether replacing the exact projection onto $\cX$ by the Stein denoising operator $\widehat\pi_t$ still yields meaningful descent and convergence guarantees. We study this question in three settings of increasing geometric complexity: first, when $\cX$ is a linear manifold, where the projection structure is exact and transparent; second, when $\cX$ is a closed convex set, where projected-gradient methods provide the natural optimization benchmark; and third, when $\cX$ is a compact Riemannian submanifold, where the denoiser approximates the nearest-point projection in a tubular neighborhood. Together, these cases show that the proposed gradient-guided diffusion update can be understood as an approximate projected-gradient method driven by the Stein score.

For the analysis, we use the geometric schedule $\sigma_{t-1}/\sigma_t=1-\epsilon$ and constant guidance scale $\eta_t=\eta$, yielding
\begin{align}\label{eq:projected-gradient-guidance-relaxed}
x_{t-1} =
(1-\epsilon)x_t
+
\epsilon\widehat\pi_t\bigl(x_t-\eta\nabla f(x_t)\bigr).
\end{align}
This choice isolates the core mechanism without time-dependent notation; we leave more general schedules and guidance scales to future work.

\paragraph{Comparison with related optimization methods.}
The update of~\citet{guo2024gradient} shares PDG's post-denoising structure but evaluates the gradient at a sampled population mean:~$x_{t-1}^{k+1}
    \! =\!
    (1\!-\!\epsilon_t)x_t^{k+1}
   \! +\!
    \epsilon_t\widehat\pi_t(x_t^{k+1})
   \! -\!
    \eta_t\nabla f(\bar x^k).$~Here,
    $\bar x^k
    \!:=\!
    \frac{1}{B}\sum_{b=1}^{B}x_0^{k,b},$~
where $\{x_0^{k,b}\}_{b=1}^{B}$ are terminal samples from $B$ reverse trajectories at outer iteration $k$. The next batch shares a guidance direction updated after estimating $\bar x^k$ from new samples. DCG instead uses the current state and places the gradient step inside the denoiser, avoiding batch averaging and outer guidance updates.

The method of~\citet{kharitenko2025landing} uses the posterior-mean denoiser and its Jacobian:
 $x^{k+1}
    =
    \widehat\pi_{\sigma}\!\left(
        x^k
        -
        \eta_k
        D\widehat\pi_{\sigma}(x^k)
        \nabla f(x^k)
    \right).$
Here, $D\widehat\pi_{\sigma}(x^k)$ approximates tangent-space projection and $\widehat\pi_{\sigma}$ approximates retraction. DCG uses only $\widehat\pi_t$ and follows a time-varying reverse schedule rather than standalone iterations at a fixed noise level.

The gradient--denoiser composition also resembles plug-and-play proximal gradient~\cite{venkatakrishnan2013plug, sreehari2016plug} and RED~\cite{romano2017little}. Our contribution is to identify the time-varying Stein posterior mean along a reverse trajectory as an approximate projection onto the learned feasible set and derive constrained-optimization guarantees.

% \gz{A few comments:
% \begin{enumerate}
%     \item The geometric schedule $\sigma_{t-1}/\sigma_t = 1-\epsilon$ is exactly what the double-integrator experiment uses ($\sigma$ geometrically spaced), but the convex/unicycle/MuJoCo experiments use a cosine VP schedule. 
%     From what I understand, those experiments are formally outside the analyzed regime. 
%     We should either add a remark that the analysis is schedule-robust (the recursions only need $\sigma_{t-1}\le\sigma_t$ and summability of $\rho^{t-1}E_t$, and this is a very simple generalization you are leaving on the table, since Lemmas 3, 4, 6 already hold with time-varying $\epsilon_t$, only the closed-form geometric sums in Theorems 2--4 are special in this sense), or align the experiments with the theory.
%     \item Algorithm~1 as written uses general $\{\eta_t\}$ and general schedule, which is good, but then the return value $x_0$ is described nowhere. We should state that the algorithm's output guarantee is on $\pi_{\cX}(x_0)$ / $\pi_\cM(x_{t^\star})$, not on $x_0$ itself (cf.\ Theorems 2 and 4), and consider adding an optional final denoise/projection step $x_{\mathrm{out}} = \widehat\pi_1(x_0)$ to Algorithm~1 so the returned iterate itself is (approximately) feasible. At present the theory bounds $\|n_0\| \le (1-\epsilon)^T\|n_T\|$, which is strong, so this is mostly presentational.
% \end{enumerate}
% }

\section{Convergence Analysis: Linear-Gaussian Setting}\label{sec:linear-analysis}

We first analyze the proposed guided diffusion update in a linear setting, where the data distribution is supported on a linear subspace and the Stein denoising operator can be characterized explicitly. 

Let $A\in\bR^{d\times n}$ satisfy $AA^\top=I_d$, so that $\range(A^\top)\subseteq\bR^n$ is a $d$-dimensional linear subspace. We assume that the clean data distribution is a Gaussian supported on this subspace. All proofs are in Appendix~\ref{app:linear}.

\begin{assump}[Linear data manifold]\label{assump:linear}
The clean data distribution satisfies
\begin{align*}
x_0 \sim \cN(0,A^\top\Sigma A),
\end{align*}
where $\Sigma\in\bR^{d\times d}$ is positive definite and $A\in\bR^{d\times n}$ satisfies $AA^\top=I_d$.
\end{assump}

Throughout this section, we use $A^\perp\in\bR^{(n-d)\times n}$ to denote the matrix whose rows form an orthonormal basis of $\ker(A)$.
We also assume that the objective is strongly convex and smooth.

\begin{assump}\label{assump:f-strongly-convex}
The function $f:\bR^n\to\bR$ is $\mu$-strongly convex and $L$-smooth.
\end{assump}

The next lemma gives an explicit expression for the Stein denoising operator in this linear-Gaussian model. It shows that $\widehat\pi_t$ is not exactly the Euclidean projection onto $\range(A^\top)$ at finite noise, but rather a noise-dependent Tikhonov-regularized projection onto the data subspace. As $\sigma_t\to 0$, this operator approaches the Euclidean projection onto $\range(A^\top)$.

\begin{lemma}[Characterization of the Stein denoising operator $\widehat\pi_t$]\label{lemma:pi-t}
Under Assumption~\ref{assump:linear}, the Stein denoising operator $\widehat\pi_t$ defined in \eqref{eq:pi_t} satisfies
    \begin{align*}
        \widehat\pi_t(x)  
        % &= \left(I - \sigma_t^2(A^\top \Sigma A + \sigma_t^2 I)^{-1}\right) x \\
         = (A^\top \Sigma A + \sigma_t^2 I)^{-1} A^\top \Sigma A x
        % & = A^\top (I + \sigma_t^2 \Sigma^{-1})^{-1}Ax\\
        &=\textstyle  \argmin_{z\in \range(A^\top)} \|x-z\|^2 + \sigma_t^2 \|Az\|_{\Sigma^{-1}}^2.
    \end{align*}
\end{lemma}

Lemma~\ref{lemma:pi-t} shows that, in the linear-Gaussian setting, the denoiser has an exact optimization interpretation: it returns the point in the data subspace that balances proximity to the noisy input with a noise-dependent quadratic regularization induced by the data covariance. This motivates the following analysis of the guided diffusion dynamics. The key idea is that, after the gradient step, the Stein denoiser acts as a regularized projection back toward the linear data manifold.

We now state the convergence result of DCG. The proof uses a time-dependent Lyapunov function that combines objective suboptimality on the data subspace, the normal component away from the subspace, and an explicit error term induced by the finite diffusion noise. %[TODO: I like $t_0$ better than $T$, remember to change it for the convex and Riemann setting]

\begin{theorem}\label{theorem: convergence-guided-diffusion}\label{theorem:convergence-guided-diffusion}
We study the guided diffusion dynamics specified by \eqref{eq:projected-gradient-guidance-relaxed} under Assumptions~\ref{assump:linear} and~\ref{assump:f-strongly-convex}. Let $t_0 = \max\{t:\sigma_t^2 \le \lambda_{\min}(\Sigma)(1-c)/c\}$, where $c\in(0,1)$ is an arbitrary constant. We further assume a constant guidance step size satisfying $\eta_t=\eta<1/(\mu c)$ and $\eta\epsilon\le 1/L$ for all $t\le t_0$.
% We study the guided diffusion dynamics specified by \eqref{eq:projected-gradient-guidance-relaxed}, and we assume that Assumption \ref{assump:linear} and \ref{assump:f-strongly-convex} hold. We also assume the following conditions
% \begin{itemize}
% \item Denote $t_0$ to be the largest timestep such that $\sigma_{t_0}^2 \le \lmin(\Sigma)(1-c)/c$, \textcolor{blue}{where $0<c < 1$ is an arbitrary constant that we pick}.
%     \item $\eta_t = \eta < \frac{1}{\mu c}, \epsilon \le \frac{1}{L\eta}$ for $t \le t_0$.
% \end{itemize}
    Define  
    \begin{align}\label{eq:def-Phi_t}
      \textstyle  \Phi_t(x):= f(x) + \frac{\sigma_t^2 \eta^{-1}}{2} x^\top A^\top \Sigma^{-1} Ax, \qquad x_t^\star = \argmin_{x\in \range(A^\top)} \Phi_t(x)
    \end{align}
    and
    \vspace{-10pt}
    \begin{align*}
       \textstyle  V_t(x):=  (\Phi_{t}(A^\top Ax) - \Phi_{t}(x_{t}^\star)) +  \frac{\eta L^2}{2(1-\eta\mu c)} \|A^\perp x\|^2 + \frac{\eta^{-1}D^2}{2(1-\eta\mu c)}\sigma_t^2,
    \end{align*}
    where $D = \max_{t\le t_0} \|Ax_t^\star\|_{\Sigma^{-1}}$.
    Then running \eqref{eq:projected-gradient-guidance-relaxed} starting from $t_0$ we have
    \begin{align*}
        V_{t-1}(x_{t-1}) \le (1-\eta \epsilon \mu c) V_t(x_{t}), ~~\forall t\le t_0
    \end{align*}
    Hence, letting $x^\star
    :=
    \argmin_{x\in\range(A^\top)}f(x),$
the projected component satisfies
\begin{align*}
    f(A^\top Ax_0)-f(x^\star)
    &\textstyle \le
    (1-\eta\epsilon\mu c)^{t_0}V_{t_0}(x_{t_0})
    +
    (1-\epsilon)^{2t_0}\frac{\sigma_T^2}{2\eta}\|Ax^\star\|_{\Sigma^{-1}}^2,
\end{align*}
while the normal component satisfies
    $\|A^\perp x_0\|
    =
    (1-\epsilon)^{t_0}\|A^\perp x_{t_0}\|$.
\end{theorem}

\paragraph{Comparison with prior linear-Gaussian theory.}
The closest linear-Gaussian analysis is~\cite{guo2024gradient}, which proves convergence of the generated-distribution mean to a data-regularized optimum, with the unregularized target requiring adaptive score updates. Algorithmically, it relies on repeated batch generation and objective linearization, whereas our method uses a single reverse trajectory with one objective-gradient evaluation per step, making it substantially more sample efficient.

\section{Convergence Analysis: Convex Constraint Setting}
\label{sec:convex-analysis}
\vspace{-5pt}
We next consider the case where the learned data geometry is a compact convex feasible region $\cX\subset\bR^n$. This is a natural intermediate case between the linear-Gaussian model and general data geometries. The exact projection $\pi_{\cX}$ is well-defined and nonexpansive, but the geometry may be nonsmooth. All proofs are in Appendix~\ref{app:convex}.
\begin{assump}[Data model on a compact convex set in an affine subspace]
\label{assump:data-convex}
Let $\cA\subseteq\bR^n$ be a $d$-dimensional affine subspace, and let
$\cX\subseteq\cA$ be compact and convex with nonempty relative interior
in $\cA$. Suppose that $X$ is supported on $\cX$ with law $\mu$
satisfying
\begin{align*}
    d\mu(q)
    =
    \vartheta(q)\,d\mathcal H^d(q),
\end{align*}
where $\mathcal H^d$ denotes the $d$-dimensional Hausdorff measure
restricted to $\cA$. Assume that there exist constants
$0<\vartheta_-\le\vartheta_+<\infty$ such that $\vartheta_-
    \le
    \vartheta(q)
    \le
    \vartheta_+,
    \qquad
    q\in\cX.$
\end{assump}

% \begin{assump}[Data model on a compact convex region]\label{assump:data-convex}
% Let $\cX\subset\bR^n$ be a compact convex set with nonempty interior. Let $X$ be supported on $\cX$ with law $\mu$ satisfying
% \begin{align*}
% d\mu(x)=\vartheta(x)dx,
% \end{align*}
% where $dx$ denotes Lebesgue measure on $\bR^n$ restricted to $\cX$. Assume that there exist constants $0<\vartheta_-\le\vartheta_+<\infty$ such that $\vartheta_-\le \vartheta(x)\le \vartheta_+,
% \qquad x\in\cX.$
% \end{assump}

We first state the projection-approximation property of the Stein denoiser in this convex setting. Under Assumption~\ref{assump:data-convex}, the denoiser is uniformly close to the Euclidean projection onto $\cX$, with an error proportional to the noise level. This result is the key ingredient that allows us to transfer projected-gradient arguments to the guided diffusion dynamics.

\begin{lemma}[Stein score approximates projection onto a compact convex set]
\label{lem:stein-convex-projection}
Under Assumption~\ref{assump:data-convex}, there exists a constant
$C_{\cX}>0$, depending only on $d$ and
$\vartheta_+/\vartheta_-$, such that, for every $x\in\bR^n$ and every
$\sigma>0$,
the Stein denoising operator $\widehat\pi_\sigma$ defined in \eqref{eq:stein-denoiser-continuous} satisfies
\begin{align*}
\big\|\widehat\pi_\sigma(x)-\pi_{\cX}(x)\big\|
\le C_{\cX}\sigma,
\qquad
\widehat\pi_\sigma(x)\in\cX.
\end{align*}
\end{lemma}

Lemma~\ref{lem:stein-convex-projection} shows that, in the small-noise regime, replacing the exact projection $\pi_{\cX}$ by the Stein denoiser introduces a controlled projection error. We now use this error bound to prove convergence guarantees for the proposed update. We first consider strongly convex objectives, and then turn to nonconvex objectives.

\subsection{Convergence for Strongly Convex Objectives}

For strongly convex objectives, the analysis uses a Lyapunov function that combines the projected iterate with control of the distance to the feasible region. Specifically, the following one-step lemma tracks both the projected component $p_t=\pi_{\cX}(x_t)$ and the normal displacement $n_t=x_t-p_t$.

\begin{lemma}\label{lemma:one-step-descent-convex}
Let $\mathcal X\subseteq\mathbb R^n$ be nonempty, closed, convex, and bounded, and let
   $ D_{\mathcal X}:=\sup_{u,v\in\mathcal X}\|u-v\|<\infty .$
Suppose that $f:\mathbb R^n\to\mathbb R$ is $\mu$-strongly convex and $L$-smooth on a convex neighborhood containing the iterates and $\mathcal X$. Let
    $x^\star:=\arg\min_{x\in\mathcal X} f(x),~~
    G_\star:=\|\nabla f(x^\star)\|.$
Consider the update in \eqref{eq:projected-gradient-guidance-relaxed}
with $\epsilon(1+L\eta)\le 2, \eta\mu<1.$
Assume that, for every $z$, $\widehat\pi_t(z)\in\mathcal X, \|\widehat\pi_t(z)-\pi_{\mathcal X}(z)\|\le E_t.$

Let ~~$ p_t:=\pi_{\mathcal X}(x_t),~n_t:=x_t-p_t,
    ~
    \mathcal V_t:=f(p_t)-f(x^\star)+\frac{1}{2\eta}\|p_t-x^\star\|^2.$
Define
\begin{align*}
    \rho:=1-\epsilon\eta\mu,&
    \quad
    A_{\eta,\epsilon}:=1+\epsilon\eta L,\quad 
    M_\eta:=G_\star+\left(L+\frac{1}{\eta}\right)D_{\mathcal X},
    \quad
    C_\eta:=\frac{L}{2}+\frac{1}{2\eta},\\
   & B_1:=M_\eta A_{\eta,\epsilon},
    \quad
    B_2:=2C_\eta A_{\eta,\epsilon}^2.
\end{align*}
Then
\vspace{-10pt}
\begin{align*}
    \mathcal V_{t-1}
    \le
    \rho\mathcal V_t
    +
    B_1\|n_t\|
    +
    B_2\|n_t\|^2
    +
    \epsilon M_\eta E_t
    +
    2\epsilon^2C_\eta E_t^2.
\end{align*}
Moreover, the normal component satisfies
   $ \|n_{t-1}\|
    \le
    (1-\epsilon)\|n_t\|.$
\end{lemma}
The lemma shows that the Lyapunov function $\mathcal V_t$ contracts by the factor
$\rho=1-\epsilon\eta\mu$, up to errors arising from the normal component and the
denoising approximation. Since both $\|n_t\|$ and $E_t$ decay geometrically along
the reverse trajectory, iterating the recursion yields the finite-time bound below.

\begin{theorem}[Convergence for strongly convex objectives]\label{thm:convergence-convex}
Let Assumption~\ref{assump:data-convex} and the assumptions of
Lemma~\ref{lemma:one-step-descent-convex} hold. Consider the update in
\eqref{eq:projected-gradient-guidance-relaxed}. Then, with $\rho$, $B_1$,
$B_2$, $M_\eta$, and $C_\eta$ defined as in
Lemma~\ref{lemma:one-step-descent-convex}, we have
\begin{align*}
    \mathcal V_0
    &\textstyle \le
    \rho^T\mathcal V_T+
    \left(
        B_1\|n_T\|
        +
        \epsilon M_\eta C_{\mathcal X}\sigma_T
    \right)
    \frac{
        \rho^T-(1-\epsilon)^T
    }{
        \rho-(1-\epsilon)
    }+
    \left(
        B_2\|n_T\|^2
        +
        2\epsilon^2C_\eta C_{\mathcal X}^2\sigma_T^2
    \right)
    \frac{
        \rho^T-(1-\epsilon)^{2T}
    }{
        \rho-(1-\epsilon)^2
    }.
\end{align*}
Consequently,
\begin{align*}
     \|p_0\!-\!x^\star\|^2
    &\textstyle\!\le\!
    2\eta
    \Bigg[\!
        \rho^T\mathcal V_T\!+\!
        \left(
            B_1\|n_T\|
            \!+\!
            \epsilon M_\eta C_{\mathcal X}\sigma_T
        \right)
        \frac{
            \rho^T\!\!-\!(1\!-\!\epsilon)^T
        }{
            \rho-(1\!-\!\epsilon)
        }\!+\!
        \left(
            B_2\|n_T\|^2
           \! + \!
            2\epsilon^2C_\eta C_{\mathcal X}^2\sigma_T^2
        \right)
        \frac{
            \rho^T\!-\!(1\!-\!\epsilon)^{2T}
        }{
            \rho-(1-\epsilon)^2
        }
    \!\Bigg]\!.
\end{align*}
Moreover,~
    $\|n_0\|
    \le
    (1-\epsilon)^T\|n_T\|.$
\end{theorem}
We use $T$ in this section because the convex-set guarantee holds for an arbitrary number of reverse steps. This differs from the preceding linear analysis, where $t_0$ denotes the largest index satisfying the required noise threshold. We use $t_0$ again in the subsequent local analysis for the same reason, namely that the guarantee applies only after the reverse process enters the required low-noise regime.
%Theorem~\ref{thm:convergence-convex} shows that the projected component $p_0=\pi_{\cX}(x_0)$ converges toward the constrained optimizer, up to explicit errors induced by the finite-noise denoising approximation. The normal component also contracts geometrically, reflecting the fact that each denoising step returns to $\cX$ up to the controlled approximation error.

\subsection{Convergence for Nonconvex Objectives}

We next consider nonconvex objectives over the same compact convex feasible region. In this case, global optimality is not expected. The appropriate stationarity measure is the projected-gradient mapping
\vspace{-10pt}
\begin{align}\label{eq:def-G-eta}
\textstyle \mathcal G_\eta(p)
:=
\frac{1}{\eta}\left(p-\pi_{\cX}(p-\eta\nabla f(p))\right),
\end{align}
which vanishes exactly at first-order stationary points of the constrained problem over $\cX$.
\begin{lemma}[One-step descent for nonconvex objectives over a convex region]\label{lem:one-step-descent-convex-nonconvex}
Let $\cX\subseteq\bR^n$ be nonempty, compact, and convex. Suppose that $f:\bR^n\to\bR$ is $L$-smooth on a convex neighborhood containing the iterates and $\cX$. Consider the backward update as in \eqref{eq:projected-gradient-guidance-relaxed}
with ~$ \epsilon\eta L<2.$
Assume that, for every $z$,
~~ $\widehat\pi_t(z)\in\cX,
~~
\|\widehat\pi_t(z)-\pi_{\cX}(z)\|\le E_t.$ Let
~~$p_t:=\pi_{\cX}(x_t),
n_t:=x_t-p_t.$
Define the following constants
\begin{align*}
&\textstyle \alpha_{\eta,\epsilon}:=\epsilon\eta\left(1-\frac{\epsilon\eta L}{2}\right),
~~A_{\eta,\epsilon}:=1+\epsilon\eta L,~~
G:=\textstyle\sup_{x\in\cX}\|\nabla f(x)\|,
~~
\Gamma_1:=GA_{\eta,\epsilon},
~~
\Gamma_2:=LA_{\eta,\epsilon}^2.
\end{align*}
Then, with $\mathcal G_\eta(p_t)$ defined in \eqref{eq:def-G-eta}, we have
\begin{align*}
f(p_{t-1})
&\le
f(p_t)
-
\alpha_{\eta,\epsilon}\|\mathcal G_\eta(p_t)\|^2
+
\Gamma_1\|n_t\|
+
\Gamma_2\|n_t\|^2\
+
\epsilon GE_t
+
\epsilon^2LE_t^2.
\end{align*}
Moreover, the normal component satisfies
~~ $\|n_{t-1}\|\le (1-\epsilon)\|n_t\|.$
\end{lemma}
The next theorem sums the one-step descent inequality over the backward diffusion trajectory. The result gives a standard nonconvex convergence guarantee: at least one iterate has small projected-gradient mapping, with explicit error terms due to the denoising approximation.

\begin{theorem}[Convergence to stationarity]\label{thm:backward-stein-nonconvex}
Let Assumption \ref{assump:data-convex} and the conditions of Lemma~\ref{lem:one-step-descent-convex-nonconvex} hold. Consider the backward update as in \eqref{eq:projected-gradient-guidance-relaxed}.
Then, with $\alpha_{\eta,\epsilon}$, $\Gamma_1$, $\Gamma_2$, $G$, and $L$ defined as in Lemma~\ref{lem:one-step-descent-convex-nonconvex}, define
\begin{align*}
\Delta_T
:=
\max_{x\in\cX}f(x)-\min_{x\in\cX}f(x)
+
\frac{\Gamma_1\|n_T\|}{\epsilon}
+
\frac{\Gamma_2\|n_T\|^2}{\epsilon}
+
GC_{\cX}\sigma_T
+
\epsilon LC_{\cX}^2\sigma_T^2,
\end{align*}
where $C_{\cX}$ is the constant from Lemma~\ref{lem:stein-convex-projection}. Then
\begin{align*}
\textstyle \alpha_{\eta,\epsilon}\sum_{t=1}^{T}\|\mathcal G_\eta(p_t)\|^2
\le
\Delta_T, \quad \Longrightarrow \quad
\min_{1\le t\le T}\|\mathcal G_\eta(p_t)\|^2
\le
\frac{\Delta_T}{\alpha_{\eta,\epsilon}T}.
\end{align*}
Moreover, the normal component satisfies
$\|n_0\|\le (1-\epsilon)^T\|n_T\|.$
\end{theorem}

\paragraph{Comparison with prior convex theory.}
Together, Theorems~\ref{thm:convergence-convex} and
\ref{thm:backward-stein-nonconvex} show that DCG recovers the standard
first-order convergence behavior of projected-gradient methods. For strongly convex objectives, the projected optimization error
converges geometrically, with the overall rate determined by the optimization
contraction, the normal contraction, and the decay of the denoising error. For smooth nonconvex objectives, DCG achieves an $O(1/T)$ best-iterate bound on the squared projected-gradient mapping, equivalently an $O(1/\sqrt{T})$ rate in gradient norm \footnote{We acknowledge that the bound in Theorem~\ref{thm:backward-stein-nonconvex} is a best-iterate convergence guarantee and does not directly establish stationarity of the final output $x_0$. This is a standard limitation in nonconvex optimization, where convergence guarantees are commonly stated for the best iterate unless additional structural assumptions are imposed.}. In both cases, the normal component contracts geometrically, so the
analysis controls feasibility together with optimization progress.

The closest convex-support analysis is~\cite{leong2025recovery}, which likewise interprets diffusion denoisers as time-varying approximate projections, but analyzes a direct projected-gradient-style iteration for linear inverse recovery under a restricted-isometry condition, corresponding to a specialized quadratic objective. In contrast, our method retains the DDIM reverse-update structure and analyzes the general problem $\min_{x\in\cX} f(x)$, establishing standard first-order rates for general strongly convex and smooth nonconvex objectives.

\vspace{-5pt}
\section{Convergence Analysis: Riemannian Manifold Setting}
\vspace{-5pt}
We finally consider the case where the learned data geometry is a smooth embedded submanifold $\cM\subseteq\bR^n$. Compared with the convex setting, the projection geometry is local: the nearest-point projection is well-defined only in a tubular neighborhood of $\cM$, and descent is measured by the Riemannian gradient along the manifold. All proofs are in Appendix~\ref{app:riemann}.
\paragraph{Problem setup.} Let $\cM \subseteq \bR^n$ be a $d$-dimensional $C^2$ embedded submanifold, equipped with the Riemannian metric induced by the ambient Euclidean inner product. For each $p\in \cM$, let $T_p\cM$ denote the tangent space of $\cM$ at $p$, and let $N_p\cM := (T_p\cM)^\perp$ denote the corresponding normal space in $\bR^n$. Let $U_p\in \bR^{d\times n}$ and $V_p\in \bR^{(n-d)\times n}$ be row-orthonormal matrices whose rows form orthonormal bases of $T_p\cM$ and $N_p\cM$, respectively. Thus $U_pU_p^\top = I_d$, $V_pV_p^\top = I_{n-d}$, and $U_pV_p^\top = 0$. By the local graph theorem for embedded submanifolds, there exist an open neighborhood $\cA_p\subseteq \bR^d$ of the origin and a $C^2$ map $\phi_p:\cA_p\to \bR^{n-d}$ such that every point of $\cM$ sufficiently close to $p$ can be uniquely written as $p+U_p^\top \alpha + V_p^\top \phi_p(\alpha)$ for some $\alpha\in \cA_p$. We refer to $\phi_p$ as the local graph function, or Monge parametrization, of $\cM$ over $T_p\cM$. Since this parametrization is centered at $p$ and tangent to $T_p\cM$ at $\alpha=0$, it satisfies $\phi_p(0)=0$ and $D\phi_p(0)=0$. We denote the corresponding local parametrization by
$\Psi_p(\alpha):=p+U_p^\top \alpha+V_p^\top \phi_p(\alpha)$. Finally, let $r_{\mathrm{tub}}>0$ be a uniform tubular radius of $\cM$, meaning that every point within distance $r_{\mathrm{tub}}$ of $\cM$ has a unique nearest point on $\cM$. We define the corresponding tubular neighborhood by
$\cT:=\{x\in\bR^n:\operatorname{dist}(x,\cM)<r_{\mathrm{tub}}\}$ and let
$\pi_\cM(x):=\operatorname*{argmin}_{q\in \cM}\|x-q\|$
denote the nearest-point projection onto $\cM$. The projection $\pi_\cM$ is well-defined and smooth on $\cT$.
%We denote the corresponding local parametrization by $\Psi_p(\alpha):=p+U_p^\top \alpha+V_p^\top \phi_p(\alpha)$. For a sufficiently small radius $\rho_p>0$, we define $\overline{\cU}_p:=\{\alpha\in \bR^d:\|\alpha\|\le \rho_p\}\subseteq \cA_p$. Finally, let $\pi$ denote the nearest-point projection onto $\cM$, defined on a sufficiently small tubular neighborhood $\cT$ of $\cM$ by $\pi_\cM(x):=\operatorname*{argmin}_{q\in \cM}\|x-q\|$. By choosing $\cT$ sufficiently small, this projection is well-defined and smooth in a neighborhood of $\cM$.
We measure stationarity using the Riemannian gradient. For $p\in\cM$, the Riemannian gradient is the tangent projection of the ambient Euclidean gradient:
$\grad_\cM f(p)
=
\Proj_{T_p\cM}\nabla f(p)
=
U_p^\top U_p\nabla f(p).$
\paragraph{Stein projection lemma.}
We first state the data and regularity assumptions needed to control the Stein denoiser near the manifold. 
\begin{assump}[Data model for the Stein score]\label{assump:data}
Let $X$ be supported on $\cM$ with density
$$d\mu(x) \;=\; \vartheta(x)\, d\mathrm{vol}_\cM(x),$$
where $\vartheta \in C^2(\cM)$ and there exist $0<\vartheta_- \le \vartheta_+<\infty$ such that $\vartheta_- \le \vartheta(x) \le \vartheta_+$ for all $x \in \cM$.
\end{assump}

\begin{assump}[Compact uniform local graph regularity]\label{assump:uniform-graph}
Assume that $\cM\subseteq\bR^n$ is compact and that there exist constants $\rho_0>0$ and $\kappa>0$, such that, for every $p\in\cM$, $\cM$ admits a Monge chart
\begin{align*}
    \Psi_p(\alpha)
    =
    p+U_p^\top\alpha+V_p^\top\phi_p(\alpha),
    \qquad
    \alpha\in B_{\rho_0}(0)\subseteq\bR^d,
\end{align*}
with $\phi_p(0)=0,~~D\phi_p(0)=0.$
Moreover, uniformly over $p\in\cM$ and $\alpha\in B_{\rho_0}(0)$,
\begin{align*}
    \|D^2\phi_p(\alpha)[u,v]\|
    \le
    \kappa\|u\|\|v\|.
\end{align*}
Consequently, whenever $\|\alpha\|\le\rho_0$,    $\|D\phi_p(\alpha)\|\le \kappa\|\alpha\|.$
In particular, if $\|\alpha\|\le 1/(6\kappa)$, then $\|D\phi_p(\alpha)\|\le \frac16 .$
\end{assump}
Under these assumptions, the Stein denoiser approximates the nearest-point projection onto $\cM$ inside a tubular neighborhood. %Compared with the convex case, the error contains a logarithmic factor, reflecting the local nature of the manifold approximation and the need to control Gaussian mass away from the local chart.

\begin{lemma}\label{lem:stein-projection-c2}
Assume that Assumption \ref{assump:data}, \ref{assump:uniform-graph} hold. Let
\begin{align}
  \textstyle \r
    :=
    \min\left\{
        r_{\mathrm{tub}},
        \rho_0,
        \frac{1}{6\kappa}
    \right\}, \qquad\cU:=\left\{x:\dist(x,\cM)\le r\right\}.\label{eq:cU}
\end{align}
Then there exist constants $\bar\sigma>0$ and $C>0$, depending only on the constants in the assumptions above, such that, for all $0<\sigma\le\bar\sigma$ and all $x\in\cU$,
    we have ~$\big\|\widehat\pi_\sigma(x)-\pi_\cM(x)\big\|
    \le
    C\sigma|\log\sigma|.$
\end{lemma}
Lemma~\ref{lem:stein-projection-c2} allows us to view $\widehat\pi_t$ as an inexact nearest-point projection, provided the iterates remain in $\cU$ and the noise level is sufficiently small. This slightly improves on the $O(\sigma|\log\sigma|^3)$ rate in~\cite{kharitenko2025landing}. %Lemma~\ref{lem:stein-projection-c2} also complements the approximate-projection interpretation of denoising in~\cite{permenter2023interpreting} by providing a deterministic and uniform small-noise bound for the ideal posterior-mean denoiser on a smooth data manifold, rather than assuming a relative projection-error model or deriving a high-probability guarantee for randomly perturbed data.
\vspace{-3pt}
\paragraph{Convergence.}
We now analyze the guided diffusion update on the manifold. Since the nearest-point projection is only locally well-defined and the Stein projection bound holds only in $\cU$, we impose a local trajectory condition ensuring that the iterates and the gradient-updated points remain in the region where the projection approximation is valid.
% \begin{assump}\label{assump:iterate-within-graph}
% We assume that the iterates satisfy $x_t\in\cU$ and $x_t - \eta \nabla f(x_t) \in \cU$, where $\cU$ is defined as in \eqref{eq:cU}.
% \end{assump}

\begin{assump}\label{assump:smoothness}
The objective function $f$ is $L$-smooth and satisfies $\|\nabla f(x)\|\le G$ on ~$\cU$.
\end{assump}

Similar to the analysis in the previous sections, the descent analysis uses a Lyapunov function that combines the objective evaluated at the projected point $p_t:=\pi_\cM(x)$ with a quadratic penalty on the normal displacement $n_t:=x-\pi_\cM(x)$. 
\begin{lemma}[Lyapunov descent with inexact projection]\label{lemma:Lyapunov-descent-riemann}
Suppose Assumption \ref{assump:uniform-graph} and \ref{assump:smoothness} hold. Consider the update \eqref{eq:projected-gradient-guidance-relaxed}. 
Let $z_t:=x_t-\eta\nabla f(x_t)$ and suppose that
   $ \|\widehat\pi_t(z_t)-\pi_\cM(z_t)\|\le E_t .$ Assume that $\dist(x_t,\cM)
    \le
    \frac{\r}{2}$, where $r$ is defined in \eqref{eq:cU}. Let $p_t:=\pi_\cM(x_t), n_t:=x_t-p_t$. If $\eta \le \min\{\frac{1}{50L},\frac{\r}{4G}\}, \epsilon \le \min\{\frac{1}{512\eta(G\kappa+L)},\frac{r}{4E_t}\}$, then we have that $\dist(x_{t-1},\cM)
    \le
    \frac{\r}{2}$, and with
\begin{align*}
    \textstyle V(x):=f(\pi_\cM(x))+\frac{1}{160\eta}\|x-\pi_\cM(x)\|^2,
\end{align*}
we have
\vspace{-15pt}
\begin{align*}
   \textstyle V(x_{t-1})
    \le
    V(x_t)
    -
    \frac{\epsilon\eta}{480}\|\grad_\cM f(p_t)\|^2
    -
    \frac{\epsilon}{2000\eta}\|n_t\|^2
    +
    \frac{6\epsilon}{\eta}E_t^2.
\end{align*}
%where  and $g_t:=U_{p_t}\nabla f(p_t)$.
\end{lemma}
The descent inequality contains two negative terms. The term involving $\|\grad_\cM f(p_t)\|^2$ controls the Riemannian stationarity of the projected iterate $p_t=\pi_\cM(x_t)$, while the term involving $\|n_t\|^2$ controls the normal distance to the manifold. Thus, the update simultaneously drives the projected point toward Riemannian stationarity and keeps the ambient iterate close to $\cM$, up to the denoising error. Moreover, the argument depends on the denoiser only through its pointwise projection error, so it extends directly to an imperfectly learned denoiser. In particular, if the implemented denoiser $\widetilde{\pi}_t$ satisfies
$\|\widetilde{\pi}_t(z_t)-\widehat{\pi}_t(z_t)\|\le\delta_t$, then the same result holds with $E_t$ replaced by $E_t+\delta_t$. Thus, finite-noise and denoiser-approximation errors enter the Lyapunov bound through the single additive term
$\frac{6\epsilon}{\eta}(E_t+\delta_t)^2$.

Combining Lemma~\ref{lemma:Lyapunov-descent-riemann} with the Stein projection approximation in Lemma~\ref{lem:stein-projection-c2} gives the following finite-time convergence result. 

\begin{theorem}\label{thm:riemann-backward-convergence}
Suppose Assumptions~\ref{assump:data} and~\ref{assump:uniform-graph} hold, and define $t_0$ as the largest index such that $\sigma_{t_0}\le \bar\sigma$. Assume that $\dist(x_{t_0},\cM)\le \frac{\r}{2}$. Consider the backward denoising scheme with gradient guidance as in \eqref{eq:projected-gradient-guidance-relaxed}
with
$\eta
\le
\min\left\{
    \frac{1}{50L},
    \frac{\r}{4G}
\right\},~
\epsilon
\le
\min\left\{
    \frac{1}{2},
    \frac{1}{512\eta(G\kappa+L)},
    \frac{\r}{4C\sigma_{t_0}|\log\sigma_{t_0}|}
\right\}.$
Let $p_t:=\pi_\cM(x_t)$ and $n_t:=x_t-p_t$. Define
\begin{align*}
\textstyle \Delta_{t_0}
:=
V(x_{t_0})-\min_{x\in\cM}f(x)
+
\frac{6C^2\sigma_{t_0}^2}{\eta}
\left(
    |\log\sigma_{t_0}|^2
    +
    4|\log\sigma_{t_0}|
    +
    8
\right),
\end{align*}
where $V$ is defined in Lemma~\ref{lemma:Lyapunov-descent-riemann} and $C$ is defined in Lemma~\ref{lem:stein-projection-c2}. Then
\begin{align*}
  \textstyle  \sum_{t=1}^{t_0}
    \left(
        \frac{\epsilon\eta}{480}\|\grad_\cM f(p_t)\|^2
        +
        \frac{\epsilon}{2000\eta}\|n_t\|^2
    \right)
    \le
    \Delta_{t_0}.
\end{align*}
Consequently,
\vspace{-15pt}
\begin{align*}
 \textstyle    \min_{1\le t\le t_0}
    \left(
        \frac{\epsilon\eta}{480}\|\grad_\cM f(p_t)\|^2
        +
        \frac{\epsilon}{2000\eta}\|n_t\|^2
    \right)
    \le
    \frac{\Delta_{t_0}}{t_0}.
\end{align*}
\end{theorem}

\vspace{-5pt}
\paragraph{Comparison with prior Riemannian theory.}
The closest theoretical comparison is~\cite{kharitenko2025landing}. Both analyses are local and establish an $O(1/\sqrt{T})$ best-iterate rate for the Riemannian gradient norm, up to a finite-noise bias. However,~\cite{kharitenko2025landing} controls both the denoiser and its Jacobian to approximate retraction and tangent projection, requiring derivative-level bounds and $C^3$ regularity. Our analysis requires only value-level approximation of the nearest-point projection, not Jacobian accuracy, which is a weaker requirement because accurate denoiser values do not generally imply accurate Jacobians. Moreover,~\cite{kharitenko2025landing} constructs a standalone optimizer that applies a full approximate retraction at each iteration, whereas DCG follows a single reverse-diffusion trajectory and interpolates between the current iterate and the denoised projected-gradient candidate. This relaxed update yields weaker per-step contraction and consequently more conservative constants.

\vspace{-5pt}
\section{Conclusion}
\vspace{-5pt}
We introduced Denoising-Corrected Gradient Guidance, an inference-time method that applies a gradient step followed by ordinary pretrained diffusion denoising. By interpreting the Stein posterior-mean denoiser as an approximate projection onto the data support, we connected DCG to inexact projected-gradient optimization and established projection bounds and finite-time guarantees for linear subspaces, compact convex sets, and compact smooth submanifolds.
%The results cover both strongly convex and smooth nonconvex objectives and explicitly quantify the effects of finite noise, normal displacement, and denoiser error.
Experiments on synthetic constrained problems, trajectory tracking, and D4RL benchmarks show that DCG improves the balance between objective value and feasibility over direct gradient guidance while remaining competitive with more geometry-intensive baselines. %Promising directions include sharper guarantees for learned denoisers, adaptive guidance and noise schedules, and extensions to stochastic objectives and broader classes of data-supported constraints.

\bibliography{bib}
\bibliographystyle{plainnat}

\appendix
\section{Related Works}
\label{app:geometry-preserving-guidance}

This appendix provides a more detailed comparison between the our DCG algorithm studied in this paper and related approaches for geometry-preserving guidance, constrained diffusion, reward alignment, and score-based optimization. 

% \paragraph{Off-manifold drift in naive guidance.}

% Suppose the data distribution is supported on, or concentrated near, a structured feasible set $\cX\subseteq\bR^n$. Depending on the application, $\cX$ may be a linear subspace, a closed convex set, a smooth embedded manifold, a set of physically admissible designs, or an implicitly defined data geometry. Standard reverse diffusion dynamics approximately transports noisy samples back toward the learned data distribution. However, when an external objective $f:\bR^n\to\bR$ is introduced, a naive guided update may add an ambient gradient direction directly to the reverse dynamics. At the level of the denoised prediction, this has the schematic form
% \begin{align*}
% \widehat\pi_t(x_t)-\lambda_t\sigma_t^2\nabla f(x_t).
% \end{align*}
% Even if $\widehat\pi_t(x_t)$ lies close to the learned data geometry, the vector $-\nabla f(x_t)$ need not be tangent to a manifold, feasible for a constraint set, or compatible with an implicit data prior. Thus, guidance can improve the objective value while degrading sample realism, feasibility, physical validity, or constraint satisfaction. This is the basic geometric tension that motivates geometry-preserving and constraint-aware guidance methods.

\paragraph{Gradient guidance from an optimization perspective.}

A closely related approach studies gradient guidance as an optimization procedure while accounting for the structure learned by the diffusion model~\cite{guo2024gradient}. Rather than applying the nonlinear objective gradient directly at each noisy state, the method operates through an outer loop. At each iteration, it generates a batch of samples, evaluates the objective gradient at their mean, and uses this gradient to construct a locally linearized objective for the next round of guided generation. The guidance is applied through the model’s predicted clean endpoint, which helps align the update with the learned data structure. With a fixed score model, the resulting sequence of distribution means can be interpreted as optimizing a data-regularized objective. The adaptive variant further fine-tunes the score model using newly generated samples, allowing this regularization to diminish over successive iterations.

The main distinction is therefore algorithmic as well as conceptual. Their method handles a nonlinear objective through repeated batch-based linearizations and full rounds of guided generation, with each guided step requiring differentiation through the clean-sample prediction. The adaptive variant additionally requires repeated model updates. In contrast, our method requires no outer-loop linearization, target reward value, batch generation, or differentiation through the denoiser. Accordingly, we analyze the method as an inexact projected-gradient algorithm for constrained optimization over $\cX$, rather than as a guided sampling procedure whose distributional mean follows a regularized optimization iteration.

\paragraph{Score-based and denoising-based manifold optimization.}

Landing with the Score is the closest work to our Riemannian-manifold analysis~\cite{kharitenko2025landing}. It considers optimization over a data manifold specified implicitly by a data distribution and shows that a learned score and its Jacobian can approximate geometric operations such as nearest-point and tangent-space projections. It then uses these approximations to construct Denoising Landing Flow and Denoising Riemannian Gradient Descent, with guarantees for approximate manifold adherence and a small Riemannian gradient norm.

Despite the shared geometric motivation, the two approaches use the diffusion model differently. Landing with the Score treats a diffusion-trained score network as a source of geometric information for building a standalone Riemannian optimizer: the score-derived denoising map approximates projection onto the manifold, while its Jacobian approximates projection onto the tangent space. The resulting algorithms explicitly assemble these operations into landing or Riemannian-gradient updates and therefore require score derivatives or Jacobian-vector products. Our method, by contrast, remains a guided-diffusion procedure. It does not separately estimate tangent spaces, construct Riemannian gradients, or differentiate through the denoiser. Our analysis shows that this simple guidance-and-denoising update itself behaves like a relaxed projected-gradient or retraction-type step when the denoiser approximates the nearest-point projection. This makes the method directly compatible with pretrained diffusion samplers while requiring only ordinary denoising evaluations. We also cover compact convex feasible sets in addition to smooth manifolds, connecting the same update to projected-gradient optimization beyond the Riemannian setting.

\paragraph{Reward alignment.}

Reward alignment updates a generative model toward a reward-specific distribution. DRaFT ~\cite{clark2024direct} backpropagates differentiable rewards through the sampling chain, whereas PRDP~\cite{deng2024prdp} uses reward-difference regression for black-box rewards. Tilt Matching~\cite{potaptchik2025tilt} and Iterative Tilting \cite{pachebat2025iterative} instead learn reward-tilted velocity or score fields without differentiating the reward through the sampling trajectory. Thus, although their training objectives differ, these methods all require a reward-specific training or distribution-matching stage.

DAS~\cite{kim2025testtime} uses Sequential Monte Carlo with multiple particles, resampling, and repeated reward evaluations. Adaptive Diffusion Guidance~\cite{azangulov2025adaptive} learns a time- and state-dependent guidance scale by solving a stochastic optimal-control problem, while Meta Flow Maps~\cite{potaptchik2026meta} require precomputing an auxiliary map to estimate reward-tilted value-function gradients. In contrast, our method requires neither auxiliary training nor Monte Carlo estimation, using only one objective-gradient evaluation and one denoising call per step.

\paragraph{Guidance for inverse problems and manifold constraints.}
Standard inverse-problem solvers combine reverse diffusion with measurement-based corrections~\cite{song2021solving,kawar2022denoising,song2023pseudoinverse}, but generally do not guarantee manifold preservation. MCG and diffusion posterior sampling apply guidance through the Tweedie estimate $\widehat x_0(x_t)$ via $\nabla_{x_t}f(\widehat x_0(x_t))$~\cite{chung2022improving,chung2022diffusion}; this requires differentiating through the score network, similar in spirit to \emph{Landing with the Score}. MPGD~\cite{he2024manifold} avoids this score-Jacobian computation by guiding directly in the estimated clean space, and uses an auxiliary autoencoder to enforce manifold-compatible updates; its guarantees rely on idealized linear-manifold and perfect-autoencoder assumptions. DiffStateGrad~\cite{zirvi2025diffstategrad} instead estimates a low-rank subspace by SVD and projects the guidance onto it, but provides no theoretical guarantee that this estimated subspace coincides with the true data manifold. CFG++~\cite{chung2025cfg++} modifies classifier-free guidance through interpolation of conditional and unconditional predictions and unconditional renoising, with its manifold-preservation guarantee focused on the locally linear setting. DMPlug~\cite{wang2024dmplug} parameterizes the solution as $x=R(z_T)$ for a deterministic DDIM sampler and optimizes the initial noise $z_T$; however, each update requires backpropagation through the unrolled sampler, making the method computationally expensive for long reverse trajectories.

\paragraph{Denoiser-based regularization and approximate projection.}
RED constructs inverse-problem regularizers from denoisers and, under
its regularity assumptions, relates the regularizer gradient to the
denoising residual~\cite{romano2017little}. A recent perspective surveys
denoisers as optimization, score-estimation, and approximate-projection
operators~\cite{milanfar2025denoising}. The closest direct technical
precedent for our projection analysis is
\citet{permenter2023interpreting}. That work models a denoiser through
a relative projection-error condition and interprets deterministic
diffusion sampling as approximate gradient descent on the squared
distance to the data support. Its convergence analysis concerns the
restoration dynamics of the diffusion sampler itself. Our analysis
introduces an additional task objective $f$ and consequently establishes
constrained optimality and stationarity guarantees.

\paragraph{Constrained and projected diffusion models.}
A line of work incorporates \emph{explicitly specified} constraints into diffusion models, assuming access to constraint functions or corresponding correction mechanisms. Projected Diffusion Models~\cite{christopher2024constrained} consider sample-level feasibility constraints and enforce them during inference through explicit projection. Mirror Diffusion Models~\cite{liu2023mirror} similarly assume access to the feasible geometry, but handle known convex constraint sets by mapping them to an unconstrained dual space through a mirror map and training diffusion in the transformed coordinates, with feasibility enforced by the inverse map. HardFlow~\cite{li2025hardflow} instead formulates hard-constrained
sampling for flow-matching models as trajectory optimization, enforcing
constraints only on the terminal sample rather than throughout the
sampling trajectory. Related constrained-training approaches~\cite{li2025aligning} instead incorporate problem-specific constraint violations into diffusion training. Another distinct line of research considers constraints at the \emph{distribution level}: Constrained Diffusion Models via Dual Training optimizes the generated distribution subject to prescribed distributional requirements~\cite{khalafi2024constrained}, while Primal-Dual Inference targets entropy-regularized distribution optimization under average constraints using a dual-conditioned score model and inference-time dual updates~\cite{hadou2026constrained}. In all these settings, the constraints themselves are explicitly available. Our setting is different: the feasible set $\cX$ is represented only implicitly by a pretrained diffusion model. Rather than assuming a constraint function or projection oracle, we show when the fixed Stein denoiser itself approximates projection onto this learned feasible geometry and use it as the correction step in projected-gradient optimization.

\paragraph{Constraint-aware diffusion models.}
A growing line of work incorporates constraints into diffusion models, with an important distinction being \textbf{where and how the constraints enter}. Projected Diffusion Models~\cite{christopher2024constrained} consider explicit, sample-level feasibility constraints and enforce them during inference through explicit projection. Related constrained-training approaches~\cite{li2025aligning} instead inject problem-specific trajectory constraints, such as collision avoidance and dynamics, into diffusion training through constraint-violation penalties. \cite{khalafi2024constrained} considers a different, distribution-level setting, where the generated distribution is optimized subject to requirements expressed through desired reference distributions using a Lagrangian dual formulation. Primal-Dual Inference~\cite{hadou2026constrained} further studies distribution optimization under average constraints, training a dual-conditioned score model and updating the dual variables alongside the samples during reverse diffusion. Most closely related to our setting, DiffOPT~\cite{kong2024diffusion} also considers an unknown feasible set learned from feasible data, but its optimization result concentrates on feasible points that are already strict local minima of the unconstrained objective, and therefore does not cover general constrained optima arising from the geometry or boundary of the feasible set. In contrast, our feasible set $\cX$ is encoded only \emph{implicitly by the pretrained diffusion model}, and our analysis covers general constrained optima, including the \emph{active-constraint regime} where the unconstrained gradient need not vanish.
\paragraph{Diffusion-based robotic motion planning.}
Diffusion models provide trajectory priors for robotic motion planning,
but guided sampling can violate collision, dynamics, or safety
requirements~\cite{janner2022planning,carvalho2025motion}. Existing
methods address this using control-barrier or control-Lyapunov
mechanisms~\cite{xiao2025safediffuser,mizuta2024cobl},
constraint-based search~\cite{shaoul2025multi}, explicit projection or
constrained optimization~\cite{liang2025simultaneous,liang2026discrete},
or receding-horizon stochastic control with terminal
constraints~\cite{giaretta2026direct}. These approaches rely on
explicitly evaluable constraints. DCG instead uses a pretrained denoiser
to correct objective-guided updates toward feasible trajectory structure
learned implicitly from demonstrations.

\section{Numerical Verifications}\label{app:numerics}

\subsection{Double-Integrator Trajectory Tracking}
\begin{figure}[htbp]
    \centering

    % Row A
    \begin{subfigure}{0.4\textwidth}
        \centering
        \includegraphics[width=\linewidth]{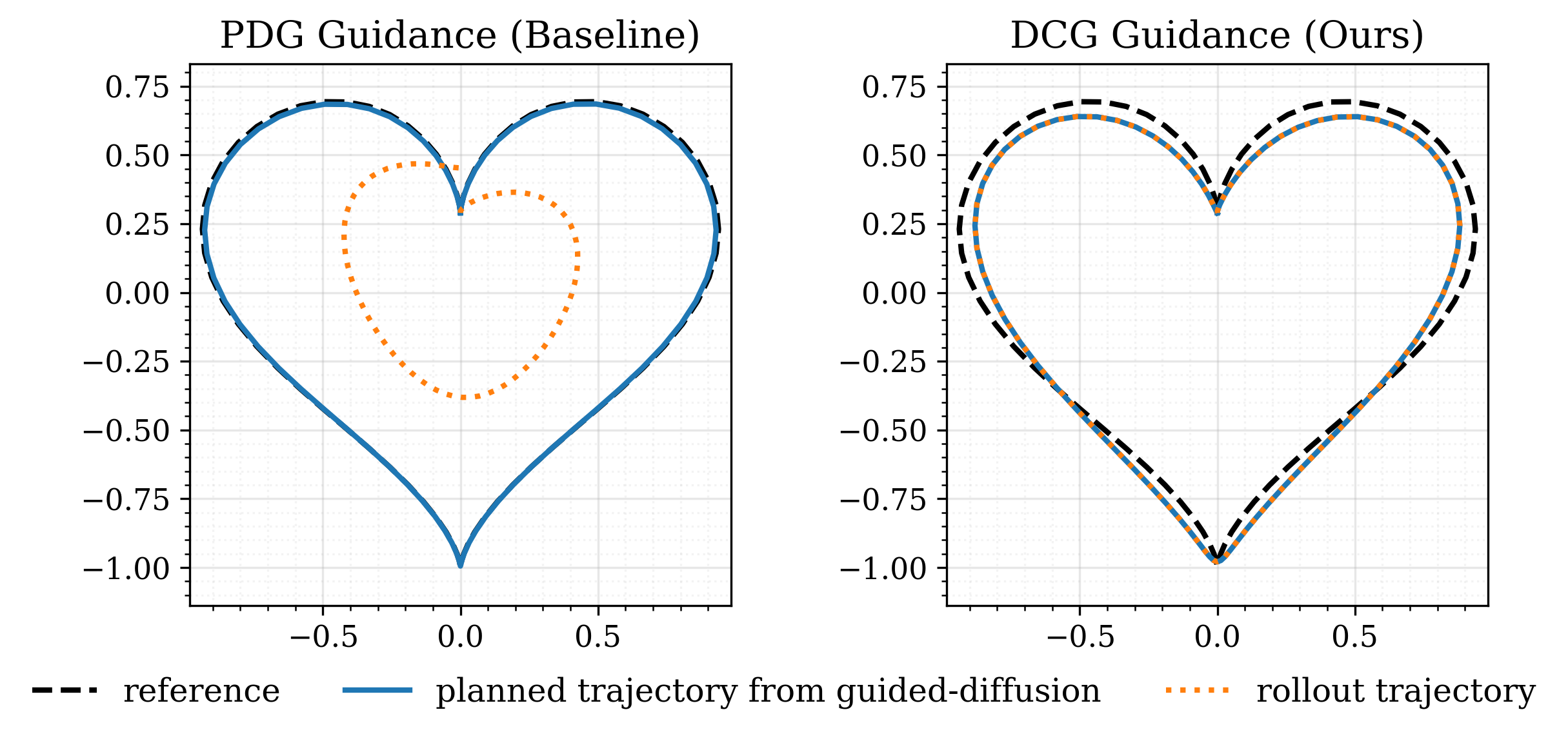}
    \end{subfigure}
    \hspace{10pt}
    \begin{subfigure}{0.43\textwidth}
        \centering
        \includegraphics[width=\linewidth]{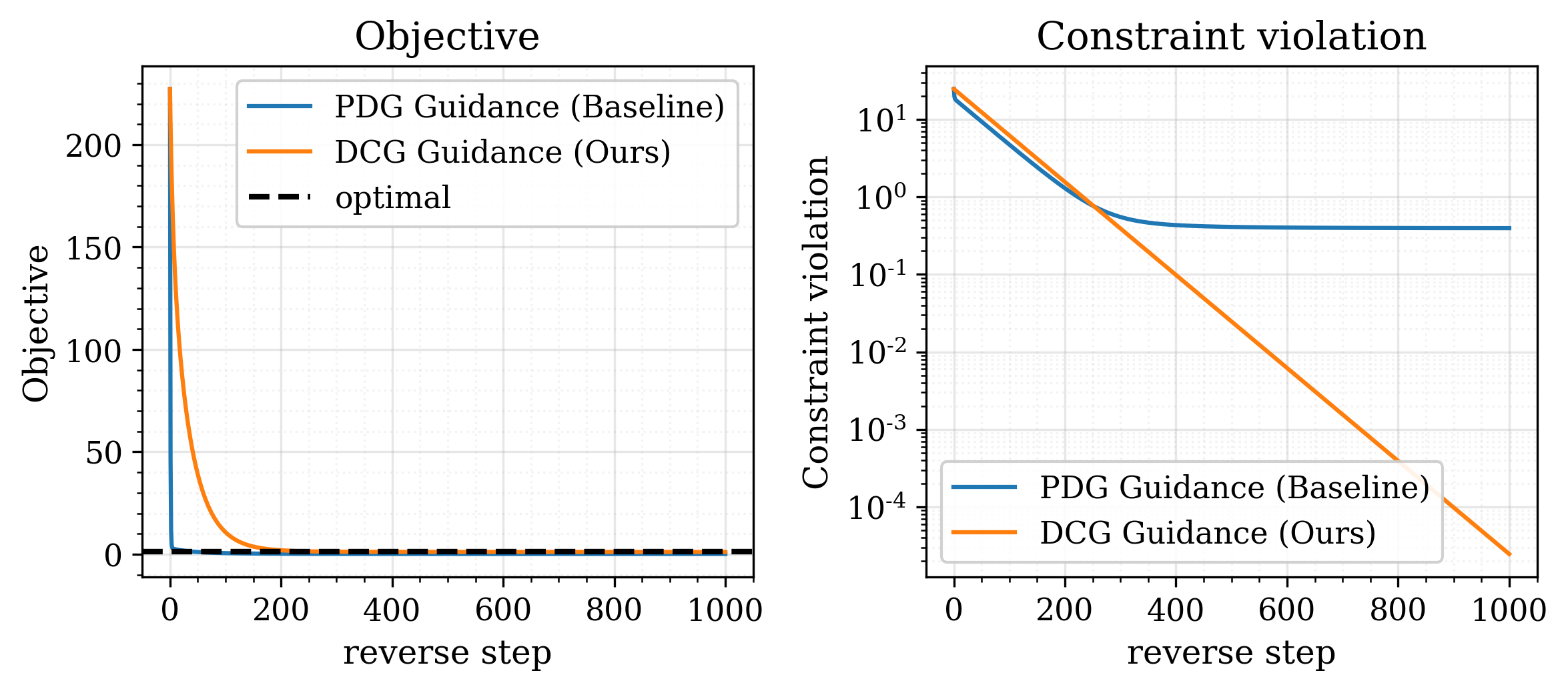}
    \end{subfigure}

    \vspace{0.6em}

    % Row B
    \begin{subfigure}{0.4\textwidth}
        \centering
        \includegraphics[width=\linewidth]{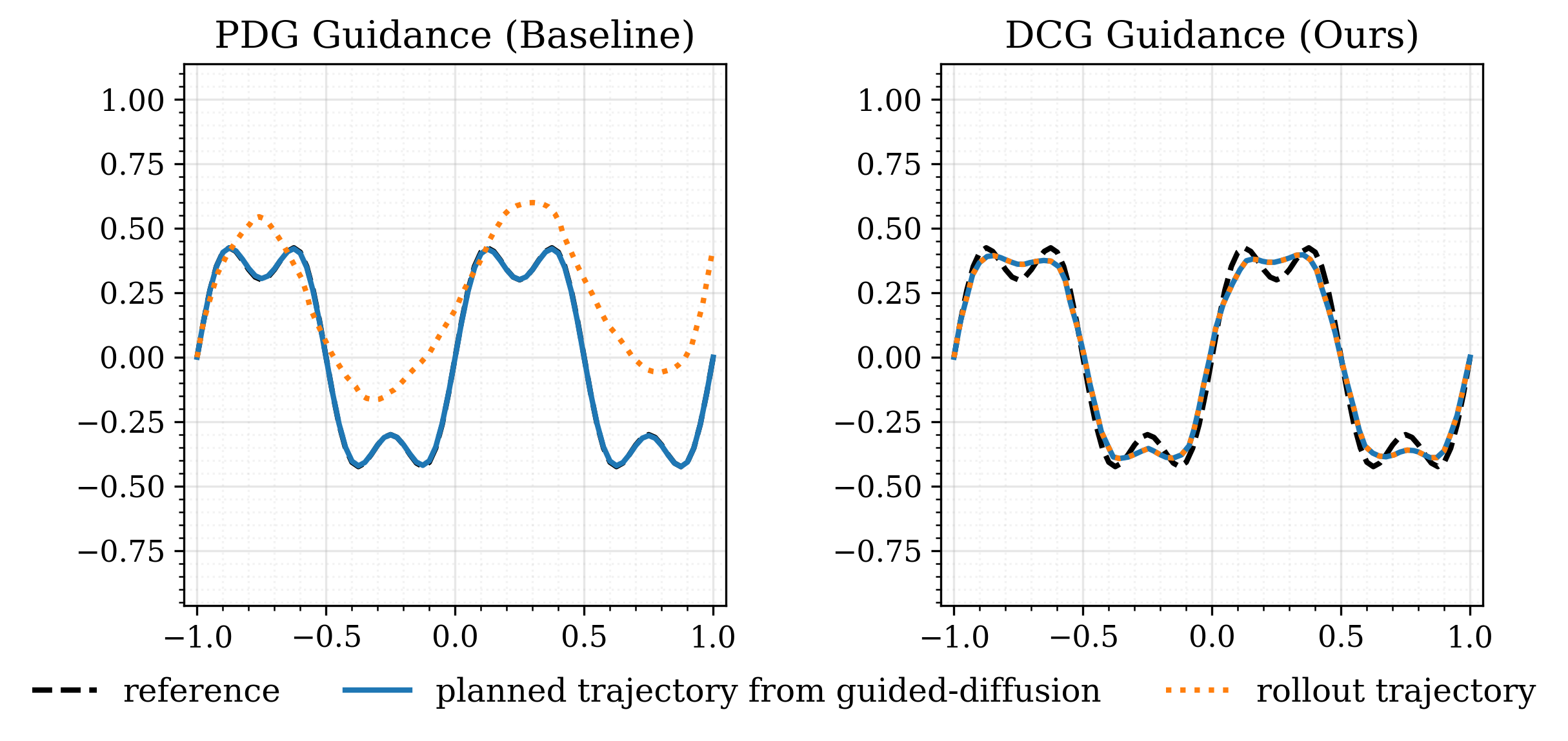}
    \end{subfigure}
    \hspace{10pt}
    \begin{subfigure}{0.43\textwidth}
        \centering
        \includegraphics[width=\linewidth]{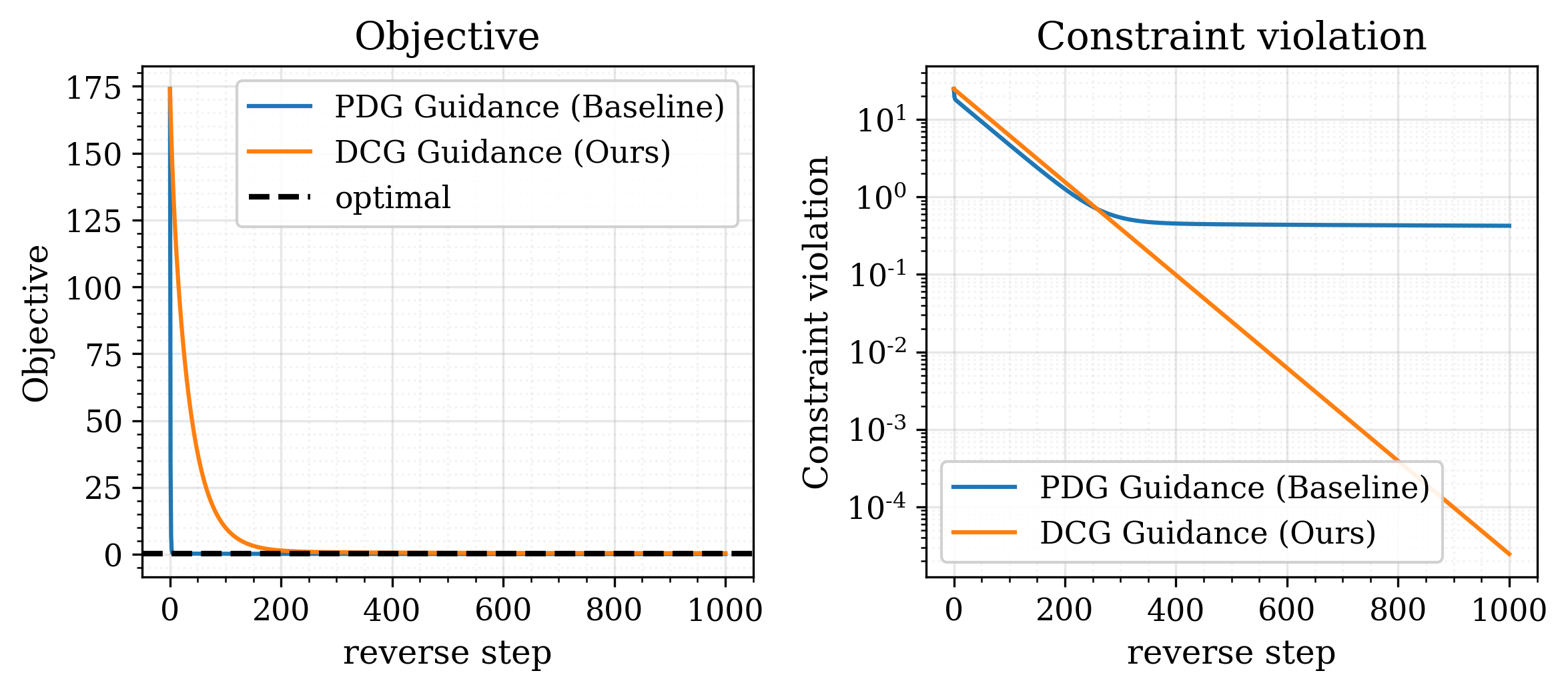}
    \end{subfigure}

    \vspace{0.6em}

    % Row C
    \begin{subfigure}{0.4\textwidth}
        \centering
        \includegraphics[width=\linewidth]{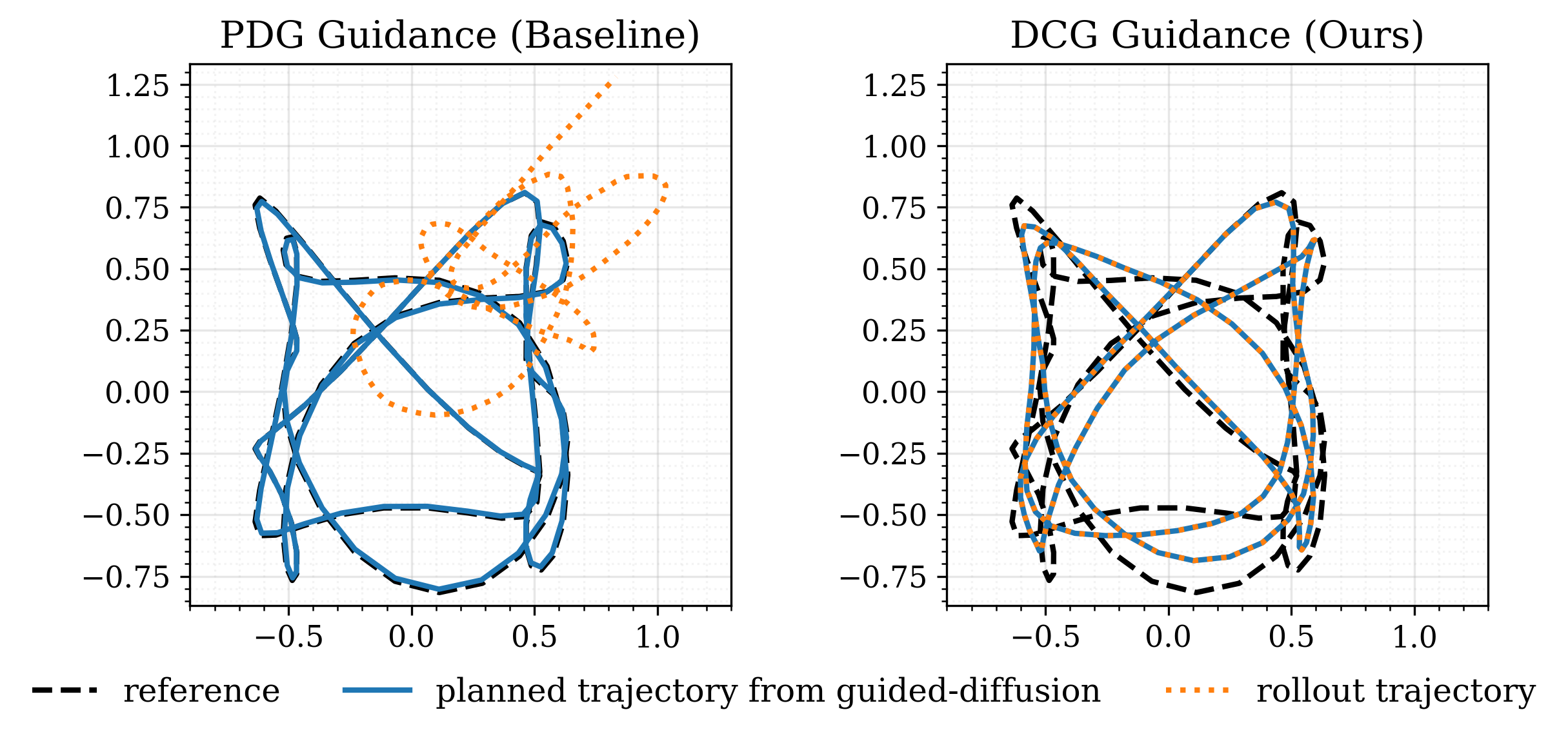}
    \end{subfigure}
    \hspace{10pt}
    \begin{subfigure}{0.43\textwidth}
        \centering
        \includegraphics[width=\linewidth]{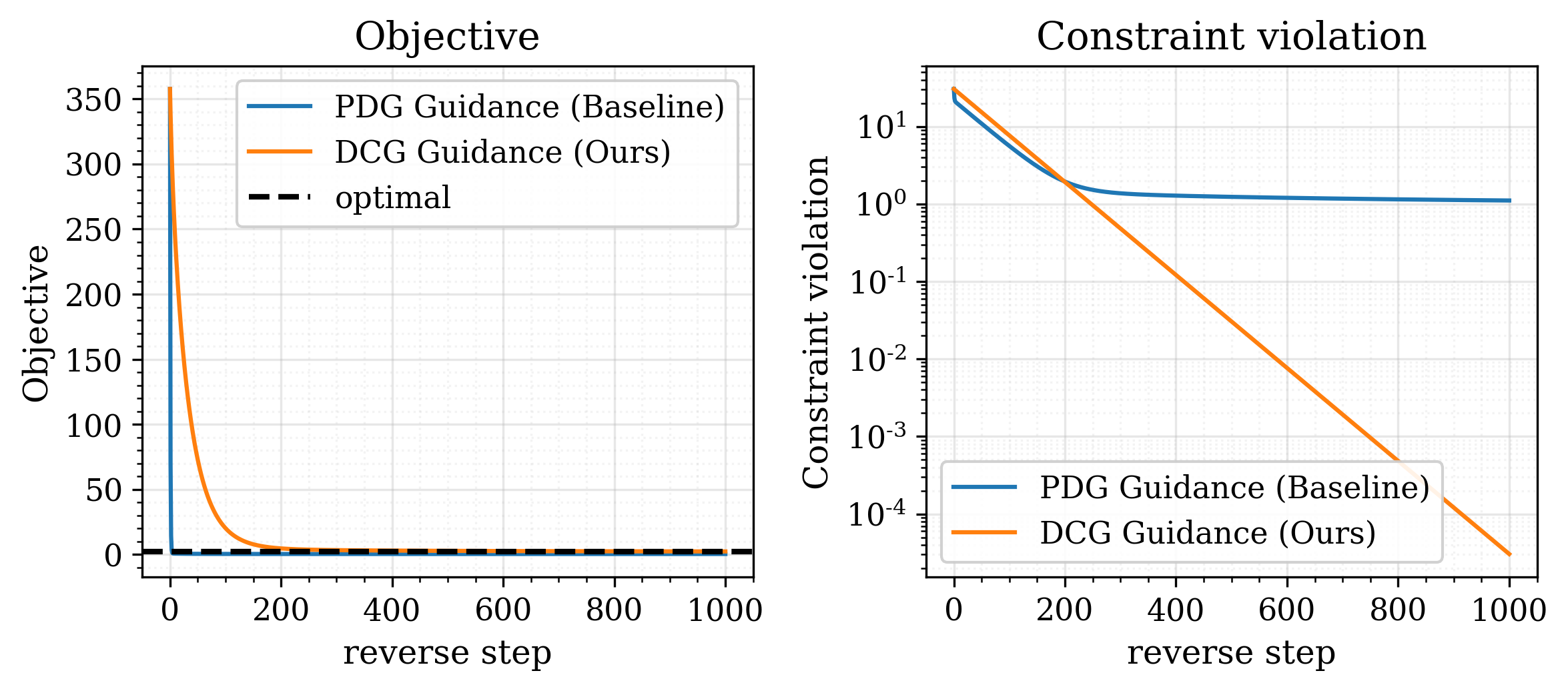}
    \end{subfigure}

    \caption{
Trajectory comparison (left) and convergence (right) for heart-, sinusoidal-, and hard-shaped references.
Black dashed: reference; blue solid: planned trajectory from guided diffusion; orange dotted: rollout from the planned control.
%PDG achieves close reference tracking in plan space but produces poor rollouts due to dynamic inconsistency, while DCG maintains agreement between planned and rolled-out trajectories and drives constraint violation to near zero.
}
\label{fig:traj-tracking-linear-comparison}
\end{figure}

To verify our algorithm for the linear Gaussian setting as in Section \ref{sec:linear-analysis}, we evaluate gradient-guided deterministic backward diffusion on a trajectory following problem. We consider a two-dimensional discrete-time
double-integrator system
\begin{equation*}
    p_{k+1} = p_k + \Delta t\, v_k, 
    \qquad
    v_{k+1} = v_k + \Delta t\, u_k,
    \qquad k = 0,\ldots,N-1,
\end{equation*}
where $p_k, v_k, u_k \in \mathbb{R}^2$ denote position, velocity, and control,
respectively. The full trajectory is stacked into a vector
\begin{equation*}
    x = (p_0, v_0, u_0, \ldots, p_{N-1}, v_{N-1}, u_{N-1}, p_N, v_N).
\end{equation*}
Hence the linear dynamics and boundary conditions are encoded as affine equality
constraints $Ax = b$.
The tracking objective is
\begin{equation}
\textstyle    f(x)
    =
    \sum_{k=0}^{N} \|p_k - p_k^{\mathrm{ref}}\|_2^2
    +
    \rho \sum_{k=0}^{N-1} \|x_k\|_2^2,
\end{equation}
where $p_k^{\mathrm{ref}}$ is the reference trajectory.
% We compare the guided diffusion output with the exact constrained least-squares
% solution. 
% The optimality gap and the data manifold constraint violation are reported as in Table \ref{tab:double-integrator-results}.
We consider three reference trajectories: a heart-shaped curve, a sinusoidal
curve, and a more challenging high-frequency Lissajous-type curve. We use $\rho = 10^{-2}$ for the heart example and $\rho = 10^{-3}$ for the
sinusoidal and hard examples.
In all experiments we use $\Delta t = 0.1$, guidance stepsize
$\eta = 20$, and $T = 1000$ reverse diffusion steps.
The noise schedule is geometrically spaced with
$\sigma_{\min}=10^{-6}$ and $\sigma_{\max}=1$.
The planning horizons are $N=80$ (heart), $N=40$ (sinusoidal), and
$N=60$ (hard).

Note that in the linear Gaussian setting, the score function can be computed analytically based on Lemma \ref{lemma:pi-t} (in this experiment we set the data covariance to $\Sigma = I$), hence in the section we provide the numerical result under the perfectly computed score function via the analytical formula.

%The hard
%example contains multiple frequency components and sharper changes in direction,
%making it more difficult to track under the double-integrator dynamics.

Figure~\ref{fig:traj-tracking-linear-comparison} compares PDG guidance (baseline) and DCG guidance (ours) on three reference trajectories. Here the constraint violation is the $\ell_2$ norm of the double-integrator feasibility residual $\|Ax-b\|_2$. In each case, PDG produces a planned trajectory that closely follows the reference, but the corresponding rollout deviates substantially once the controls are executed through the double-integrator dynamics. This mismatch indicates that PDG optimizes tracking in state space without enforcing dynamic consistency, so low tracking error in the planned solution does not translate into reliable closed-loop behavior. In contrast, DCG yields planned trajectories that remain consistent with their rollouts: the executed trajectory nearly coincides with the guided-diffusion plan across all three settings. This comes at the cost of slightly higher tracking error relative to the reference, but it ensures that the recovered solution is dynamically feasible. The convergence plots reinforce this distinction. PDG rapidly drives down the objective, yet its constraint violation remains large throughout denoising, whereas DCG converges to solutions with near-zero constraint violation while approaching the optimal objective. Together, these results show that dynamic consistency is the key advantage of DCG: it recovers plans that can actually be executed, rather than reference-fitting trajectories that violate the system dynamics.

\subsection{Convex constrained}

To test our result for the convex region setting in Section \ref{sec:convex-analysis}, we construct two synthetic settings where we solve $\min_{x\in\cX} f(x)$ via diffusion, where $x\in\bR^{10}$ and $f(x)$ is a quadratic cost that takes the form of $f(x) = \|x-x_{\mathrm{tgt}}\|^2$ for certain $x_{\mathrm{tgt}}$ that lies outside of $\cX$. We consider two different settings for data distribution on the convex region $\cX$.

\paragraph{Test A (square).}
Data are drawn uniformly from a two-dimensional subspace over the box $\{|x_1| \le 1,\ |x_2| \le 1\}$.
We use PDG with $w_{\mathrm{cg}} = 0.6$ and DCG with optimization scale $\eta = 15$.

% PDG consistently steers samples toward the external target and frequently leaves the feasible box, yielding high violation and large path-to-path variability.
% DCG improves the objective by approximately $1.0$ while reducing violation by a factor of $\approx 4.5$, producing trajectories that remain near the constrained optimum with low variance across initializations.
% Figure~\ref{fig:test-a-paths} visualizes the reward landscape, feasible region, backward paths, and the constrained optimum (red star); Figure~\ref{fig:test-a-conv} reports objective and violation along reverse diffusion.

\paragraph{Test B (ellipsoid).}
Data are uniform on a five-dimensional subspace inside the ellipsoid $\sum_{i=1}^{5} (x_i / a_i)^2 \le 1$ with axes $a = (1.0,\, 0.85,\, 0.7,\, 0.55,\, 0.4)$.
We use PDG with $w_{\mathrm{cg}} = 10$ and DCG with $\eta = 10$.

Both tests use a JannerUNet1d backbone ($d_{\mathrm{model}}=64$), cosine VP schedule, $50$ diffusion training/sampling steps, DDIM sampling, $200{,}000$ training points, and $240{,}000$ optimizer steps on Gaussian-normalized data. Although the analysis in Section~\ref{sec:convex-analysis} is stated for a geometric VE schedule, the experiments use the cosine VP schedule commonly adopted in diffusion implementations. The theoretical guarantees therefore do not apply verbatim to this setting; rather, these experiments test whether the predicted denoising-as-projection behavior persists beyond the analyzed schedule, which we observe empirically for both convex regions.

\begin{figure}[htbp]
  \centering
  \begin{minipage}{0.25\linewidth}
    \centering
    \includegraphics[width=\linewidth]{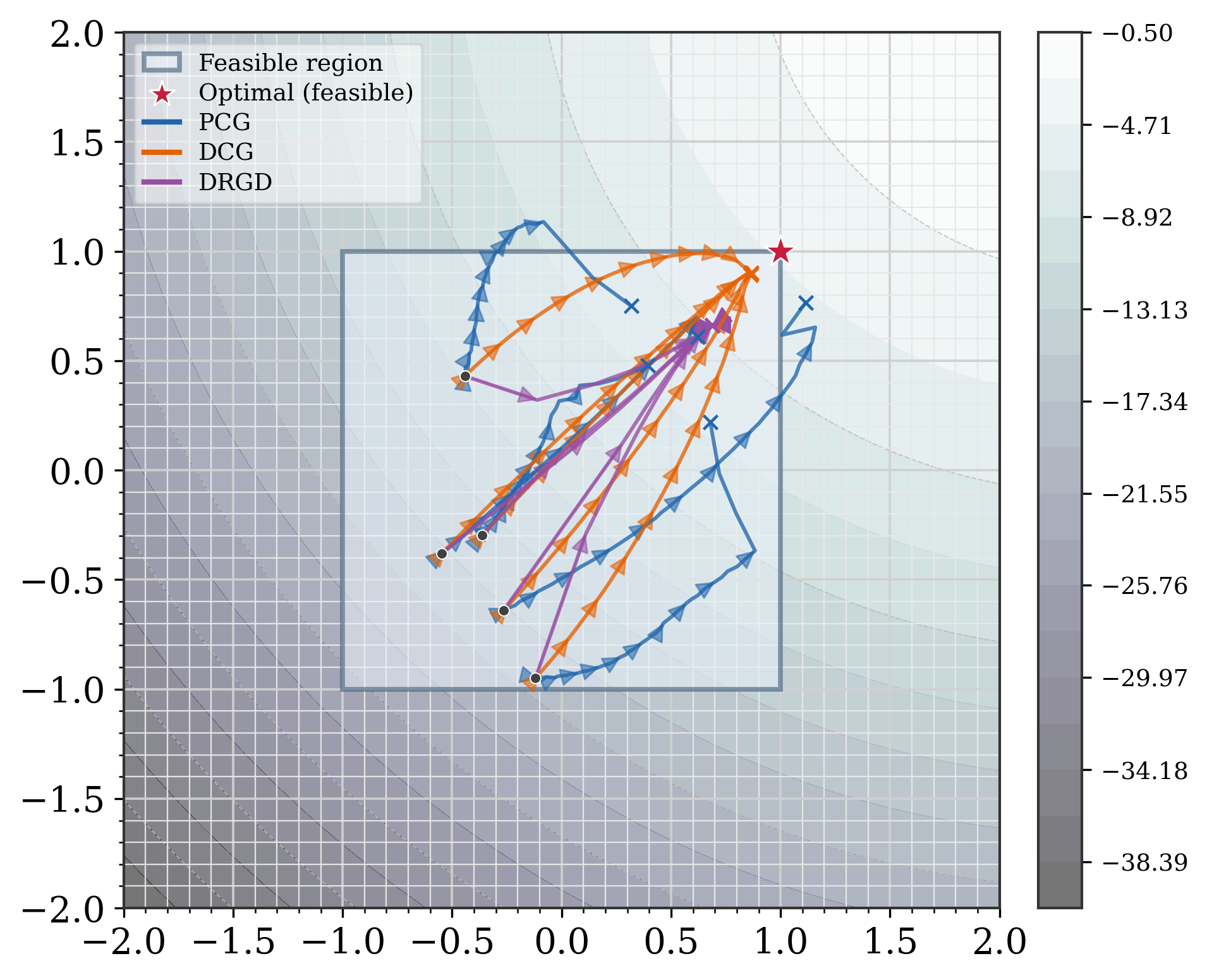}
  \end{minipage}\hspace{5pt}
  \begin{minipage}{0.4\linewidth}
    \centering
    \includegraphics[width=\linewidth]{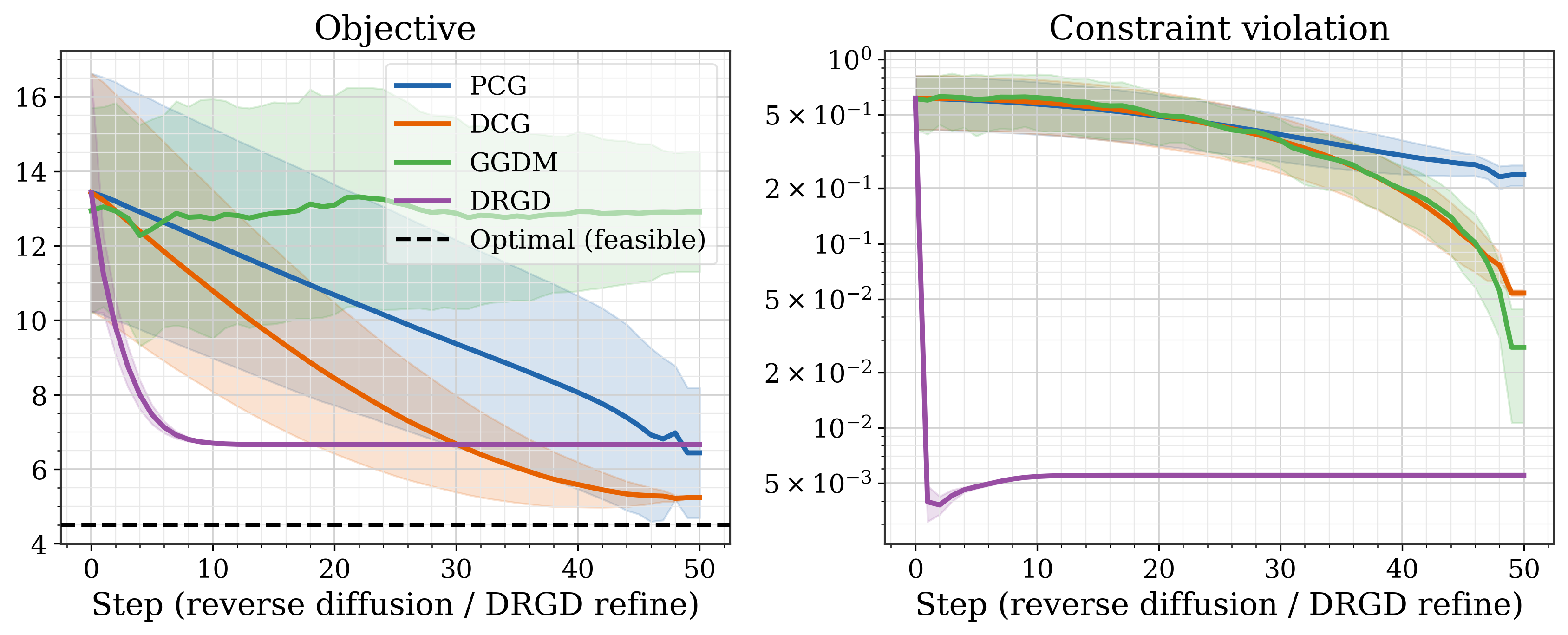}
  \end{minipage}
\vspace{0.6em}

  \begin{minipage}{0.25\linewidth}
    \centering
    \includegraphics[width=\linewidth]{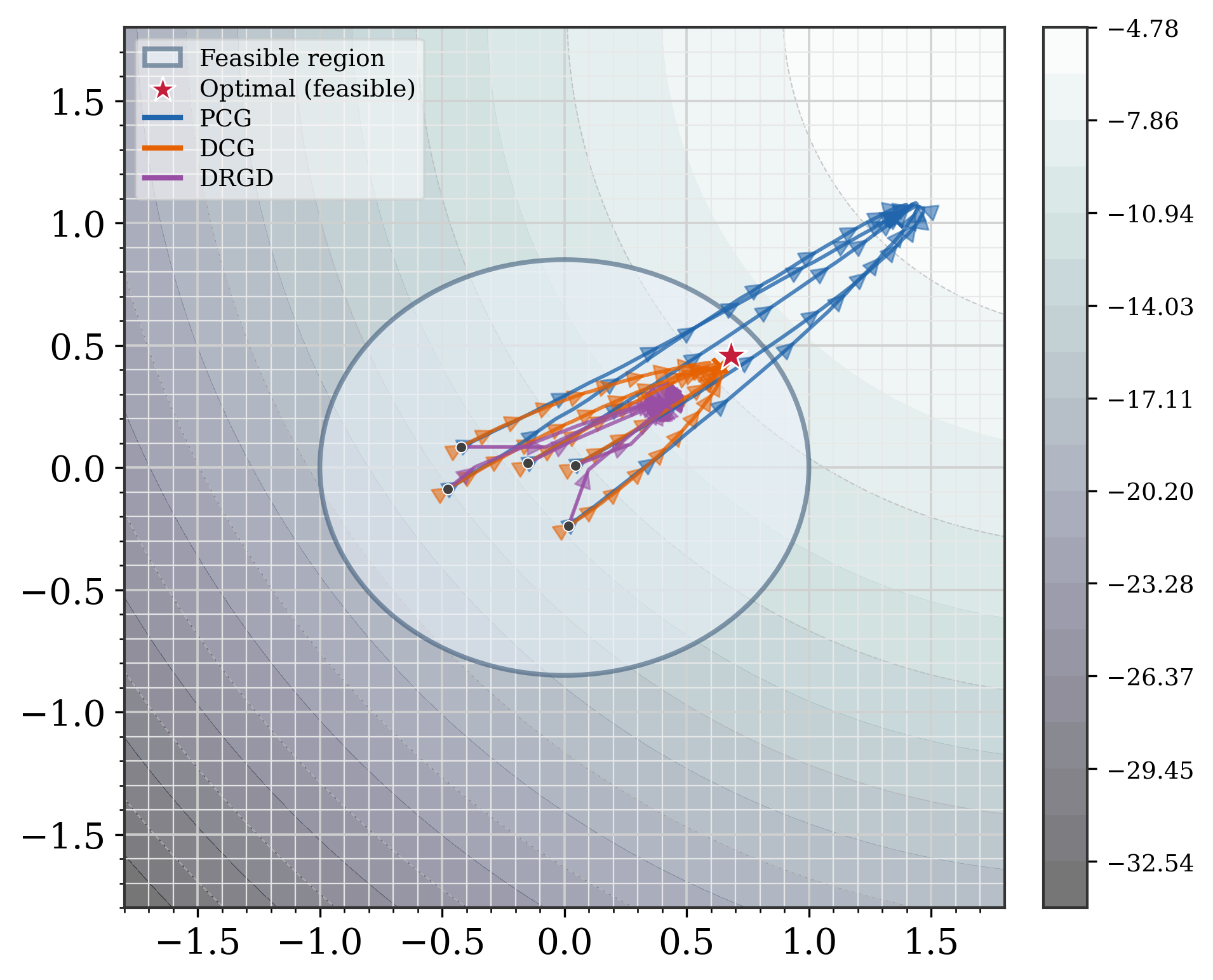}
  \end{minipage}\hspace{5pt}
  \begin{minipage}{0.4\linewidth}
    \centering
    \includegraphics[width=\linewidth]{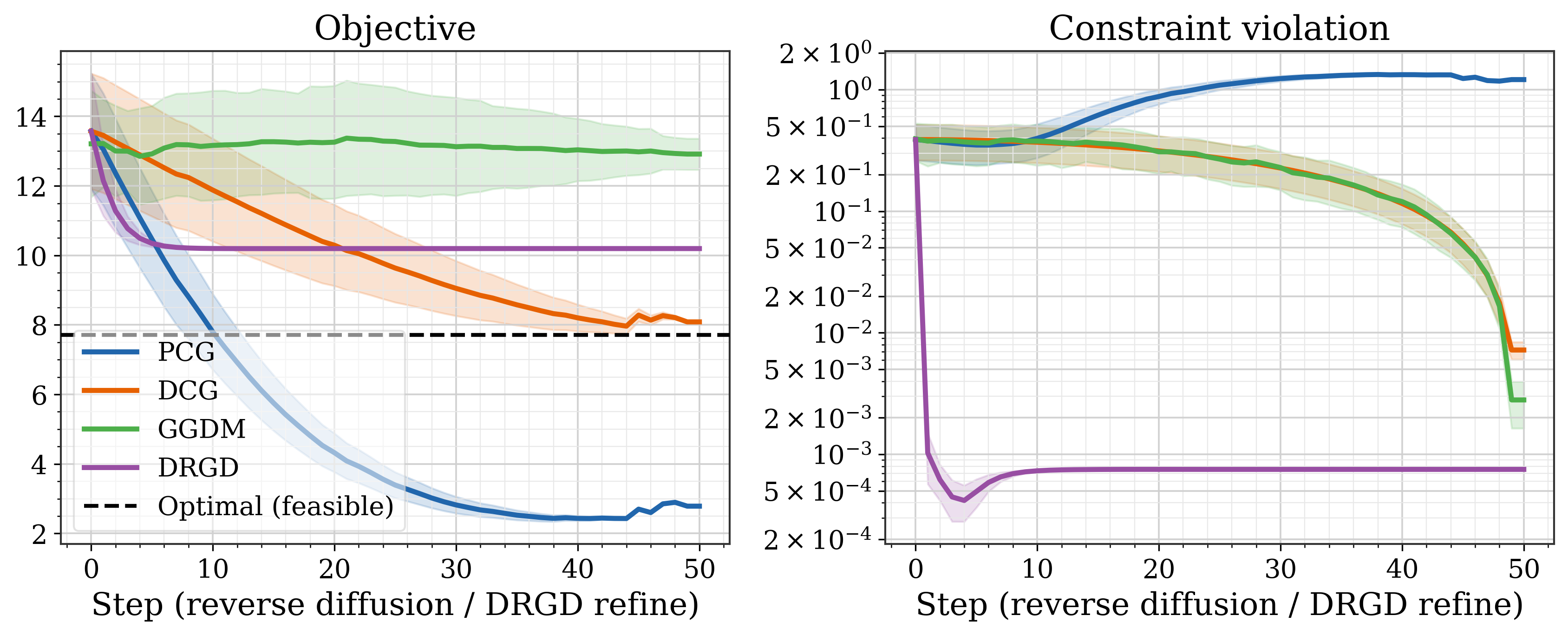}
  \end{minipage}
  \caption{Left: backward diffusion paths in the $(x_1,x_2)$ plane over the reward landscape and feasible region; right: convergence of the objective and constraint violation along reverse steps. (Since the resulting
sample paths for GGDM are highly noisy and visually cluttered, we omit GGDM from
the left trajectory plots for clarity and report it only in the
convergence curves on the right.) Top: Test A; bottom: Test B.}
\label{fig:synthetic-guidance}
\end{figure}

Apart from the PDG baseline, we additionally compare against two most relevant baselines from prior work GGDM \cite{guo2024gradient} and DRGD \cite{kharitenko2025landing}. Figure~\ref{fig:synthetic-guidance} visualizes the results. In each
panel, the left plot shows the backward-diffusion trajectories projected
onto the subspace coordinates $(x_1,x_2)$, together with the reward
landscape and the feasible training-support region. The right plot
reports the corresponding objective value and constraint violation
along the reverse-diffusion or refinement steps, with the dashed line
indicating the objective value at the constrained optimum $x^\star$.
The curves report the mean, with shaded regions indicating one standard
deviation, over $100$ independently initialized runs.

For the PDG baseline, we intentionally present two different operating
regimes to illustrate its optimality--feasibility trade-off. The
guidance parameters are carefully selected after tuning to provide
representative behavior: Test~A uses the smaller scale
$w_{\mathrm{cg}}=0.6$ to place greater emphasis on feasibility, whereas
Test~B uses the larger scale $w_{\mathrm{cg}}=10$ to place greater
emphasis on objective improvement. In Test~A, PDG keeps the trajectories
closer to the feasible box, but achieves weaker objective improvement
and exhibits substantial variability across initializations. In
Test~B, stronger guidance produces more consistent progress toward the
target, but drives the trajectories outside the feasible ellipsoid.
Together, these two panels demonstrate the difficulty of simultaneously
achieving strong objective improvement and low constraint violation
with PDG: weaker guidance favors feasibility at the expense of
optimality, while stronger guidance improves the objective at the
expense of feasibility.

GGDM fails to converge toward the constrained optimum in either test: its objective quickly stagnates despite a gradual reduction in constraint violation. We conjecture that this is partly due to a mismatch between the present nonlinear constrained setting and the linear-Gaussian setting underlying much of GGDM's design and analysis. Its reliance on repeated linearization and several sensitive hyperparameters also makes effective tuning difficult. Because the resulting trajectories are highly noisy and visually cluttered, we omit GGDM from the left trajectory plots and report it only in the convergence plots.

DRGD reaches a nearly feasible solution within only a few refinement steps, consistent with our discussion in the Riemannian setting that direct refinement with a low-noise denoiser can converge faster than guidance distributed across the reverse process. However, it consistently stabilizes at a suboptimal objective value. We conjecture that repeatedly using only the smallest-noise denoiser introduces a persistent finite-noise projection bias, causing convergence to a biased stationary point.

In contrast, DCG achieves the best overall optimality--feasibility balance in both tests. It attains substantially better objective values than DRGD while maintaining much smaller constraint violation than PDG, and its performance is comparatively robust to the choice of guidance scale.

\subsection{Control and Planning of Nonlinear Dynamical Systems}

\subsubsection{Unicycle trajectory tracking}
\label{sec:unicycle}

\paragraph{Dynamics and state representation.}
We consider a planar unicycle with discrete-time dynamics
\begin{align}
    x_{t+1} &= x_t+\Delta t\,v_t\cos\theta_t, &
    y_{t+1} &= y_t+\Delta t\,v_t\sin\theta_t, &
    \theta_{t+1} &= \theta_t+\Delta t\,\omega_t,
\end{align}
where $\Delta t=0.1\,\mathrm{s}$. The physical state is
$s_t=(x_t,y_t,\theta_t)$, while the diffusion model uses
$o_t=[x_t,y_t,\cos\theta_t,\sin\theta_t]\in\mathbb{R}^4$ to avoid the angular
discontinuity at $\theta=\pm\pi$. The control is
$a_t=[v_t,\omega_t]\in\mathbb{R}^2$, with
$v_t\in[0,2]$ and $\omega_t\in[-1.5,1.5]$. The workspace is
$x,y\in[-8,8]$, and trajectories leaving the workspace are terminated during
data collection.

\paragraph{Trajectory generation and conditioning.}
At inference, the diffusion model generates a horizon-$H$ joint trajectory
$\tau=(o_{0:H-1},a_{0:H-1})\in\mathbb{R}^{H\times 6}$ with $H=64$,
corresponding to $6.4\,\mathrm{s}$. The initial observation $o_0$ is fixed to
the current observation through a \texttt{fix\_mask}, while
$(o_{1:H-1},a_{0:H-1})$ are generated. Generated actions are clipped to their
admissible bounds before rollout.

\paragraph{Offline dataset and diffusion pretraining.}
The diffusion model is trained without task-specific costs on an offline
dataset of unicycle trajectories, where each data point is a complete segment
$\tau=\{(o_t,a_t)\}_{t=0}^{H-1}$. The data are obtained by rolling out smooth
random controls constructed from low-frequency sinusoidal modes, temporal
smoothing, and occasional local perturbations. We use a Janner-style temporal
U-Net over the six trajectory channels, train for $500{,}000$ gradient steps
with cosine learning-rate decay and increased weight on the initial action,
and sample using DDIM with $20$ reverse steps.

\paragraph{Reference-tracking cost.}
For each test instance, we specify a reference path
$\{(x_t^\star,y_t^\star)\}_{t=0}^{H-1}$ and define
\begin{equation}
 \textstyle   c(\tau)
    =
    \sum_{t=0}^{H-1}
    \left(
        \lVert x_t-x_t^\star\rVert^2
        +
        \lVert y_t-y_t^\star\rVert^2
    \right).
    \label{eq:unicycle-tracking-cost}
\end{equation}
At inference, gradient guidance uses $\nabla c(\tau)$ to steer the reverse
diffusion process toward trajectories that track the reference path.

\paragraph{Planned and executed trajectories.}
Let
$\widehat{\tau}=\{(\widehat{o}_t,\widehat{a}_t)\}_{t=0}^{H-1}$
denote the trajectory generated by diffusion. Because its planned observations
and actions need not be dynamically consistent, we also roll out the generated
actions from the same initial state according to
$\widetilde{o}_0=\widehat{o}_0$ and
$\widetilde{o}_{t+1}=F(\widetilde{o}_t,\widehat{a}_t)$, where $F$ is the true
unicycle dynamics. This produces the executed trajectory
$\widetilde{\tau}=\{(\widetilde{o}_t,\widehat{a}_t)\}_{t=0}^{H-1}$.
\paragraph{Trajectory manifold interpretation.}
For a fixed initial observation, the set of dynamically consistent observation--action sequences can be written as
\[
\cM
=
\left\{
\tau:
o_{t+1}=F(o_t,a_t),\ 
(\cos\theta_t)^2+(\sin\theta_t)^2=1,\ 
t=0,\ldots,H-2
\right\}.
\]
Away from the action and workspace boundaries, the unicycle rollout map is smooth, and $\cM$ is the graph of this map over the action sequence; it is therefore a smooth embedded submanifold of the ambient trajectory space. The offline data are supported on a compact portion of this manifold, and the pretrained denoiser is expected to approximate a local projection toward this support. Consequently, this experiment partially reflects the mechanism analyzed in Theorem~\ref{thm:riemann-backward-convergence}.%: gradient guidance improves the tracking objective in the ambient trajectory space, while subsequent denoising restores approximate dynamic consistency. The correspondence is not literal at trajectories reaching the hard action or workspace boundaries, where the feasible set may have a boundary, or when denoising and sampling errors prevent the learned denoiser from exactly realizing the nearest-point projection.
\paragraph{Evaluation metrics.}
We report the \emph{planned cost}
$c_{\mathrm{plan}}:=c(\widehat{\tau})$, the \emph{rollout cost}
$c_{\mathrm{rollout}}:=c(\widetilde{\tau})$, and the dynamic-feasibility error $ e_{\mathrm{dyn}}
    :=
    \sum_{t=0}^{H-1}
    \left\|
        \widehat{o}_t-\widetilde{o}_t
    \right\|_2^2.$
The first two measure tracking quality before and after execution, respectively,
while $e_{\mathrm{dyn}}$ measures consistency between the generated trajectory
and the true dynamics.

\paragraph{Baselines.}
We compare DCG with three baselines using the same pretrained diffusion model and DDIM sampling procedure. \emph{Monte Carlo} denotes unguided diffusion sampling: it measures the tracking performance obtained by drawing directly from the learned trajectory distribution without using the reference cost. \emph{PDG} applies standard gradient guidance by directly incorporating the tracking-cost gradient into each reverse-diffusion step. It can strongly reduce the cost of the generated state sequence, but does not explicitly account for whether the resulting states and actions remain dynamically consistent. \emph{DRGD} is the denoising-based Riemannian-gradient method of~\cite{kharitenko2025landing}, which uses score or denoising information to construct the geometric operations required for optimization over the learned trajectory manifold. In contrast, DCG directly applies a tracking-gradient step followed by the pretrained denoising map. All methods are evaluated using 100 generated trajectories for each reference.
For the qualitative visualization, we show the Monte Carlo sample with the top-4 lowest planned tracking cost among the 100 generated candidates in Figure \ref{fig:unicycle-reference-plans}. 
\paragraph{Results.}
Table~\ref{tab:unicycle-comparison} reveals a substantial difference between optimizing the \emph{planned} trajectory and producing a trajectory that remains effective after execution. Unguided Monte Carlo sampling generally has low dynamic-feasibility error because it remains close to the training distribution, but it does not reliably track the specified references. PDG achieves very low planned costs, yet much of this apparent improvement disappears when its actions are rolled out through the true dynamics. This discrepancy is also visible in Figure~\ref{fig:unicycle-reference-plans}: the trajectories planned by PDG closely follow the references, while their corresponding executed trajectories can deviate substantially.

DCG provides the strongest overall combination of executable tracking and dynamic consistency. As reported in Table~\ref{tab:unicycle-comparison}, it achieves the lowest rollout cost and dynamic-feasibility error among the guided methods on both random references. The improvement is especially pronounced for Random Trajectory~2, where DCG reduces the rollout cost from $3.8825$ for PDG and $1.2894$ for DRGD to $0.1846$. In the Circle setting, DCG achieves the lowest dynamic-feasibility error among the guided methods, while its rollout cost is slightly higher than that of DRGD. Thus, DRGD is generally comparable to or slightly worse than DCG, whereas DCG performs more consistently across the three references. The representative samples in Figure~\ref{fig:unicycle-reference-plans} further show that DCG produces planned trajectories whose tracking improvement is largely retained after execution. Together, these results demonstrate that DCG effectively steers trajectories toward the reference while preserving the learned relationship between states and controls.
\begin{table}[htbp]
\centering
\caption{Comparison of Monte Carlo, PDG, DRGD and DCG across different reference trajectories. Results are reported as mean $\pm$ standard deviation over 100 samples.}

\label{tab:unicycle-comparison}
\resizebox{\linewidth}{!}{
\begin{tabular}{llccc}
\toprule
Setting & Algorithm
& $c_{\mathrm{plan}}$
& $c_{\mathrm{rollout}}$
& $e_{\mathrm{dyn}}$ \\
\midrule

\multirow{4}{*}{Circle}
& Monte Carlo
& $1.3815 \pm 0.3015$
& $1.4877 \pm 0.2715$
& $0.2222 \pm 0.0650$ \\
& PDG
& $0.0198 \pm 0.0055$
& $1.1694 \pm 0.3203$
& $1.6452 \pm 0.2061$ \\
& DRGD 
& $0.0473 \pm 0.0216$
& $0.1791 \pm 0.0486$
& $0.7057 \pm 0.0557$ \\
& DCG
& $0.0599 \pm 0.0281$
& $0.2222 \pm 0.0947$
& $0.5031 \pm 0.0861$ \\
\midrule

\multirow{4}{*}{Random Trajectory 1}
& Monte Carlo
& $0.5053 \pm 0.3084$
& $0.5724 \pm 0.3238$
& $0.2125 \pm 0.0700$ \\
& PDG
& $0.0045 \pm 0.0032$
& $0.3576 \pm 0.2542$
& $0.8317 \pm 0.2941$ \\
& DRGD 
& $0.0281 \pm 0.0241$
& $0.1259 \pm 0.0301$
& $0.4209 \pm 0.0747$ \\
& DCG
& $0.0171 \pm 0.0129$
& $0.0744 \pm 0.0459$
& $0.3167 \pm 0.0916$ \\
\midrule

\multirow{4}{*}{Random Trajectory 2}
& Monte Carlo
& $6.5428 \pm 0.9921$
& $6.3648 \pm 0.9836$
& $0.2125 \pm 0.0700$ \\
& PDG
& $0.0519 \pm 0.0082$
& $3.8825 \pm 0.5179$
& $2.8682 \pm 0.2486$ \\
& DRGD 
& $0.5722 \pm 0.0790$
& $1.2894 \pm 0.4768$
& $0.8985 \pm 0.2637$ \\
& DCG
& $0.0211 \pm 0.0102$
& $0.1846 \pm 0.0846$
& $0.5870 \pm 0.1445$ \\
\bottomrule
\end{tabular}
}\end{table}

\begin{figure}[htbp]
\centering

% ---------- Legend ----------
\includegraphics[width=0.9\linewidth]{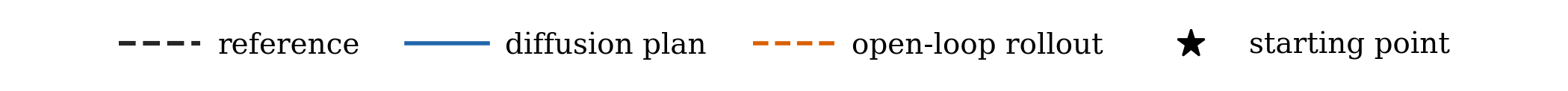}

% ---------- Circle ----------
\textbf{Circle}

\vspace{2mm}

\begin{tabular}{@{}c@{\hspace{3mm}}cccc@{}}
\textbf{PDG} &
\includegraphics[width=0.15\textwidth,valign=c]{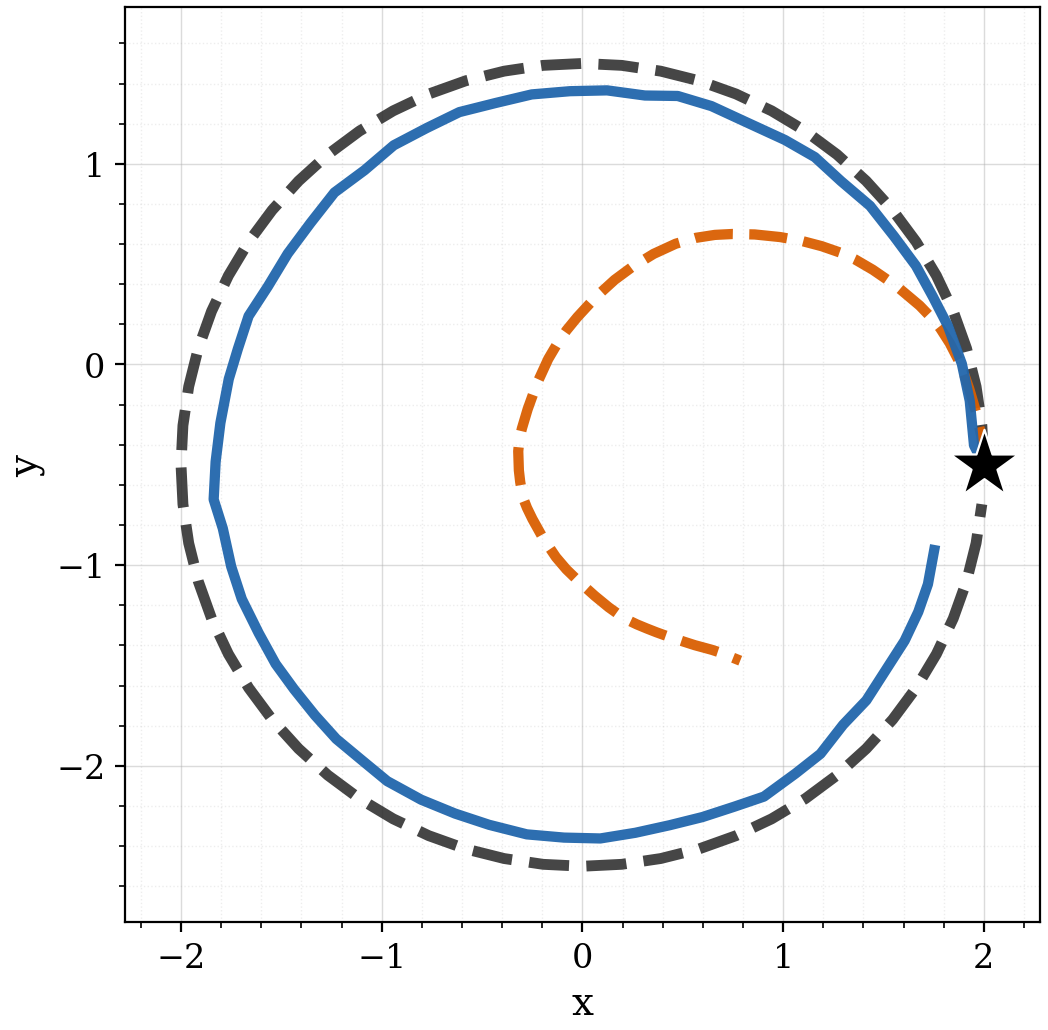} &
\includegraphics[width=0.15\textwidth,valign=c]{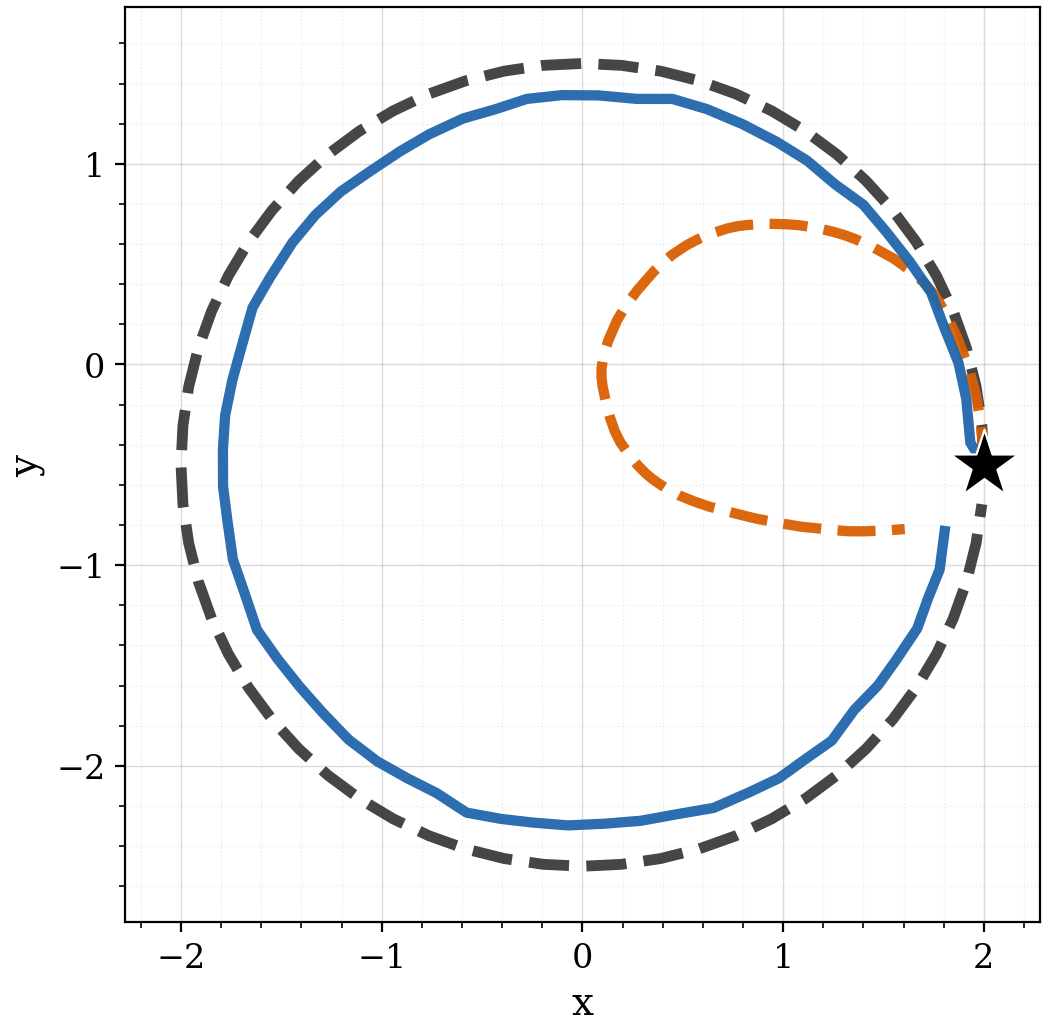} &
\includegraphics[width=0.15\textwidth,valign=c]{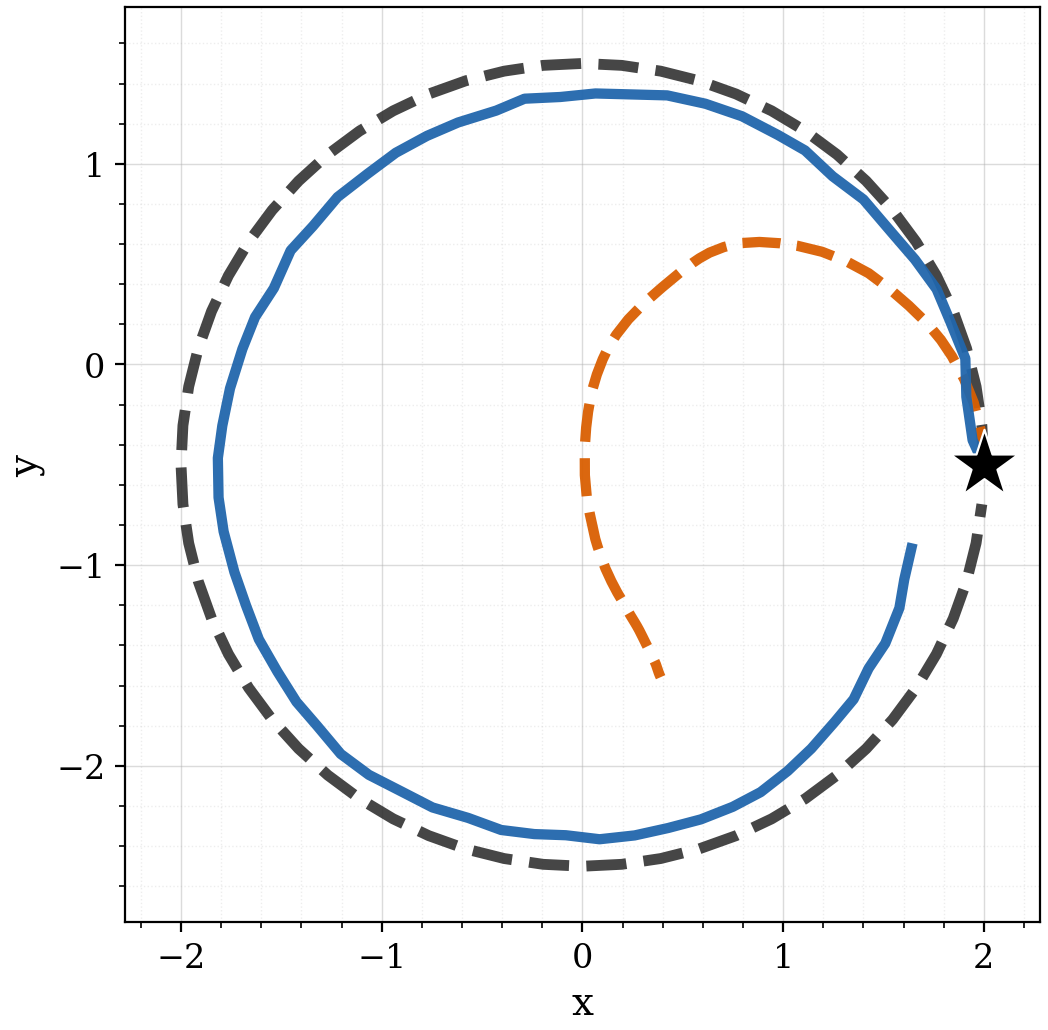} &
\includegraphics[width=0.15\textwidth,valign=c]{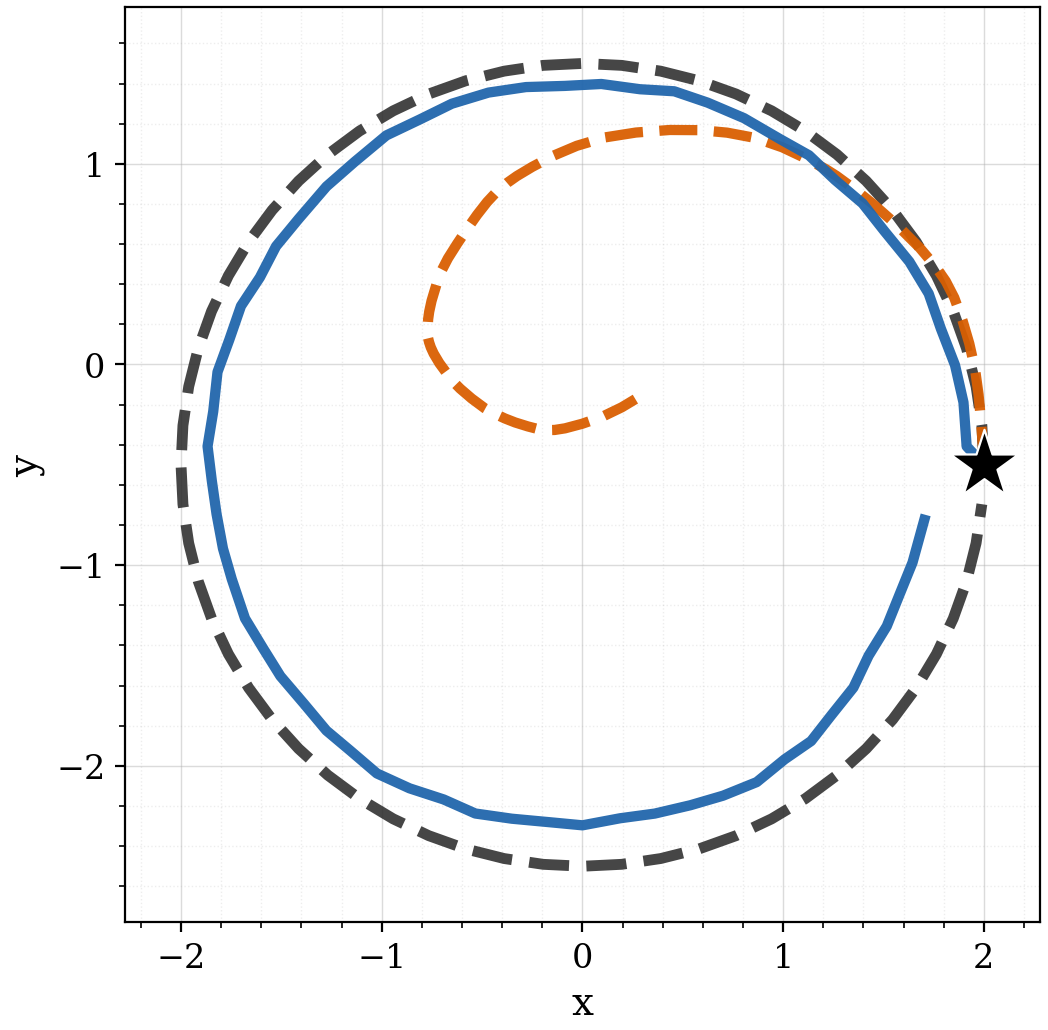} \\[2mm]
\textbf{DRGD} &
\includegraphics[width=0.15\textwidth,valign=c]{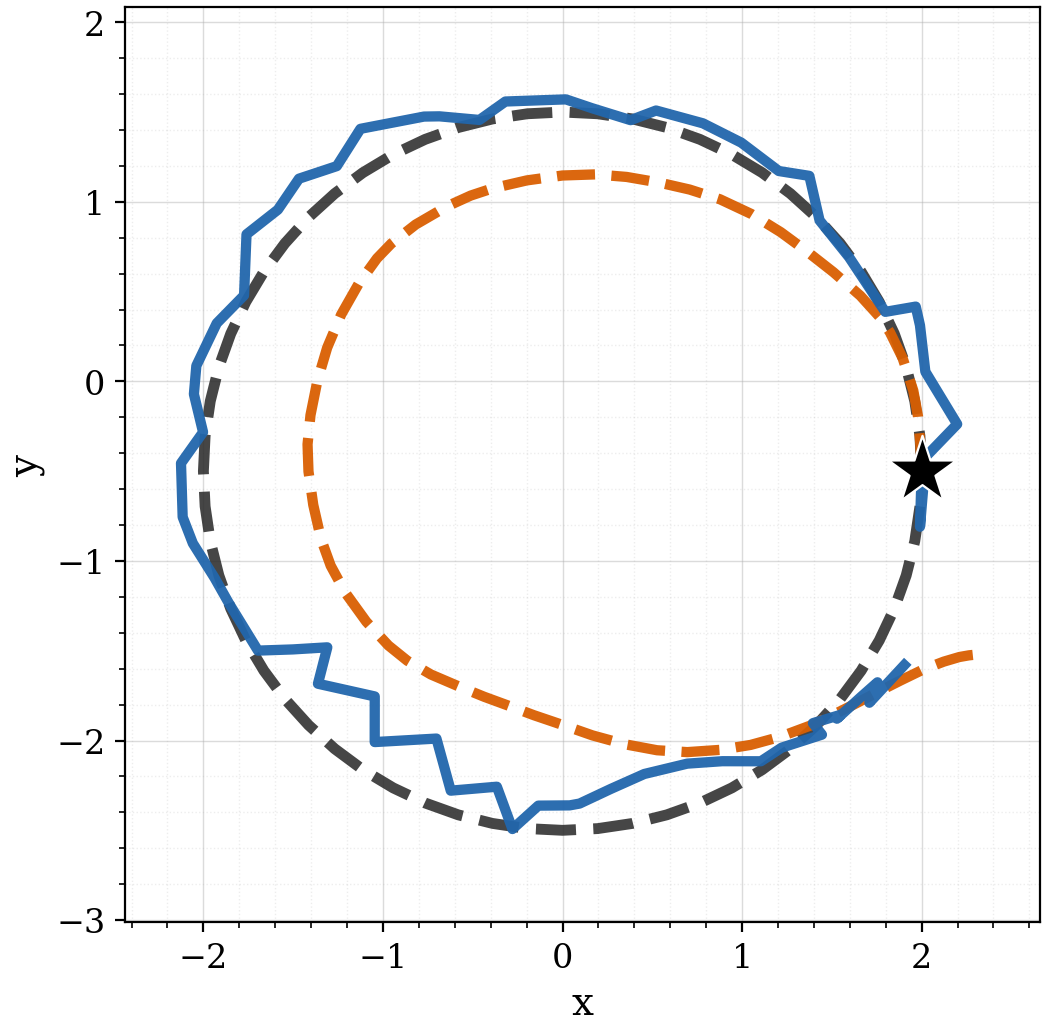} &
\includegraphics[width=0.15\textwidth,valign=c]{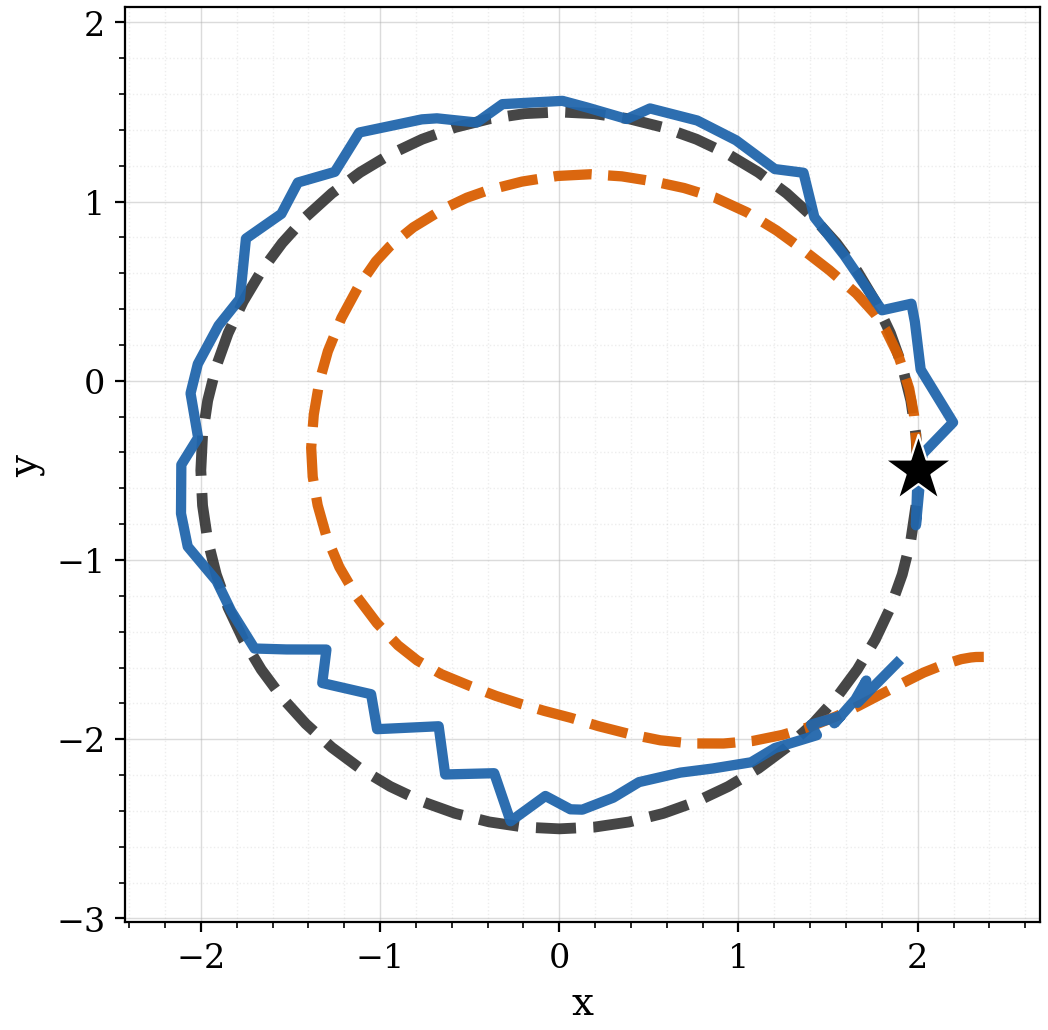} &
\includegraphics[width=0.15\textwidth,valign=c]{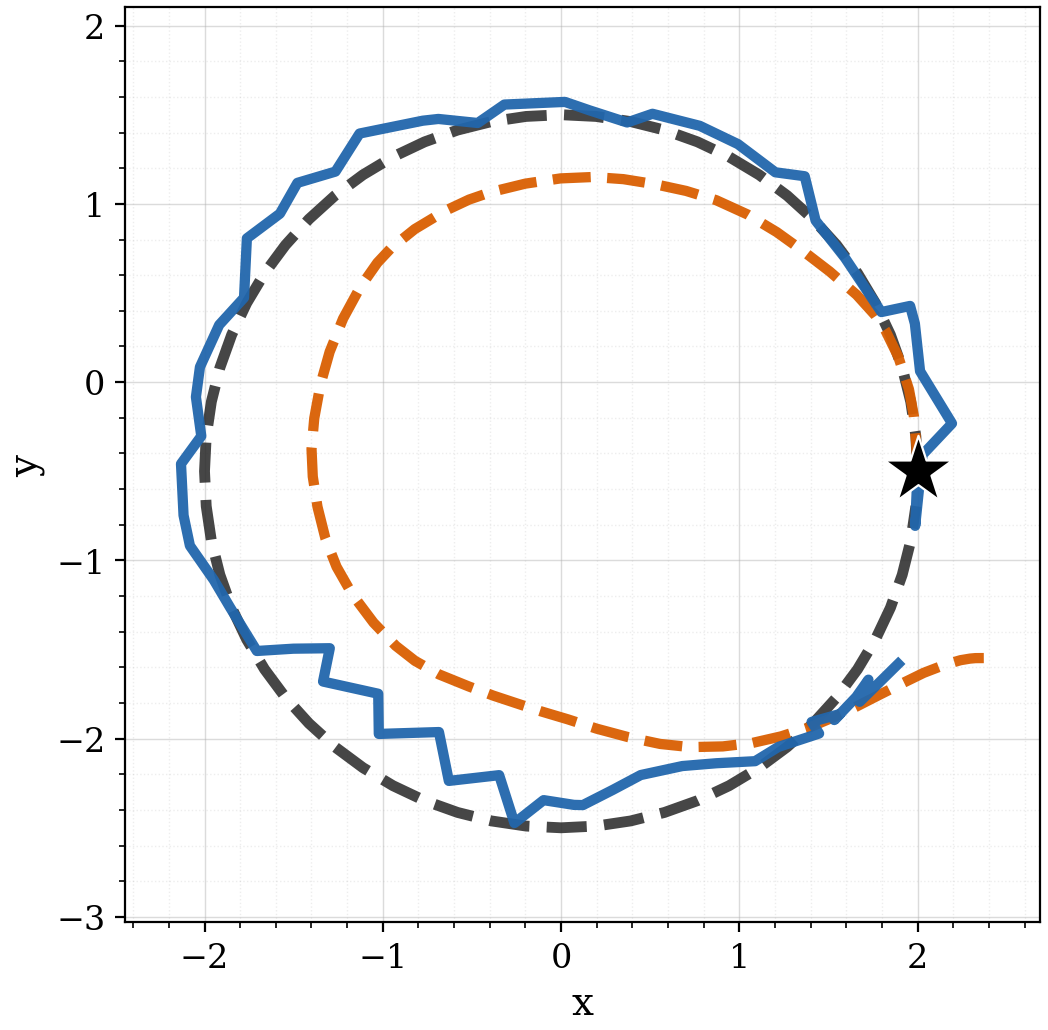} &
\includegraphics[width=0.15\textwidth,valign=c]{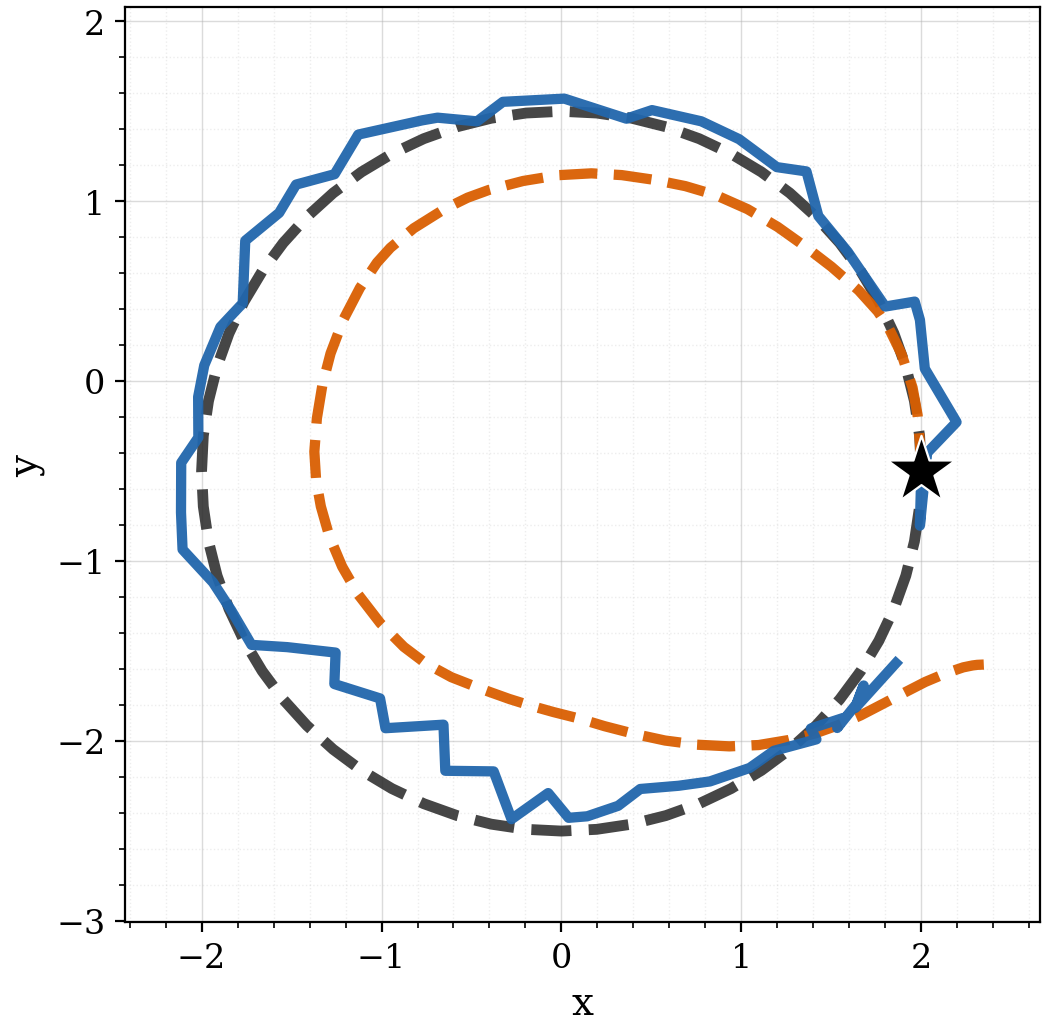} \\[2mm]
\textbf{DCG} &
\includegraphics[width=0.15\textwidth,valign=c]{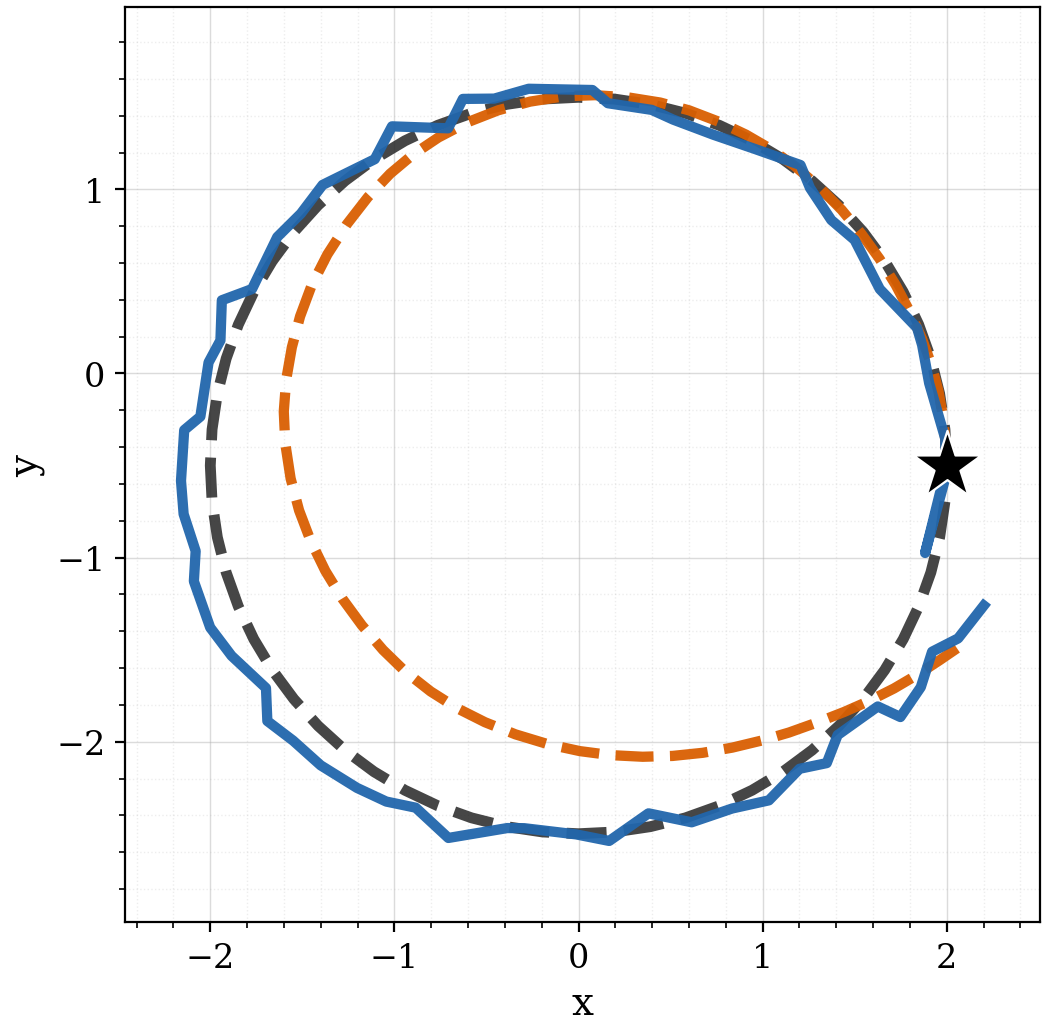} &
\includegraphics[width=0.15\textwidth,valign=c]{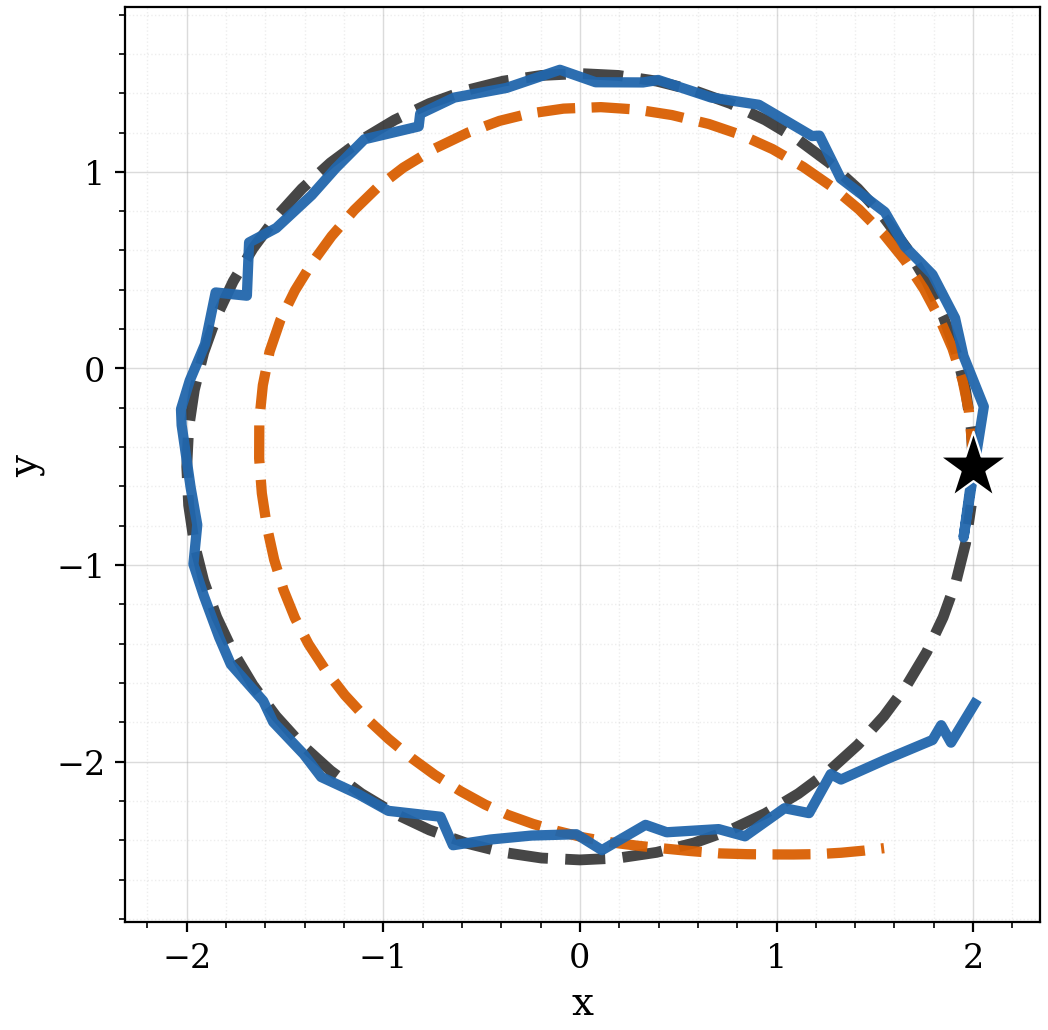} &
\includegraphics[width=0.15\textwidth,valign=c]{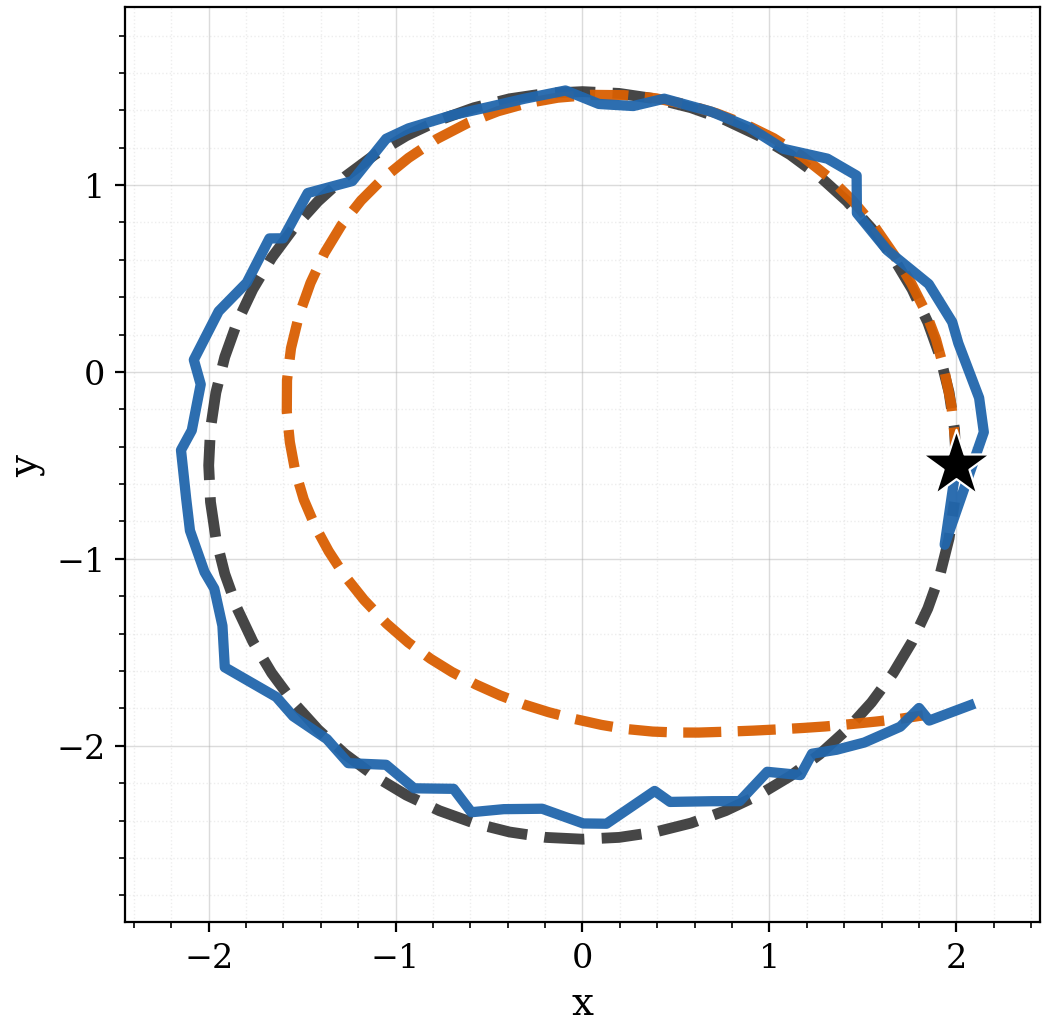} &
\includegraphics[width=0.15\textwidth,valign=c]{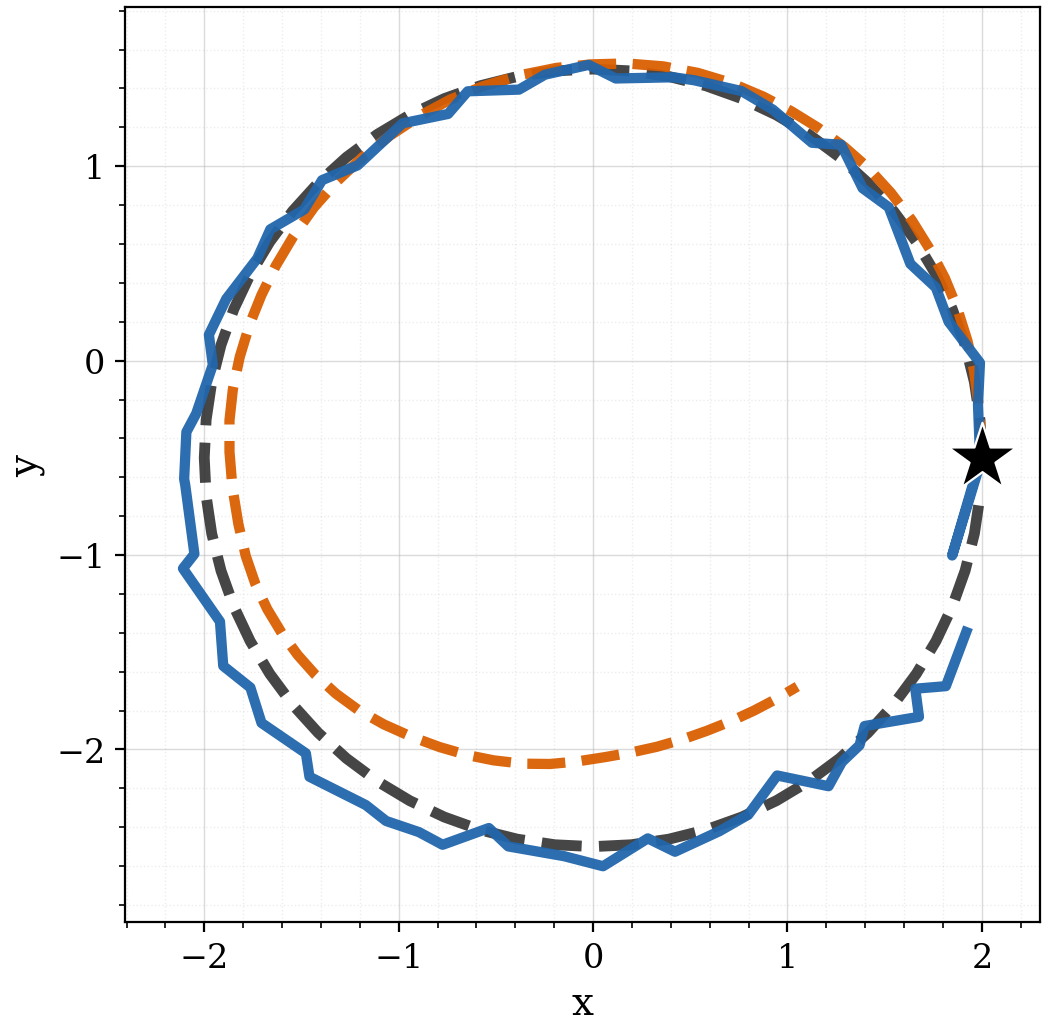}
\end{tabular}

\vspace{2mm}

% ---------- Trajectory 1 ----------
\textbf{Trajectory 1}

\vspace{2mm}

\begin{tabular}{@{}c@{\hspace{3mm}}cccc@{}}
\textbf{PDG} &
\includegraphics[width=0.15\textwidth,valign=c]{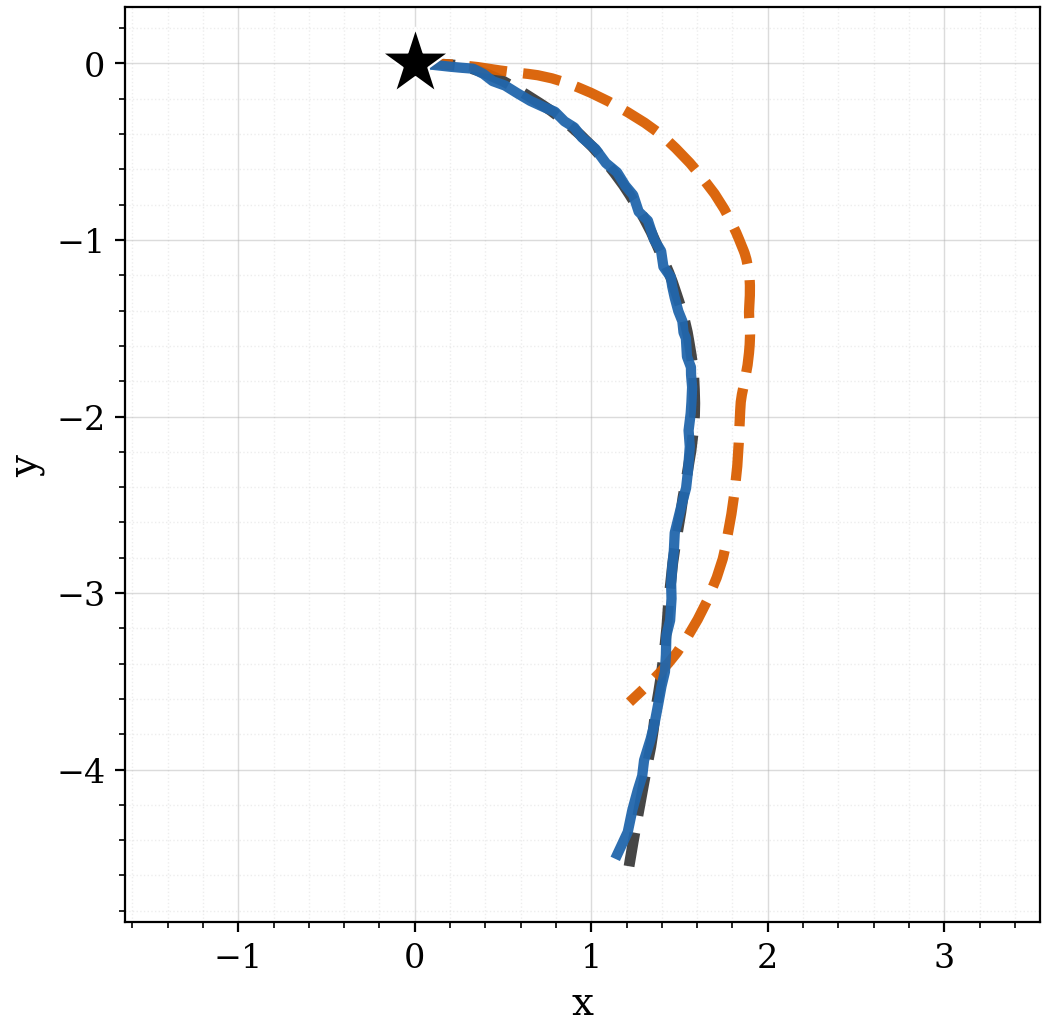} &
\includegraphics[width=0.15\textwidth,valign=c]{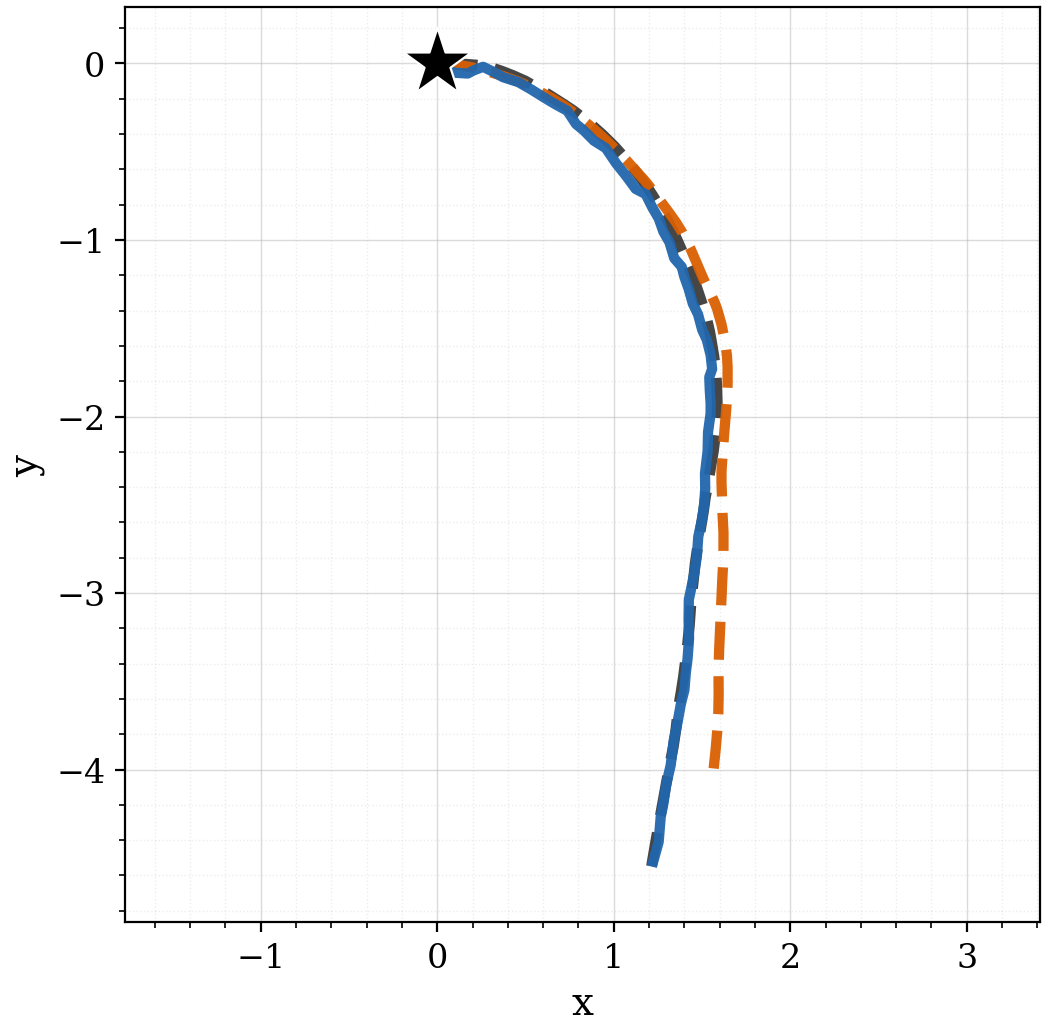} &
\includegraphics[width=0.15\textwidth,valign=c]{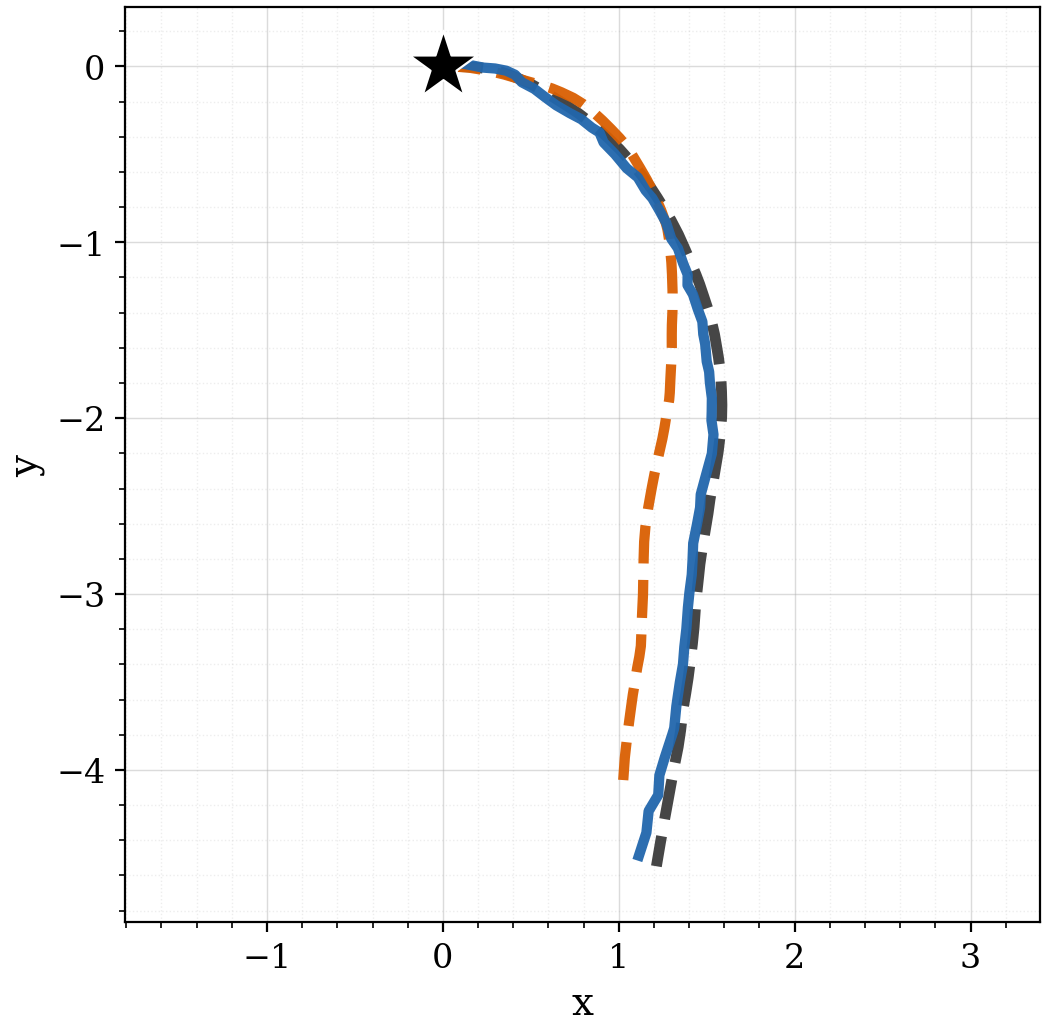} &
\includegraphics[width=0.15\textwidth,valign=c]{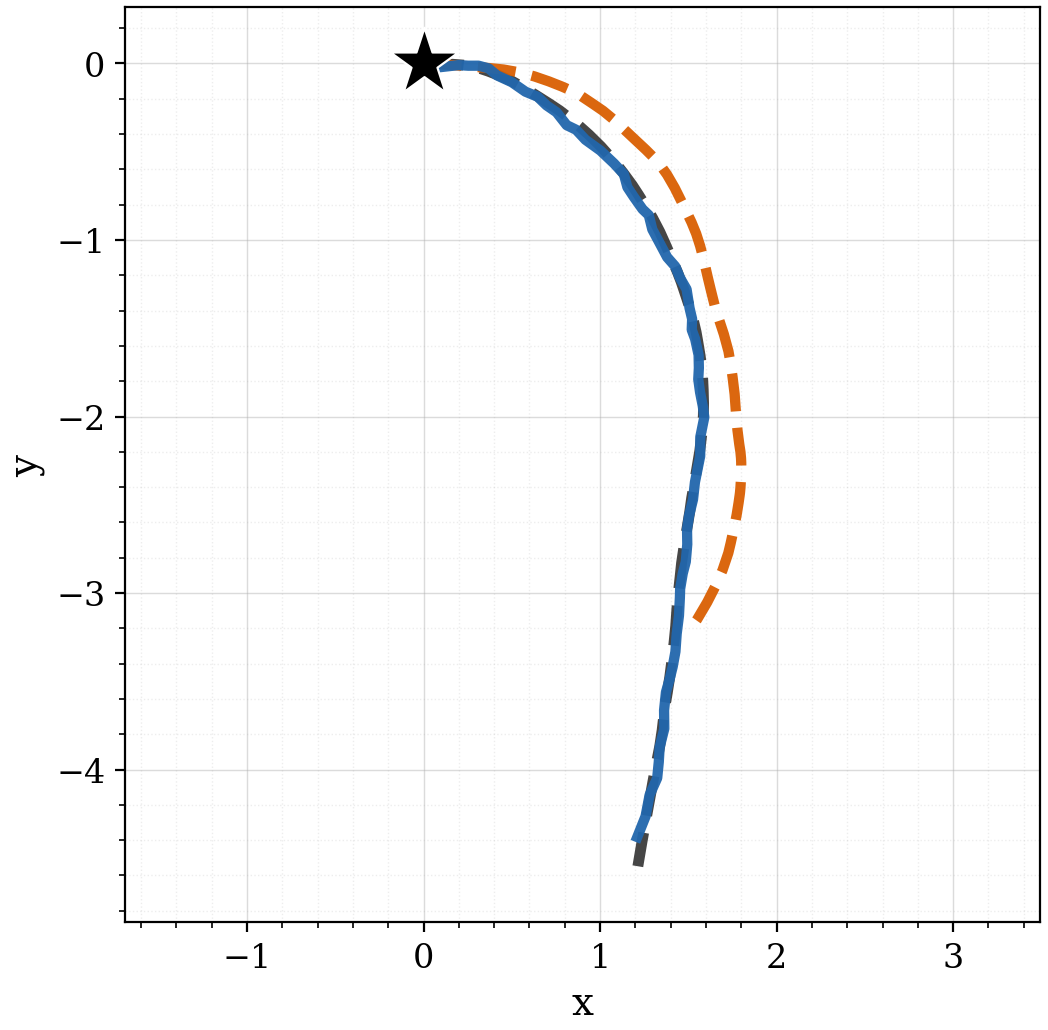} \\[2mm]
\textbf{DRGD} &
\includegraphics[width=0.15\textwidth,valign=c]{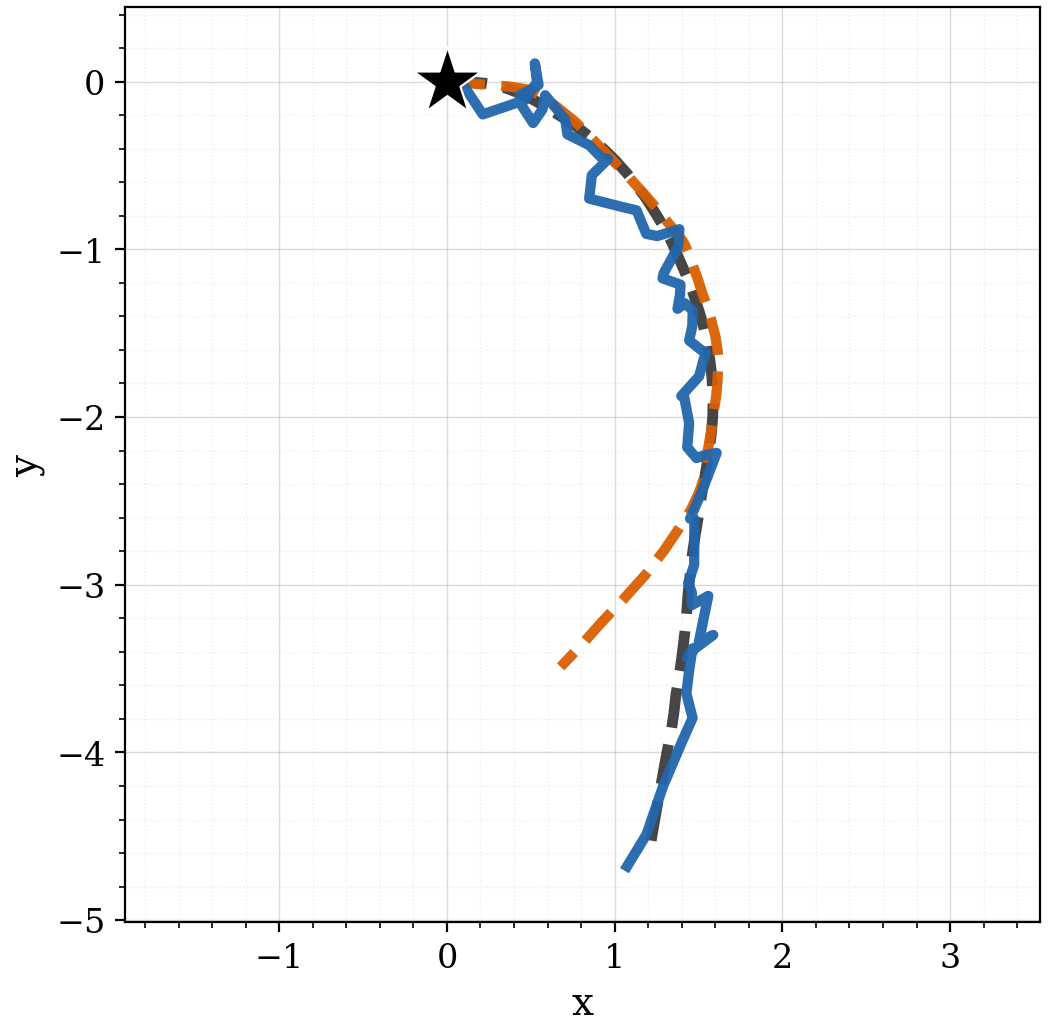} &
\includegraphics[width=0.15\textwidth,valign=c]{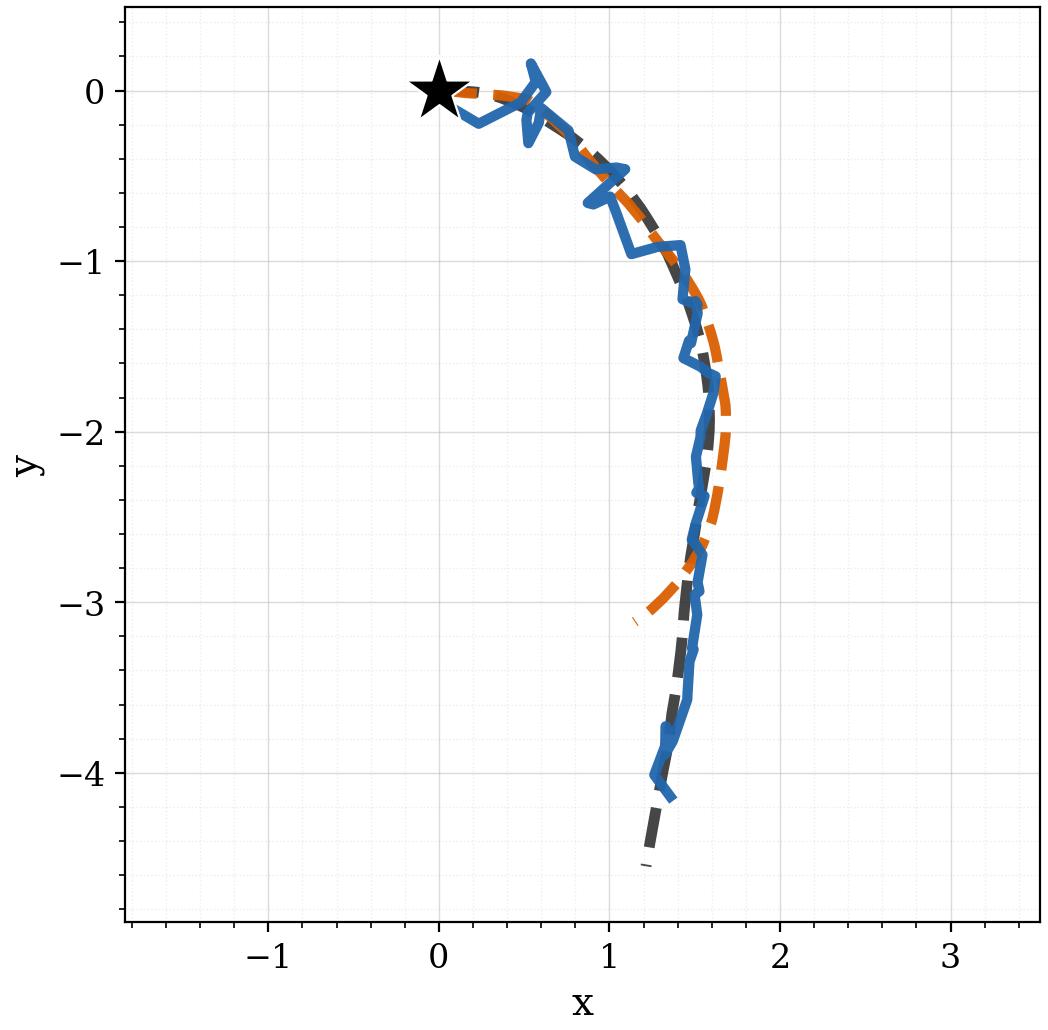} &
\includegraphics[width=0.15\textwidth,valign=c]{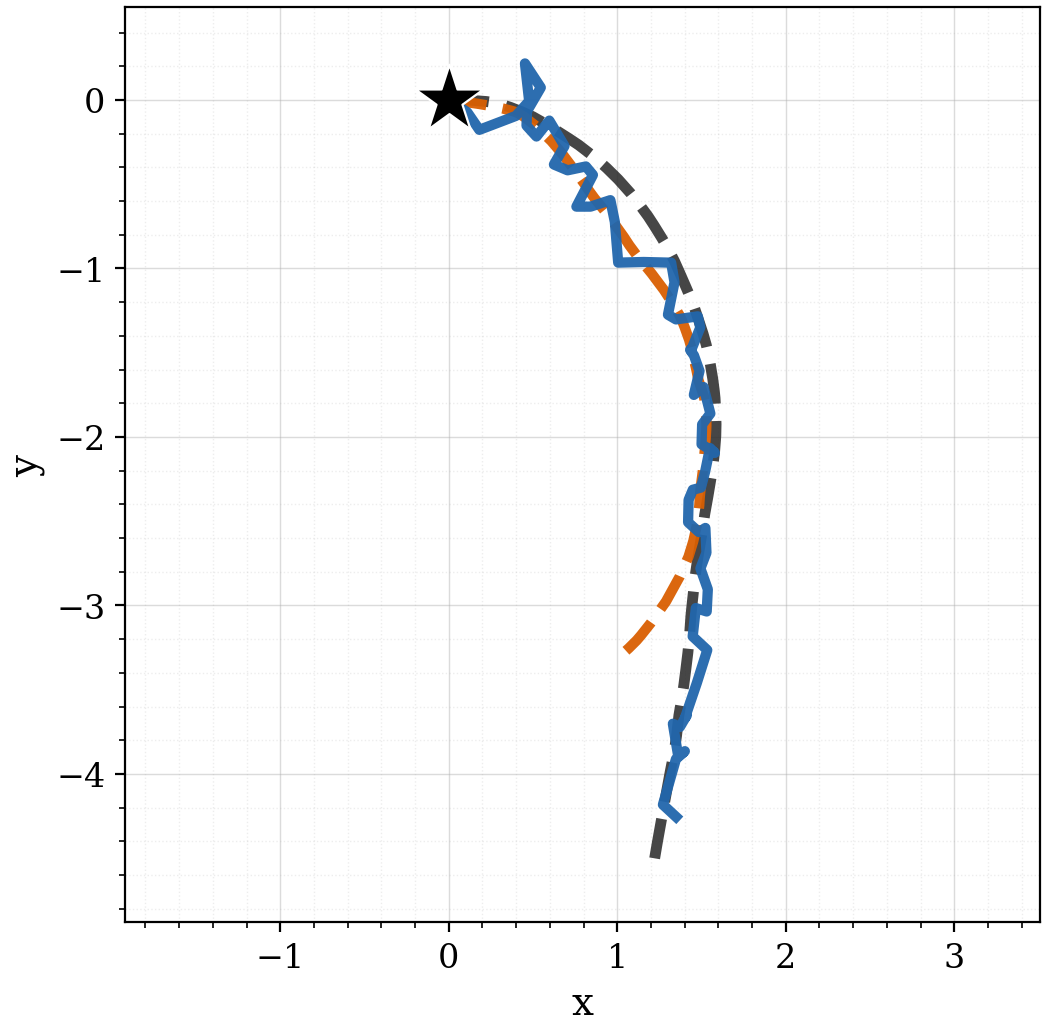} &
\includegraphics[width=0.15\textwidth,valign=c]{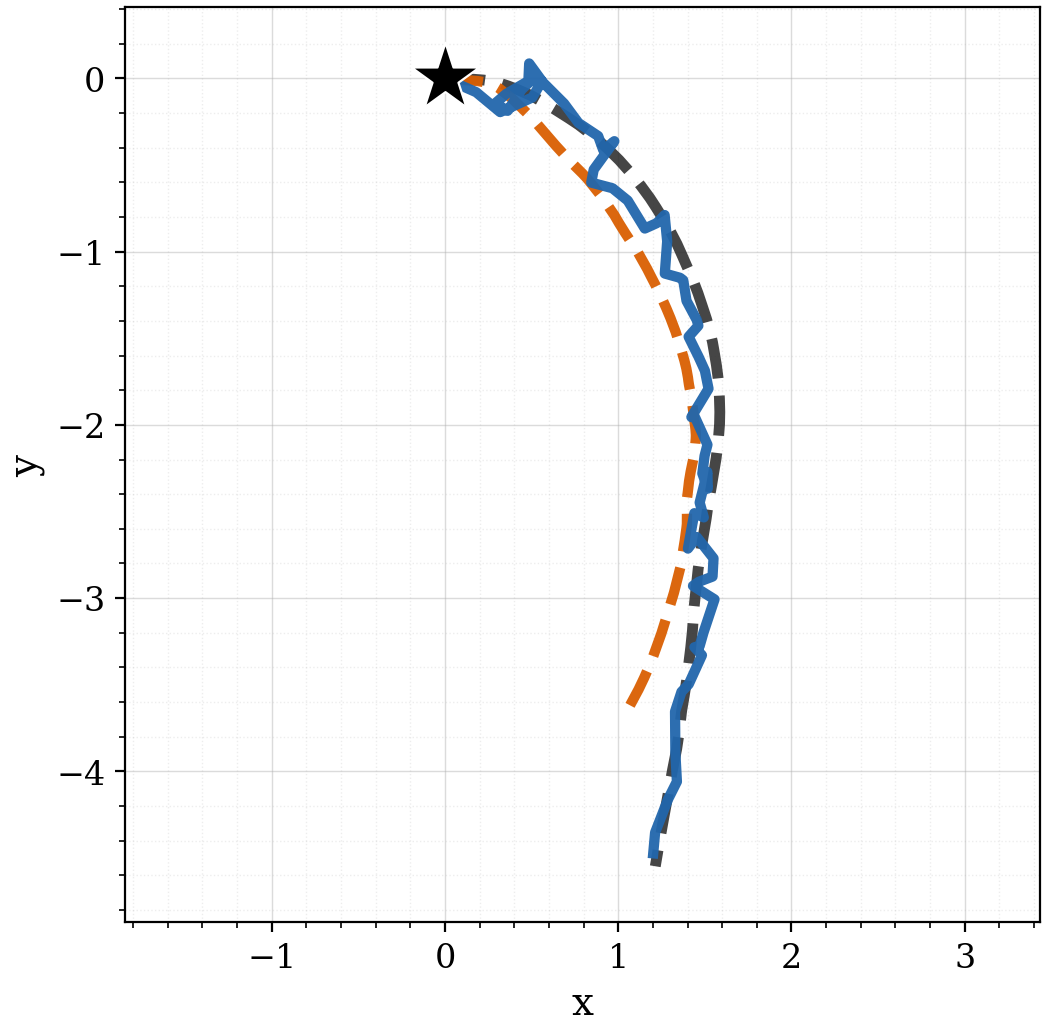} \\[2mm]
\textbf{DCG} &
\includegraphics[width=0.15\textwidth,valign=c]{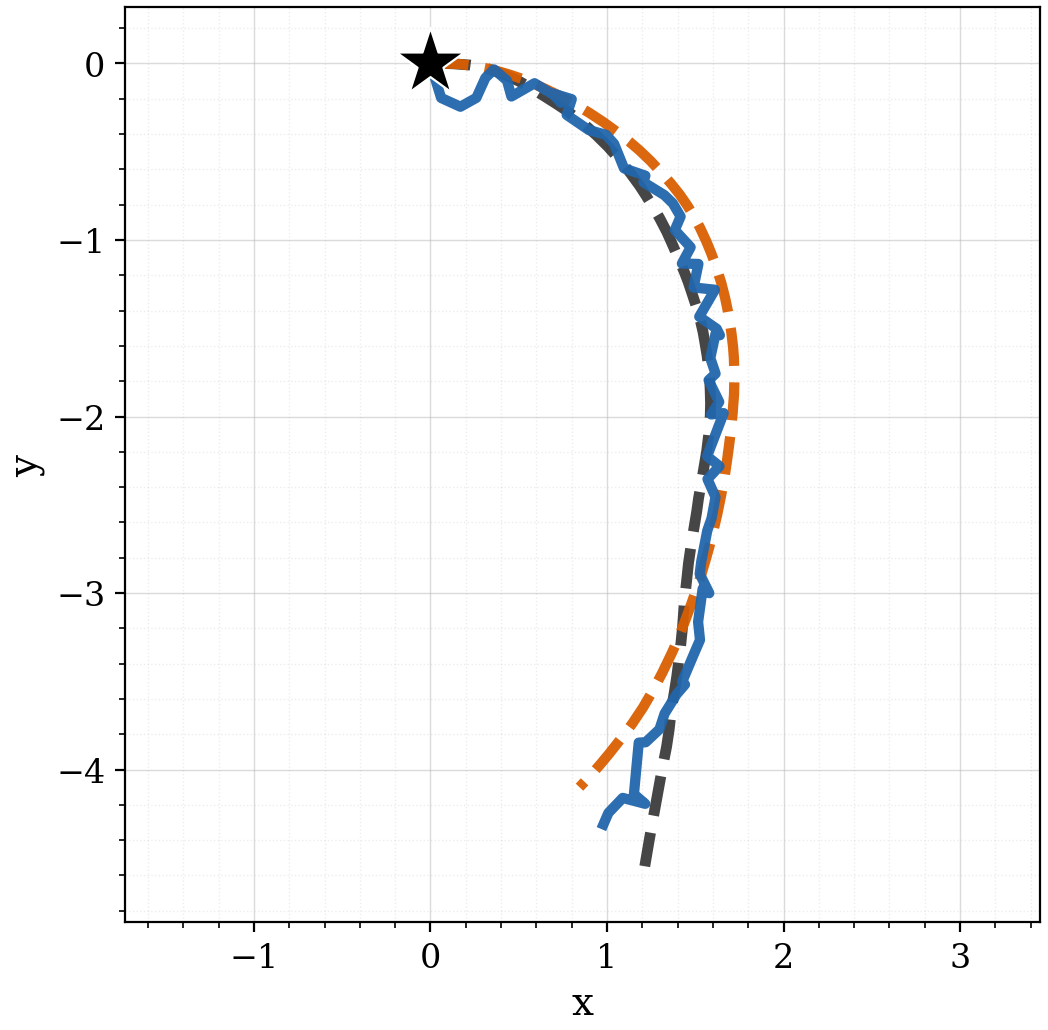} &
\includegraphics[width=0.15\textwidth,valign=c]{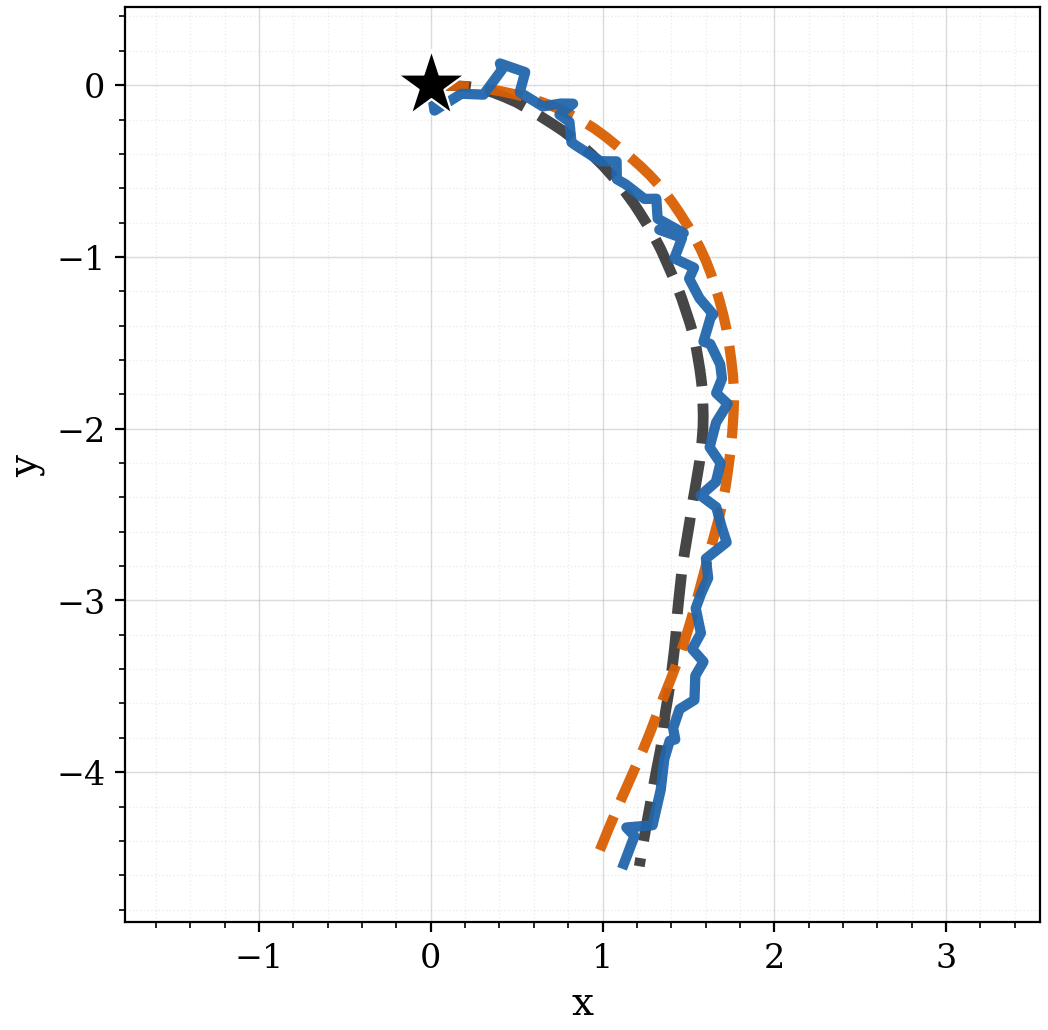} &
\includegraphics[width=0.15\textwidth,valign=c]{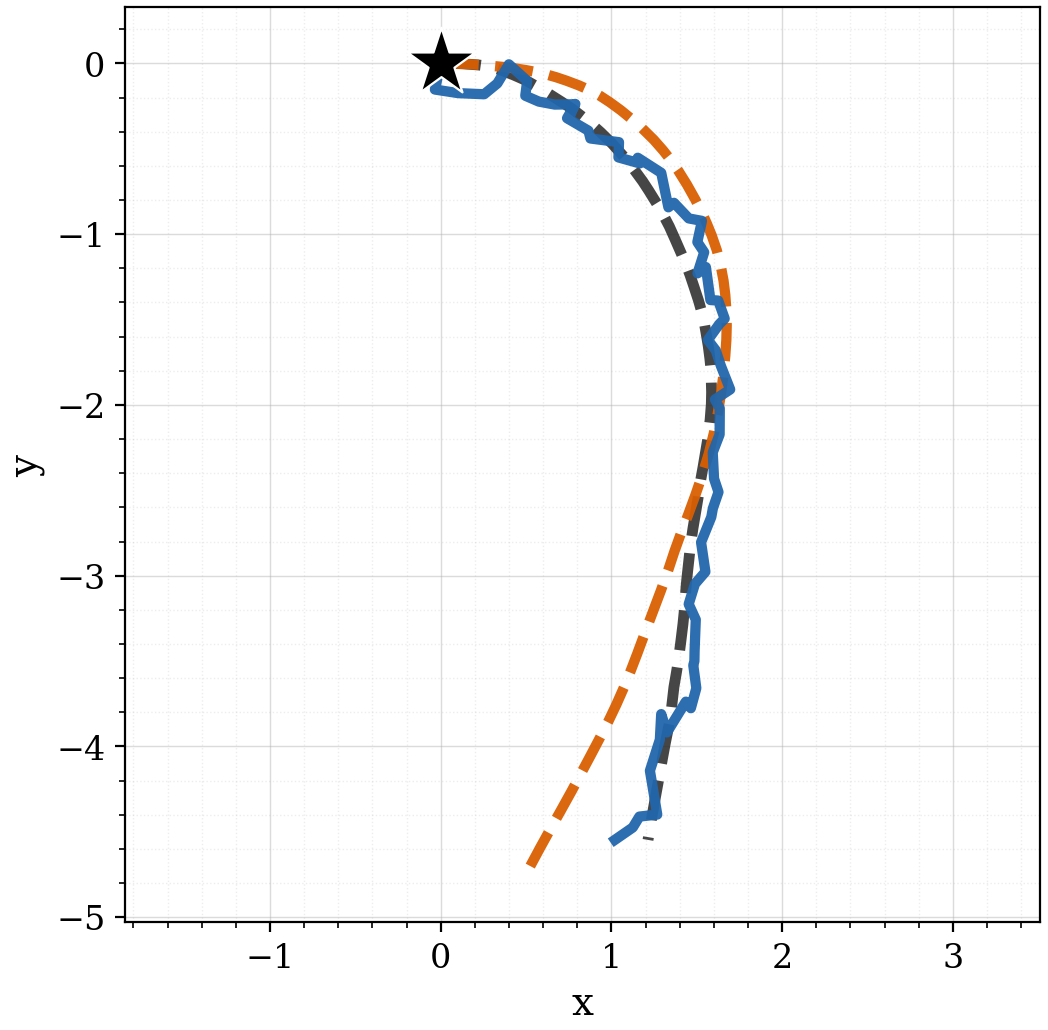} &
\includegraphics[width=0.15\textwidth,valign=c]{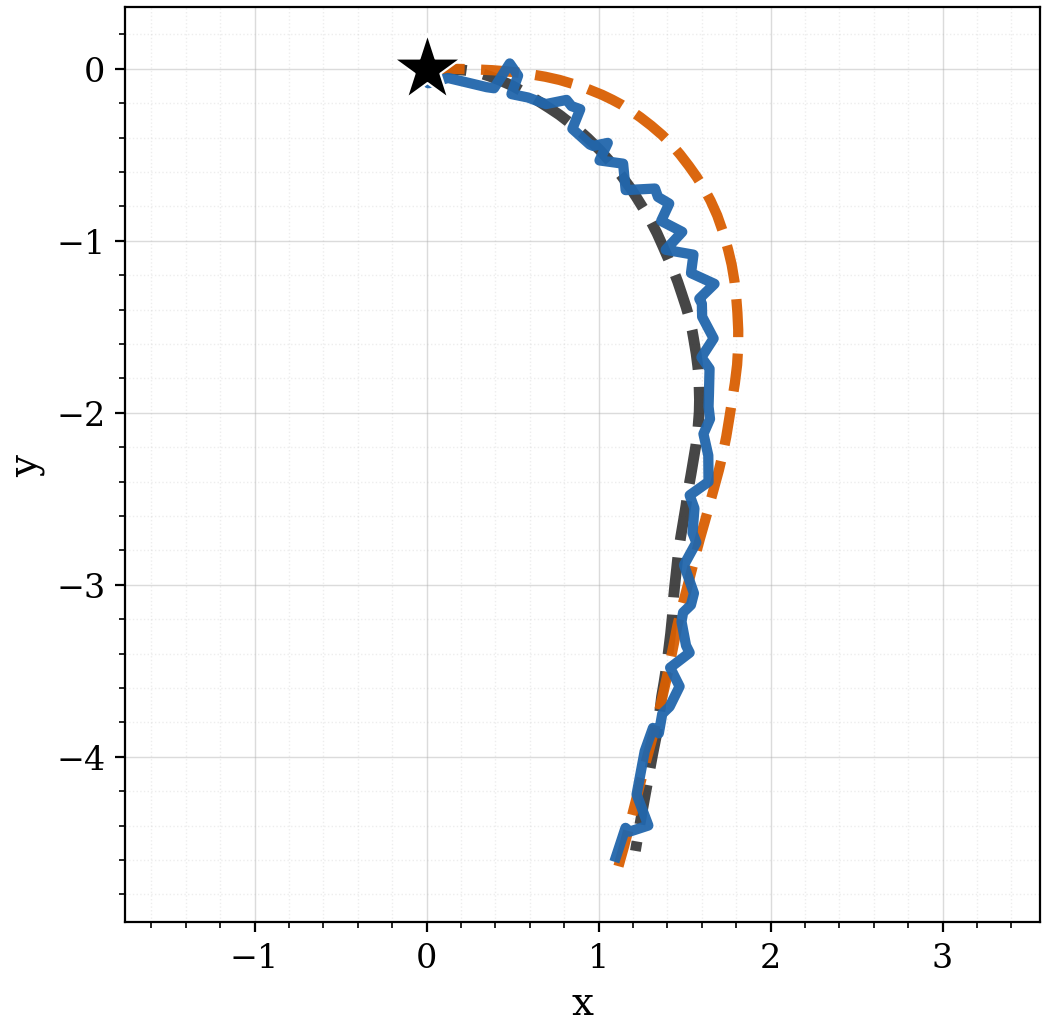}
\end{tabular}

\vspace{2mm}

% ---------- Trajectory 2 ----------
\textbf{Trajectory 2}

\vspace{2mm}

\begin{tabular}{@{}c@{\hspace{3mm}}cccc@{}}
\textbf{PDG} &
\includegraphics[width=0.15\textwidth,valign=c]{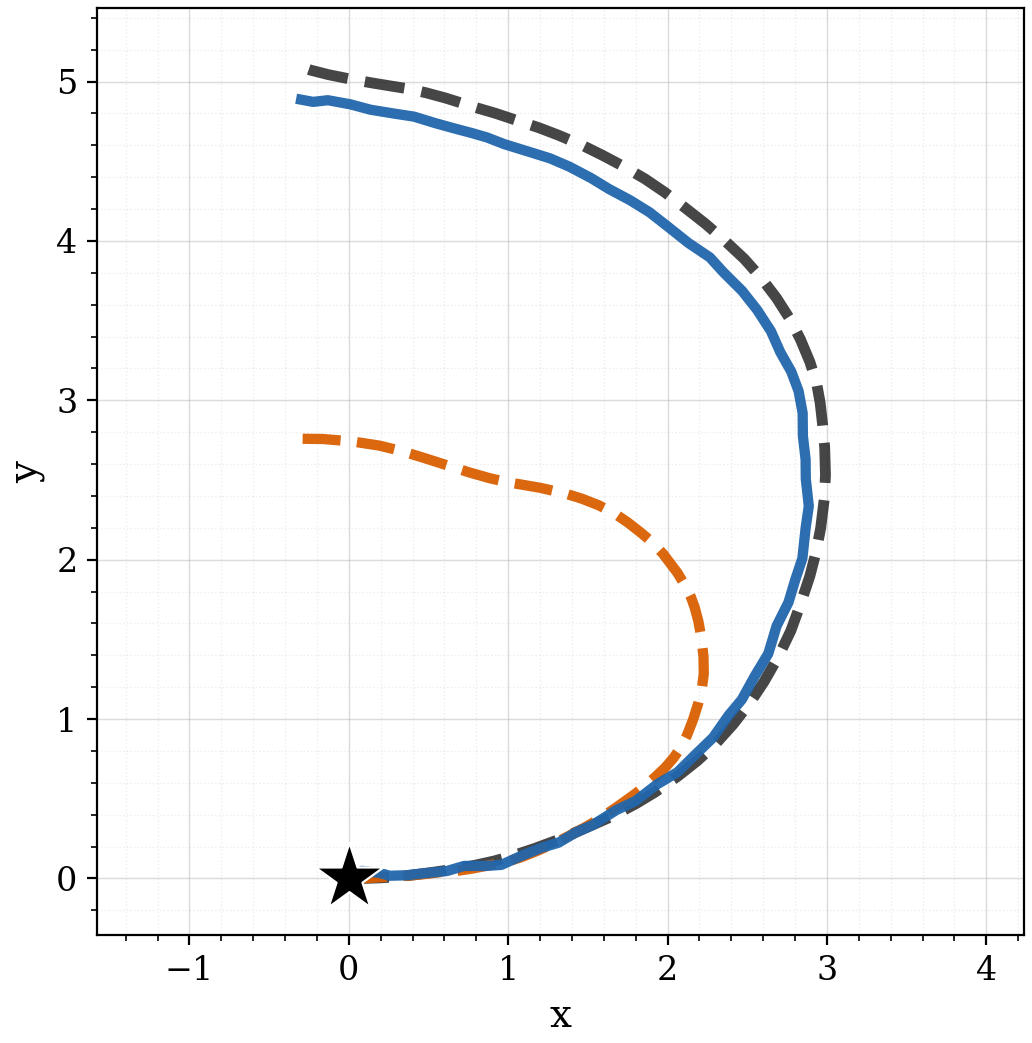} &
\includegraphics[width=0.15\textwidth,valign=c]{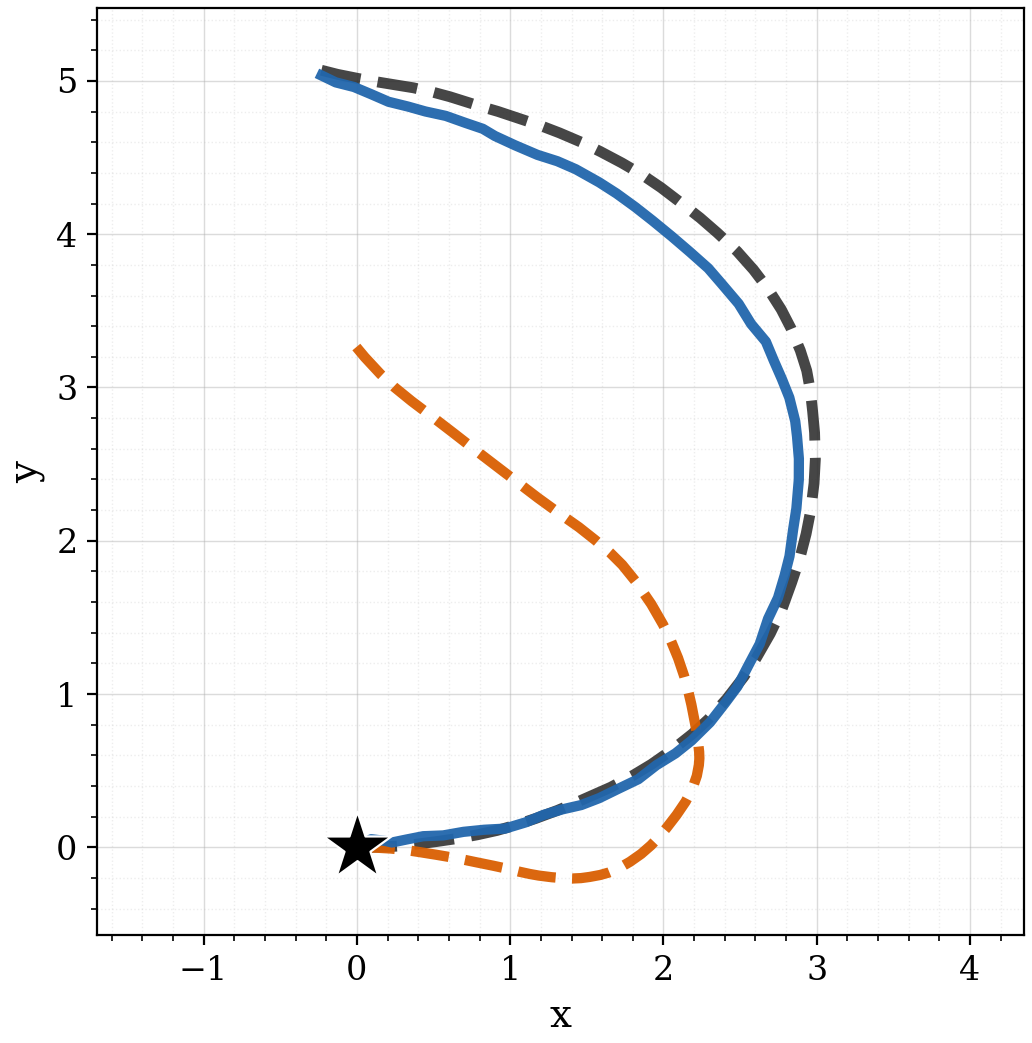} &
\includegraphics[width=0.15\textwidth,valign=c]{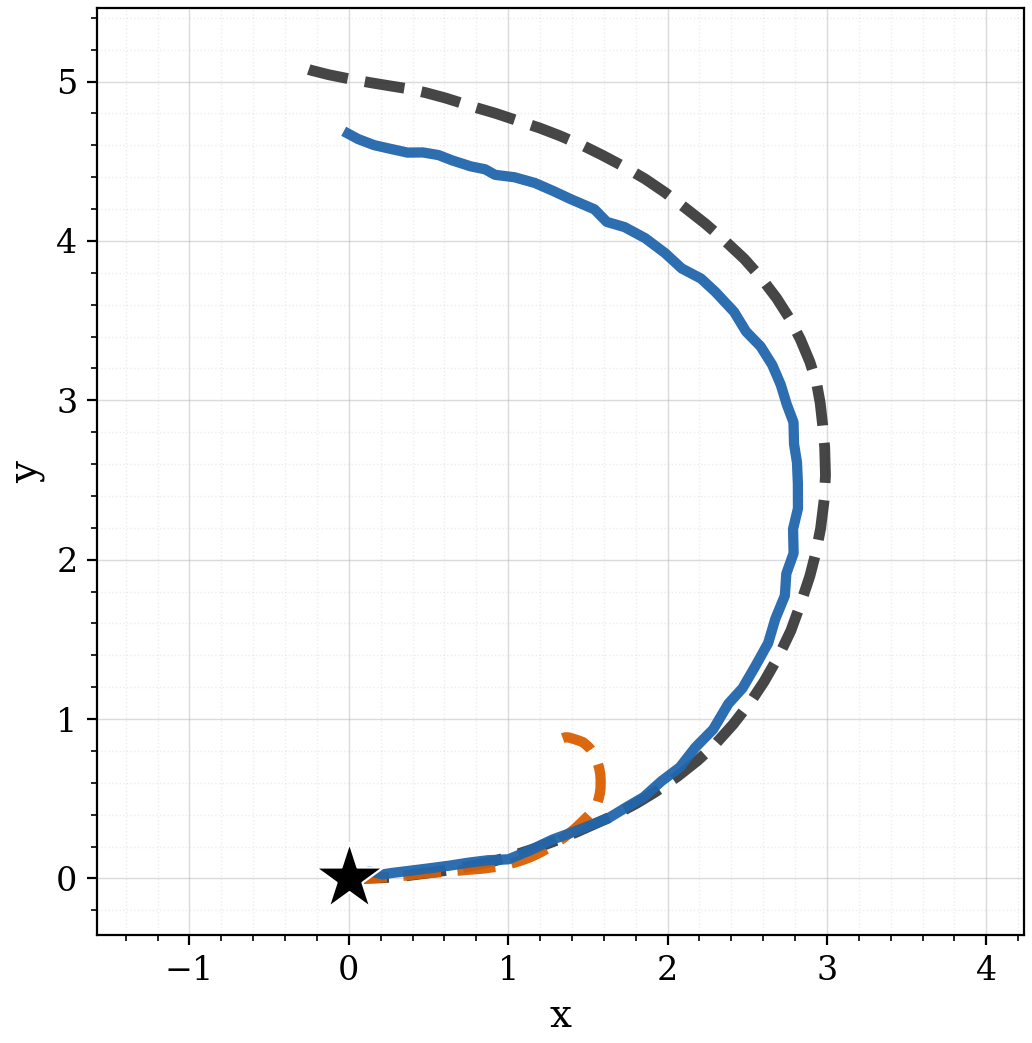} &
\includegraphics[width=0.15\textwidth,valign=c]{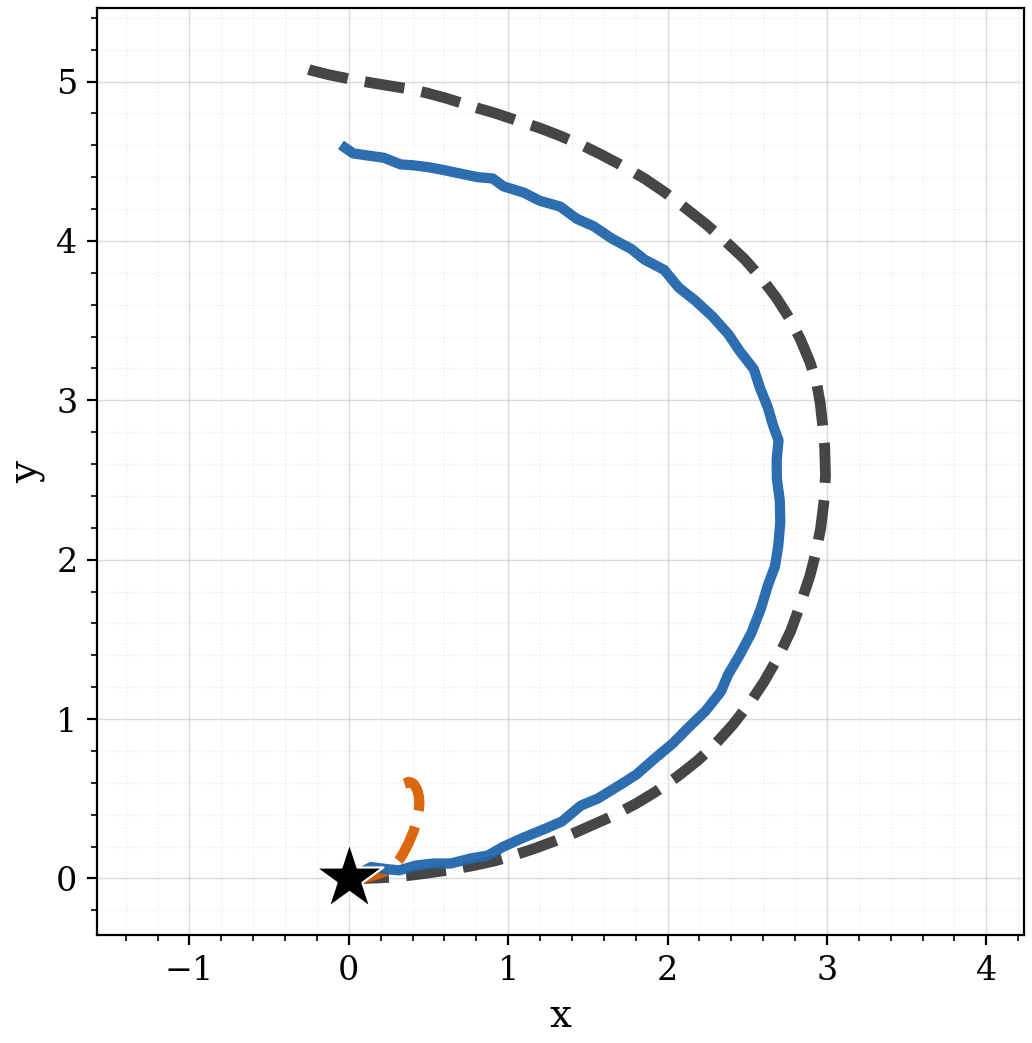} \\[2mm]
\textbf{DRGD} &
\includegraphics[width=0.15\textwidth,valign=c]{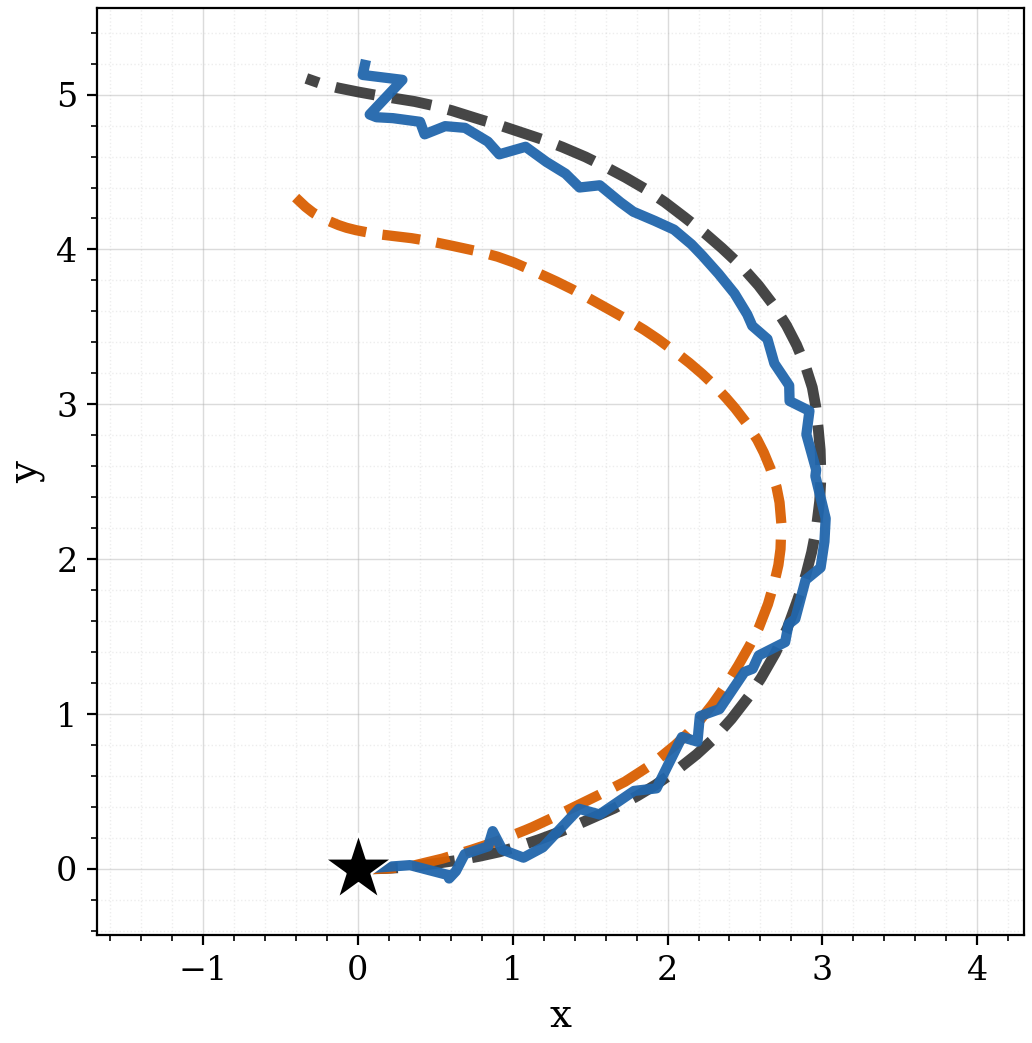} &
\includegraphics[width=0.15\textwidth,valign=c]{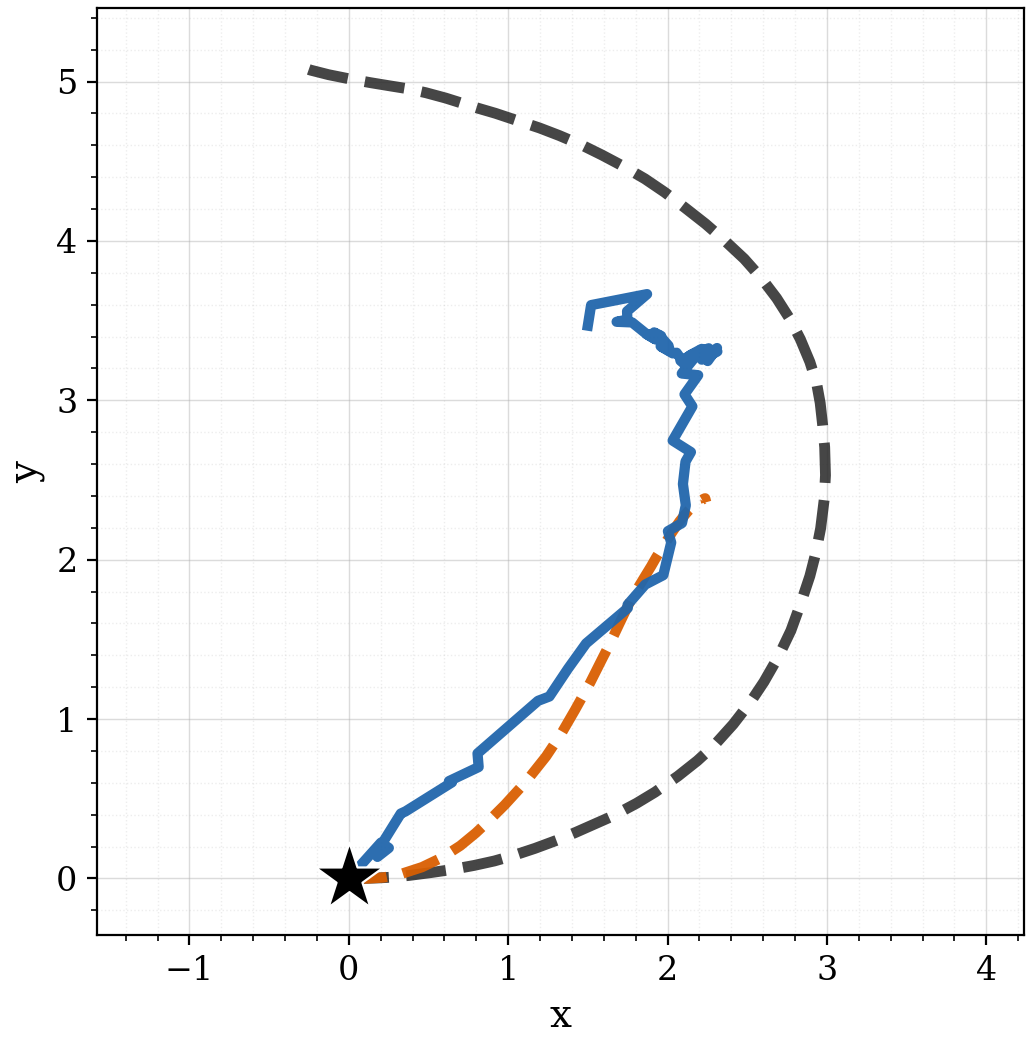} &
\includegraphics[width=0.15\textwidth,valign=c]{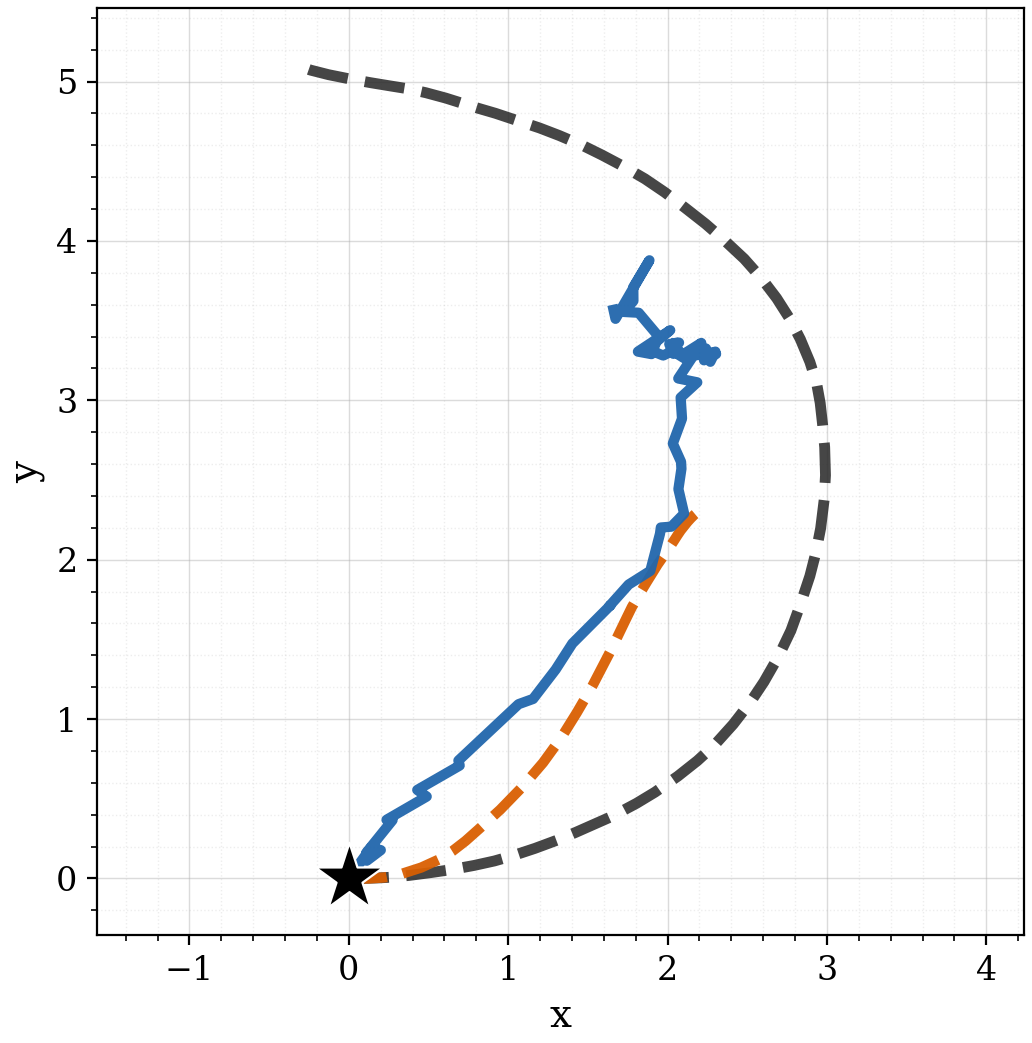} &
\includegraphics[width=0.15\textwidth,valign=c]{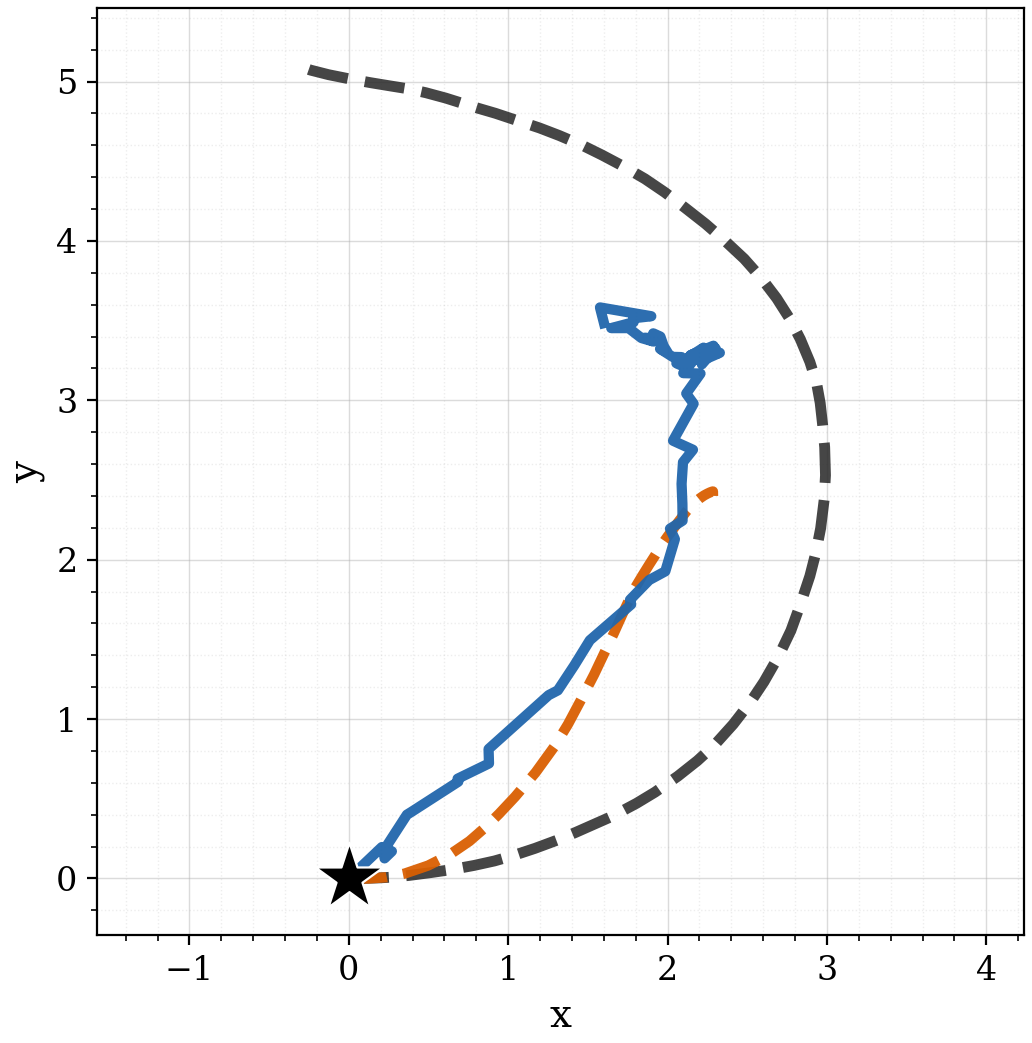} \\[2mm]
\textbf{DCG} &
\includegraphics[width=0.15\textwidth,valign=c]{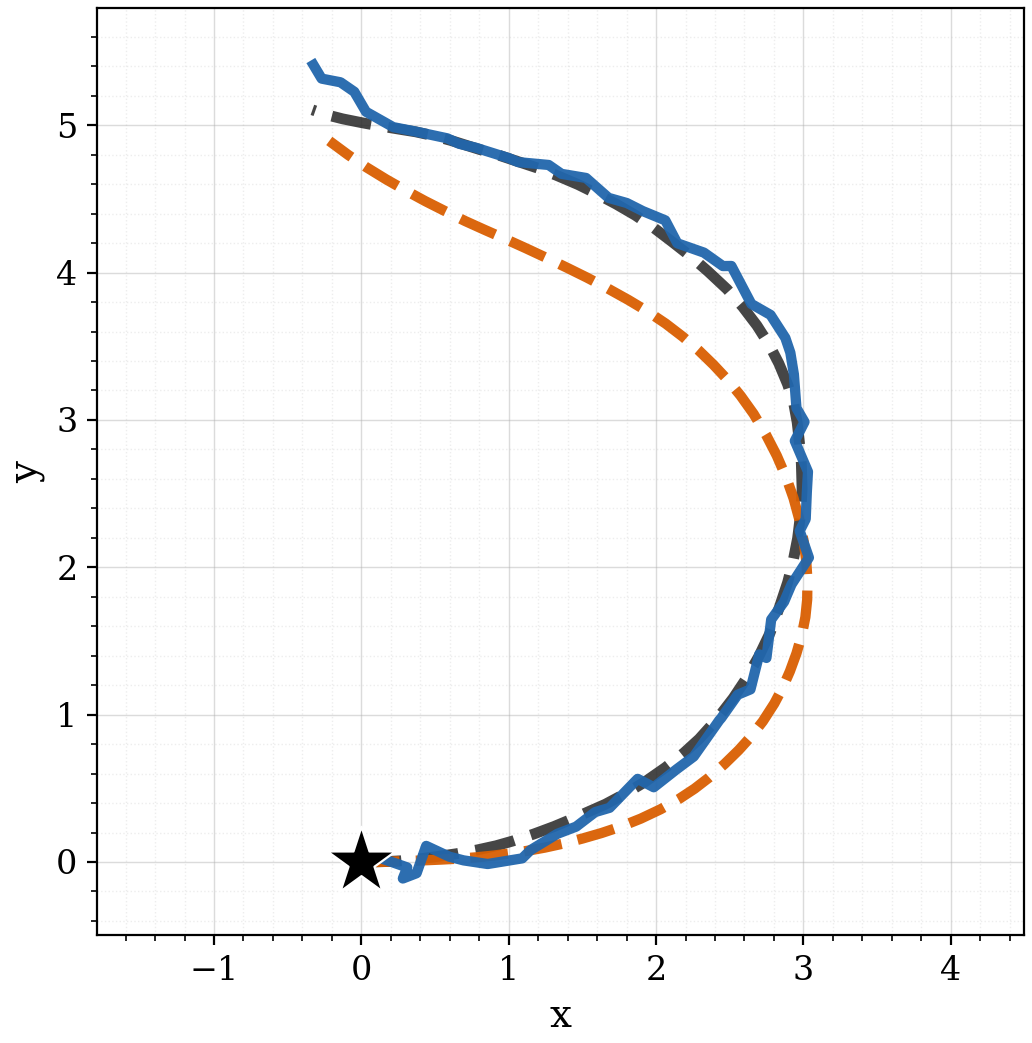} &
\includegraphics[width=0.15\textwidth,valign=c]{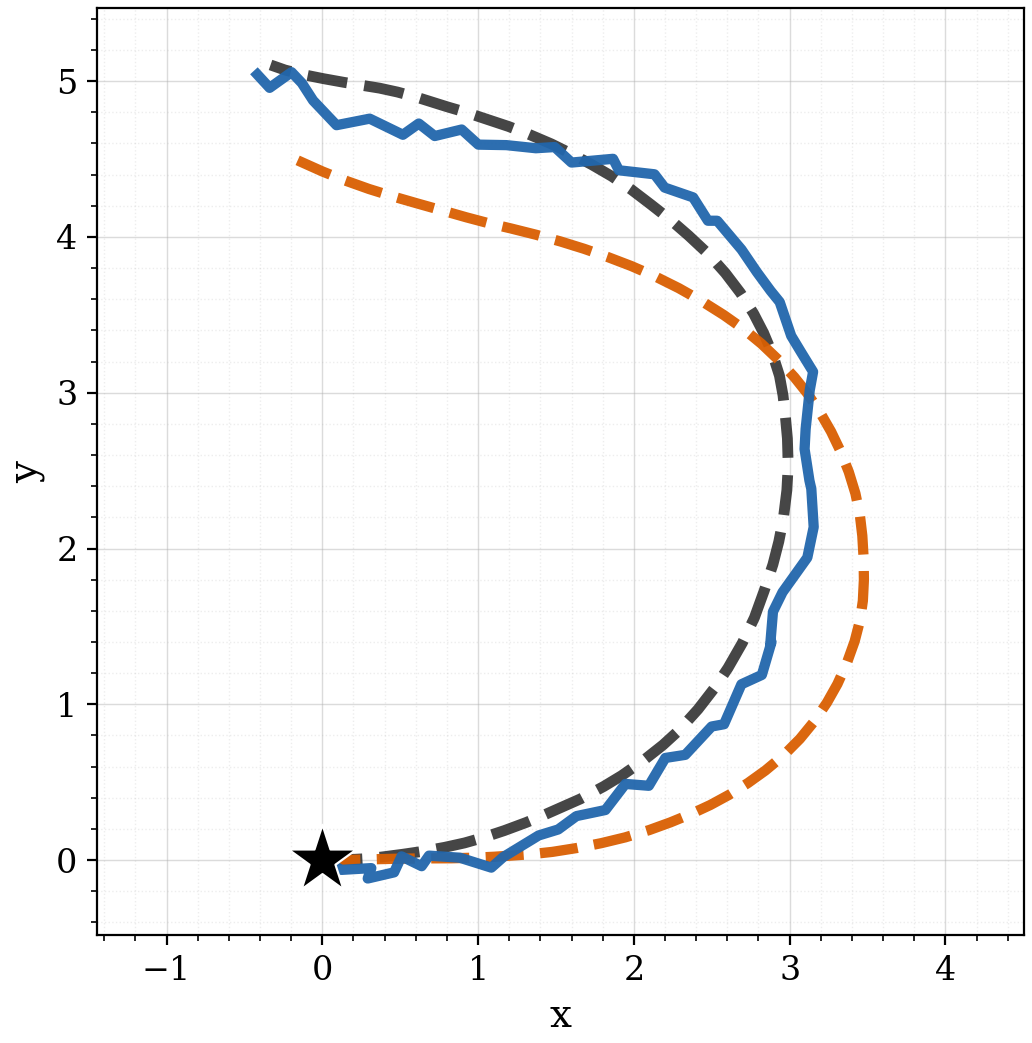} &
\includegraphics[width=0.15\textwidth,valign=c]{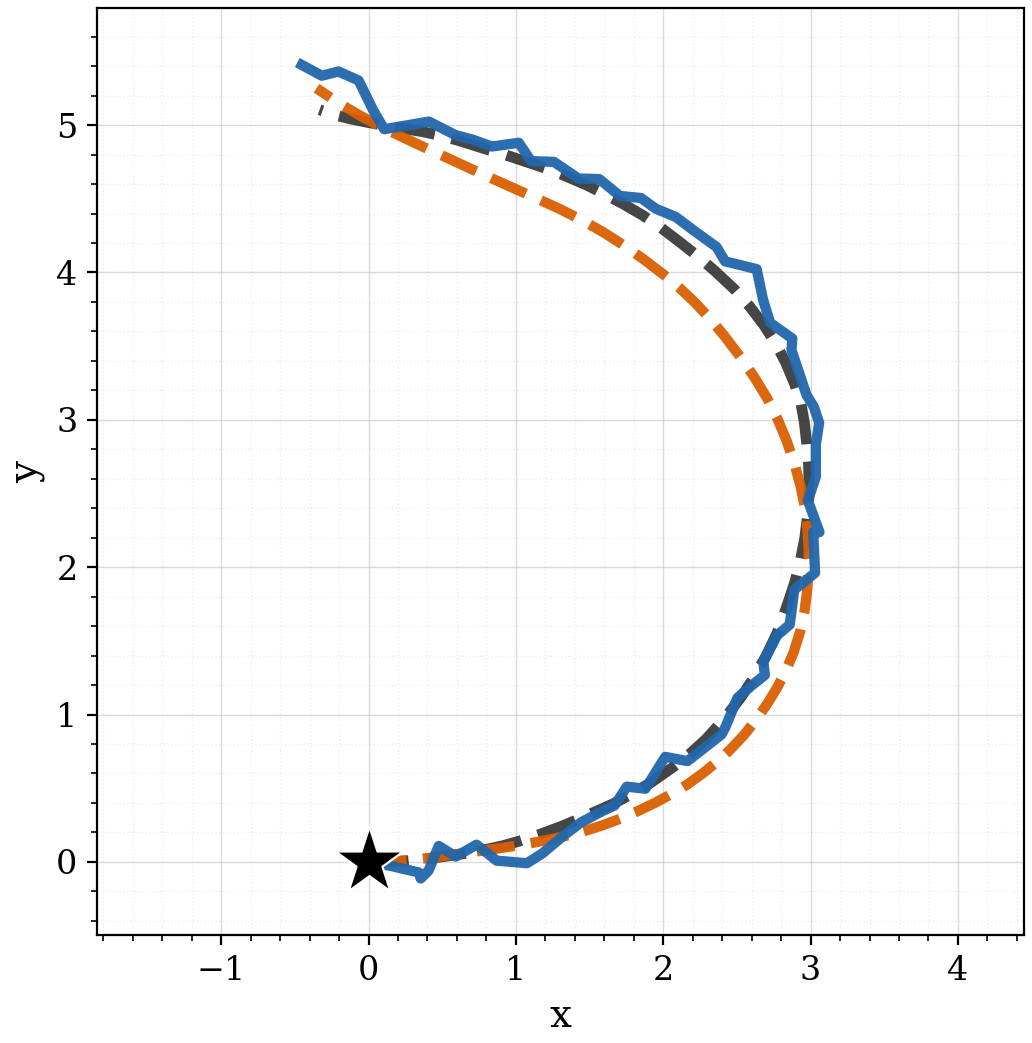} &
\includegraphics[width=0.15\textwidth,valign=c]{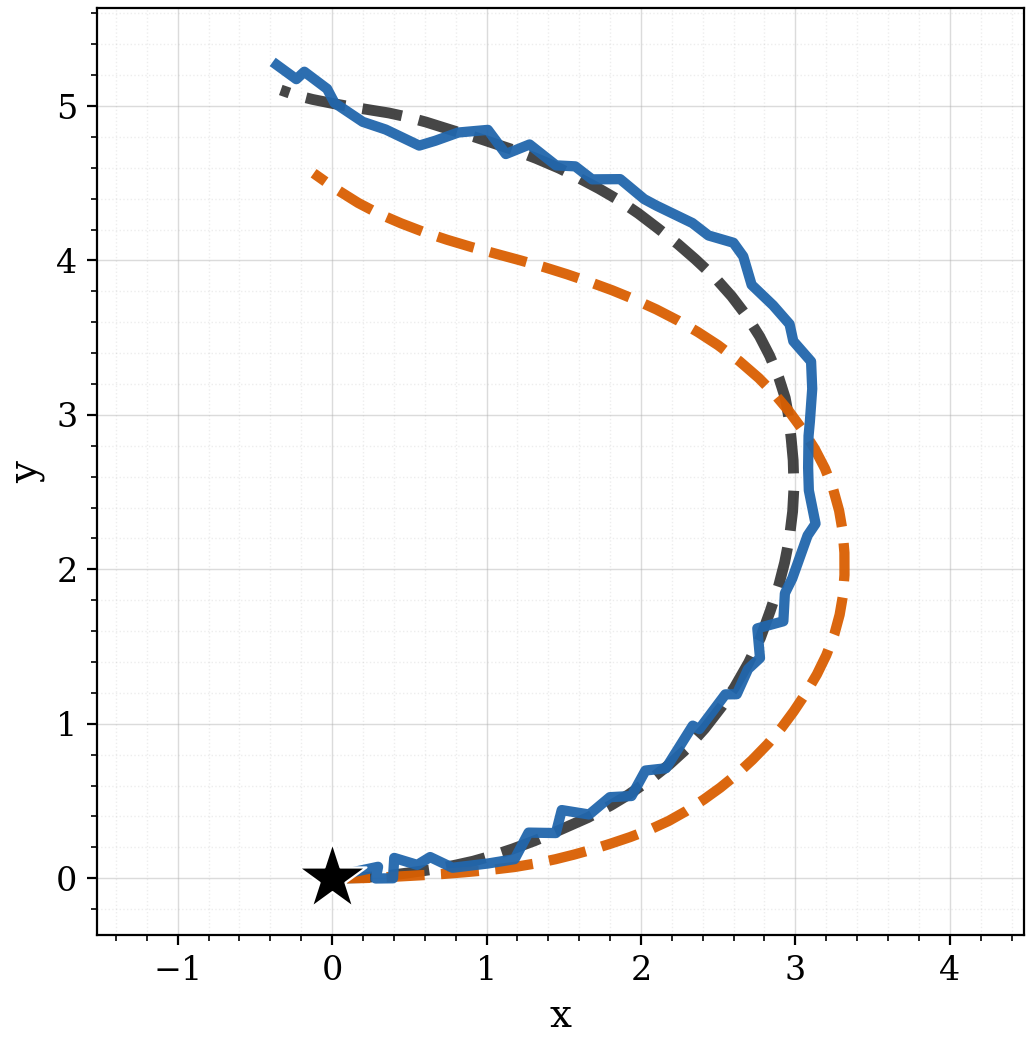}
\end{tabular}

\caption{Representative reference-plan samples generated by PDG, DRGD, and DCG for the circle and two random-trajectory settings. Each row shows the four highest-ranked candidates generated by the corresponding guidance method.}
\label{fig:unicycle-reference-plans}
\end{figure}

\subsection{Diffuser Benchmarks}
To verify that our algorithm generalizes to larger-scale and more complex dynamical systems, we further evaluate DCG on standard MuJoCo offline-control benchmarks using a Diffuser-style planning pipeline. Similar to the unicycle setting, the diffusion model is trained on offline trajectories and generates candidate state–action sequences through reverse denoising, while gradient guidance biases the generated plans toward higher predicted return. The resulting action sequence is then executed in the environment, and performance is evaluated using normalized return. We compare our proposed DCG method with the classical PDG baseline across Hopper-medium-v2, Walker-medium-v2, and HalfCheetah-medium-v2. All results are reported over $150$ evaluation
episodes. PDG and DCG are evaluated using the
same initial states and environment seeds. For each task, we tune the PDG and DCG guidance scale separately. 
As shown in Table~\ref{tab:mujoco-comparison}, DCG achieves slightly higher normalized returns on all three tasks. We would also like to note that these experiments are intended primarily to validate that DCG can be integrated seamlessly into an existing Diffuser-based control pipeline and transferred to related benchmark tasks. Although the improvements are consistent, they remain modest in magnitude. Further improving performance and evaluating the method across a broader range of settings, datasets, and D4RL algorithms remain important directions for future work.

\begin{table}[t]
\centering
\caption{Performance comparison of PDG and DCG on MuJoCo medium-dataset tasks. Results are reported as mean $\pm$ standard deviation.}
\label{tab:mujoco-comparison}
\begin{tabular}{lcc}
\toprule
Task & PDG & DCG \\
\midrule
Hopper-medium-v2
& $93.03 \pm 0.71$
& $\mathbf{94.82 \pm 0.29}$ \\
Walker-medium-v2
& $76.70 \pm 1.05$
& $\mathbf{78.14 \pm 0.86}$ \\
HalfCheetah-medium-v2
& $44.47 \pm 0.10$
& $\mathbf{44.53 \pm 0.08}$ \\
\bottomrule
\end{tabular}
\end{table}

% \begin{table}[t]
% \centering
% \caption{Dynamic-feasibility comparison on MuJoCo tasks. Results report the mean $\pm$ standard deviation of the trajectory-wide $L_2$ dynamics error.}
% \label{tab:mujoco-dynamic-feasibility}
% \begin{tabular}{lccc}
% \toprule
% Task & Monte Carlo & PDG & DCG \\
% \midrule
% Hopper
% & $0.394 \pm 0.053$
% & $0.413 \pm 0.053$
% & $\mathbf{0.393 \pm 0.048}$ \\
% HalfCheetah
% & $0.338 \pm 0.083$
% & $0.338 \pm 0.084$
% & $\mathbf{0.338 \pm 0.083}$ \\
% Walker2d
% & $3.188 \pm 0.178$
% & $3.119 \pm 0.222$
% & $\mathbf{2.811 \pm 0.129}$ \\
% \bottomrule
% \end{tabular}
% \end{table}
\section{Convergence Analysis for the Linear Gaussian Setting}\label{app:linear}
\begin{proof}[Proof of Lemma \ref{lemma:pi-t}]
From \eqref{eq:xtx0} and Assumption~\ref{assump:linear}, the noisy variable satisfies
\begin{align*}
x_t\sim\cN(0,A^\top\Sigma A+\sigma_t^2I).
\end{align*}
Therefore,
\begin{align*}
\nabla\log P_t(x)
=
-(A^\top\Sigma A+\sigma_t^2I)^{-1}x.
\end{align*}
Using the definition of the Stein denoising operator gives
\begin{align*}
\widehat\pi_t(x)
&=
x+\sigma_t^2\nabla\log P_t(x) 
=
\left(I-\sigma_t^2(A^\top\Sigma A+\sigma_t^2I)^{-1}\right)x =
(A^\top\Sigma A+\sigma_t^2I)^{-1}A^\top\Sigma A x.
\end{align*}
The variational characterization follows from the first-order optimality condition of the stated quadratic minimization problem over $\range(A^\top)$.
\end{proof}

\subsection{Guided Diffusion for under the linear assumption}
We first characterize the specific form of guided diffusion under the linear data manifold assumption (Assumption \ref{assump:linear}) as follows:
\begin{lemma}\label{lemma:guided-diffusion-u-v}
    Define
    \begin{align}\label{eq:def-u-v}
        u_t:= Ax_t, \qquad v_t := \Aperp x_t
    \end{align}
    then \eqref{eq:guided-diffusion} with Assumption \ref{assump:linear} would give the following update
    \begin{align*}
        v_{t-1} &= \frac{\sigma_{t-1}}{\sigma_t} v_t\\
        u_{t-1} &= u_t - (1-\frac{\sigma_{t-1}}{\sigma_t})(I + \sigma_t^2\Sigma^{-1})^{-1}(\sigma_t^2 \Sigma^{-1} u_t + \eta_t A \nabla f(x_t)) 
    \end{align*}
    \begin{proof}
        From Lemma \ref{lemma:pi-t}, $\widehat\pi_t(x) = A^\top (I + \sigma_t^2 \Sigma^{-1})^{-1}Ax$, hence
        \begin{align*}
            v_{t-1} &= \Aperp x_{t-1} = \Aperp\left(\frac{\sigma_{t-1}}{\sigma_t} x_t + \left(1- \frac{\sigma_{t-1}}{\sigma_t}\right)\widehat\pi_t(x_t -\eta_t \nabla f(x_t))\right)\\
            &= \Aperp\left(\frac{\sigma_{t-1}}{\sigma_t} x_t + \left(1- \frac{\sigma_{t-1}}{\sigma_t}\right)A^\top (I + \sigma_t^2 \Sigma^{-1})^{-1}A(x_t -\eta_t \nabla f(x_t))\right)\\
            &= \frac{\sigma_{t-1}}{\sigma_t}\Aperp x_t =\frac{\sigma_{t-1}}{\sigma_t} v_t\\
            u_{t-1} &= A x_{t-1} = A\left(\frac{\sigma_{t-1}}{\sigma_t} x_t + \left(1- \frac{\sigma_{t-1}}{\sigma_t}\right)\widehat\pi_t(x_t -\eta_t \nabla f(x_t))\right)\\
            &= A\left(\frac{\sigma_{t-1}}{\sigma_t} x_t + \left(1- \frac{\sigma_{t-1}}{\sigma_t}\right)A^\top (I + \sigma_t^2 \Sigma^{-1})^{-1}A(x_t -\eta_t \nabla f(x_t))\right)\\
            &= \frac{\sigma_{t-1}}{\sigma_t} u_t + \left(1- \frac{\sigma_{t-1}}{\sigma_t}\right)(I + \sigma_t^2 \Sigma^{-1})^{-1} (u_t - \eta_tA\nabla f(x_t))\\
            & = u_t + \left(1- \frac{\sigma_{t-1}}{\sigma_t}\right)\left(-u_t + (I + \sigma_t^2 \Sigma^{-1})^{-1} (u_t - \eta_tA\nabla f(x_t))\right)\\
            & = u_t + \left(1- \frac{\sigma_{t-1}}{\sigma_t}\right)(I + \sigma_t^2\Sigma^{-1})^{-1} (-\sigma_t^2\Sigma^{-1}u_t -\eta_t A\nabla f(x_t) )
        \end{align*}
    \end{proof}
\end{lemma}

According to Lemma \ref{lemma:guided-diffusion-u-v}, if we consider
\begin{align*}
    \frac{\sigma_{t-1}}{\sigma_t} = 1-\epsilon, \qquad 
    M_t := (I + \sigma_t^2\Sigma^{-1})^{-1}, \qquad \eta_t = \eta
\end{align*}
And consider $\Phi_t(x)$ defined as in \eqref{eq:def-Phi_t}, then from Lemma \ref{lemma:guided-diffusion-u-v} we get
\begin{align} \label{eq:subspace-dynamics}
    u_{t-1} = u_t - \underbrace{\eta \epsilon}_{:=\teta} M_t A\nabla\Phi_t(x_t)
\end{align}
Hence, \eqref{eq:subspace-dynamics} can be interpreted as a gradient-type update on the time-varying objective $\Phi_t$, where $\teta=\eta\epsilon = \eta (1-\frac{\sigma_{t-1}}{\sigma_t})$ acts as the effective stepsize. The matrix $M_t A$ modulates the descent direction by a time-varying linear transformation, which anisotropically rescales different components of the gradient according to the current diffusion level and subspace geometry. Overall, the dynamics resemble a gradient descent scheme with both time-varying objective and time-varying anisotropic rescale that are implicitly controlled by the diffusion schedule.
\subsection{Proof of Theorem \ref{theorem:convergence-guided-diffusion}}

\begin{proof}[Proof of Theorem \ref{theorem: convergence-guided-diffusion}]
We decompose the function $V_t(x_t)$ as follows
\begin{align*}
    V_t(x_t):=  \underbrace{(\Phi_{t}(A^\top Ax_t) - \Phi_{t}(x_{t}^\star))}_{:= \Delta_t} +  \underbrace{\frac{\eta L^2}{2(1-\eta\mu c)} \|A^\perp x_t\|^2}_{e_t} + \underbrace{\frac{\eta^{-1}D^2}{2(1-\eta\mu c)}\sigma_t^2}_{w_t}.
\end{align*}
From Lemma \ref{lemma:Delta_t}, 
\begin{align*}
        \Delta_{t-1} \le (1-\teta \mu c)\Delta_t + \frac{(\sigma_t^2 - \sigma_{t-1}^2)\eta^{-1}D^2}{2} +  \frac{\teta L^2}{2}\|v_t\|^2
    \end{align*}
From Lemma \ref{lemma:e_t},
\begin{align*}
    e_{t-1} \le  (1-\teta \mu c)e_t -  \frac{\teta L^2}{2}\|v_t\|^2
\end{align*}
From Lemma \ref{lemma:w_t},
 \begin{align*}
         w_{t-1} \le (1-\teta \mu c)w_t - \frac{(\sigma_t^2 - \sigma_{t-1}^2)\eta^{-1}D^2}{2},
    \end{align*}
    where $\teta$ is defined as in \eqref{eq:subspace-dynamics}. Hence summing the above three equations we get
    \begin{align*}
        &\Delta_{t-1} + e_{t-1} + w_{t-1} \le (1-\teta \mu c)  \Delta_{t} + e_{t} + w_{t} \\
        \Longrightarrow\quad &  V_{t-1}(x_{t-1} ) \le (1-\eta\epsilon\mu c) V_t(x_t).
    \end{align*}
    
    It remains to relate $V_0(x_0)$ to the suboptimality of the projected iterate. 
Since $x_0^\star$ minimizes $\Phi_0$ over $\range(A^\top)$, we have
\begin{align*}
    \Phi_0(x_0^\star)
    \le
    \Phi_0(x^\star).
\end{align*}
Moreover, since $A^\top Ax_0,x^\star\in\range(A^\top)$, we have
\begin{align*}
    \Phi_0(A^\top Ax_0)
    &=
    f(A^\top Ax_0)
    +
    \frac{\sigma_0^2}{2\eta}
    \|Ax_0\|_{\Sigma^{-1}}^2,\\
    \Phi_0(x^\star)
    &=
    f(x^\star)
    +
    \frac{\sigma_0^2}{2\eta}
    \|Ax^\star\|_{\Sigma^{-1}}^2.
\end{align*}
Therefore,
\begin{align*}
    f(A^\top Ax_0)-f(x^\star)
    &=
    \Phi_0(A^\top Ax_0)-\Phi_0(x_0^\star)+
    \Phi_0(x_0^\star)-\Phi_0(x^\star)+
    \frac{\sigma_0^2}{2\eta}
    \left(
        \|Ax^\star\|_{\Sigma^{-1}}^2
        -
        \| Ax_0\|_{\Sigma^{-1}}^2
    \right)\\
    &\le
    \Phi_0(A^\top Ax_0)-\Phi_0(x_0^\star)
    +
    \frac{\sigma_0^2}{2\eta}
    \|Ax^\star\|_{\Sigma^{-1}}^2\\
    &=
    \Delta_0
    +
    \frac{\sigma_0^2}{2\eta}
    \|Ax^\star\|_{\Sigma^{-1}}^2\\
    &\le
    V_0(x_0)
    +
    \frac{\sigma_0^2}{2\eta}
    \|Ax^\star\|_{\Sigma^{-1}}^2\\
&\le
    (1-\eta\epsilon\mu c)^{t_0}
    V_{t_0}(x_{t_0})
    +
    (1-\epsilon)^{2t_0}\frac{\sigma_T^2}{2\eta}
    \|Ax^\star\|_{\Sigma^{-1}}^2.
\end{align*}
Finally, by Lemma~\ref{lemma:guided-diffusion-u-v},
\begin{align*}
    A^\perp x_{t-1}
    =
    (1-\epsilon)A^\perp x_t.
\end{align*}
Iterating from $t=t_0$ down to $t=1$ gives
\begin{align*}
    \|A^\perp x_0\|
    =
    (1-\epsilon)^{t_0}\|A^\perp x_{t_0}\|.
\end{align*}
This completes the proof.
\end{proof}

\subsection{Proof of auxiliary lemma}
\begin{lemma}\label{lemma:Delta_t} 
Let $u_t, v_t$ be defined as in \eqref{eq:def-u-v} and let $\Delta_t := (\Phi_{t}(A^\top u_{t}) - \Phi_{t}(x_{t}^\star))$ and $\teta$ defined as in \eqref{eq:subspace-dynamics}, then under the conditions specified in Theorem \ref{theorem: convergence-guided-diffusion}, we have
    \begin{align*}
        \Delta_{t-1} \le (1-\teta \mu c)\Delta_t + \frac{(\sigma_t^2 - \sigma_{t-1}^2)\eta^{-1}D^2}{2} +  \frac{\teta L^2}{2}\|v_t\|^2
    \end{align*}
\begin{proof}
    We start with bounding
    \begin{align*}
        \mathrm{LHS}
        &:= \bigl(\Phi_t(A^\top u_{t-1})-\Phi_t(x_t^\star)\bigr)
        -\bigl(\Phi_t(A^\top u_t)-\Phi_t(x_t^\star)\bigr)\\
        &=\Phi_t(A^\top u_{t-1})-\Phi_t(A^\top u_t).
    \end{align*}
    From the $L$-smoothness of $f$, the definition of $\Phi_t$, and the fact that
    $AA^\top=I$, we have
    \begin{align*}
        \mathrm{LHS}
        &\le
        \left\langle
        \nabla\Phi_t(A^\top u_t),
        A^\top(u_{t-1}-u_t)
        \right\rangle\\
        &\quad
        +\frac{1}{2}
        \left\langle
        u_{t-1}-u_t,
        \left(
        LI+\frac{\sigma_t^2}{\eta}\Sigma^{-1}
        \right)
        (u_{t-1}-u_t)
        \right\rangle.
    \end{align*}
    By \eqref{eq:subspace-dynamics},
    \[
        u_{t-1}-u_t
        =
        -\teta M_tA\nabla\Phi_t(x_t).
    \]
    Therefore,
    \begin{align*}
        \mathrm{LHS}
        &\le
        -\teta
        \left\langle
        \nabla\Phi_t(A^\top u_t),
        A^\top M_tA\nabla\Phi_t(x_t)
        \right\rangle\\
        &\quad
        +\frac{\teta^2}{2}
        \left\langle
        M_tA\nabla\Phi_t(x_t),
        \left(
        LI+\frac{\sigma_t^2}{\eta}\Sigma^{-1}
        \right)
        M_tA\nabla\Phi_t(x_t)
        \right\rangle.
    \end{align*}
    Note that, since $\teta=\eta\epsilon$, $\teta L\le 1$, and
    $\epsilon\le 1$,
    \begin{align*}
        \teta
        \left(
        LI+\frac{\sigma_t^2}{\eta}\Sigma^{-1}
        \right)
        &=
        \teta LI+\epsilon\sigma_t^2\Sigma^{-1}\\
        &\preceq
        I+\sigma_t^2\Sigma^{-1}
        =
        M_t^{-1}.
    \end{align*}
    Hence,
    \begin{align*}
        \teta M_t
        \left(
        LI+\frac{\sigma_t^2}{\eta}\Sigma^{-1}
        \right)
        M_t
        \preceq M_t.
    \end{align*}
    It follows that
    \begin{align*}
        \mathrm{LHS}
        &\le
        -\teta
        \left\langle
        \nabla\Phi_t(A^\top u_t),
        A^\top M_tA\nabla\Phi_t(x_t)
        \right\rangle\\
        &\quad
        +\frac{\teta}{2}
        \left\langle
        \nabla\Phi_t(x_t),
        A^\top M_tA\nabla\Phi_t(x_t)
        \right\rangle.
    \end{align*}
    Define
    \[
        \delta_t
        :=
        \nabla\Phi_t(x_t)-\nabla\Phi_t(A^\top u_t).
    \]
    Then
    \begin{align*}
        \mathrm{LHS}
        &\le
        -\teta
        \left\langle
        \nabla\Phi_t(A^\top u_t),
        A^\top M_tA
        \bigl(\nabla\Phi_t(A^\top u_t)+\delta_t\bigr)
        \right\rangle\\
        &\quad
        +\frac{\teta}{2}
        \left\langle
        \nabla\Phi_t(A^\top u_t)+\delta_t,
        A^\top M_tA
        \bigl(\nabla\Phi_t(A^\top u_t)+\delta_t\bigr)
        \right\rangle\\
        &=
        -\frac{\teta}{2}
        \left\|A\nabla\Phi_t(A^\top u_t)\right\|_{M_t}^2
        +\frac{\teta}{2}\|A\delta_t\|_{M_t}^2.
    \end{align*}
    For $t\le t_0$, the definition of $t_0$ gives
    \[
        cI\preceq
        M_t=(I+\sigma_t^2\Sigma^{-1})^{-1}
        \preceq I.
    \]
    Therefore,
    \begin{align*}
        \mathrm{LHS}
        &\le
        -\frac{\teta c}{2}
        \left\|A\nabla\Phi_t(A^\top u_t)\right\|^2
        +\frac{\teta}{2}\|\delta_t\|^2.
    \end{align*}
    Since the restriction of $\Phi_t$ to $\operatorname{range}(A^\top)$
    is $\mu$-strongly convex, we have
    \begin{align*}
        \frac{1}{2}
        \left\|A\nabla\Phi_t(A^\top u_t)\right\|^2
        \ge
        \mu\bigl(
        \Phi_t(A^\top u_t)-\Phi_t(x_t^\star)
        \bigr).
    \end{align*}
    Moreover, the quadratic terms in
    $\nabla\Phi_t(x_t)$ and $\nabla\Phi_t(A^\top u_t)$ are identical
    because $Ax_t=AA^\top u_t=u_t$. Hence,
    \begin{align*}
        \delta_t
        &=
        \nabla f(x_t)-\nabla f(A^\top u_t),
    \end{align*}
    and therefore
    \begin{align*}
        \|\delta_t\|
        &\le
        L\|x_t-A^\top u_t\|
        =
        L\|x_t-A^\top Ax_t\|
        =
        L\|v_t\|.
    \end{align*}
    Consequently,
    \begin{align*}
        \mathrm{LHS}
        &\le
        -\teta c\mu
        \bigl(
        \Phi_t(A^\top u_t)-\Phi_t(x_t^\star)
        \bigr)
        +\frac{\teta L^2}{2}\|v_t\|^2.
    \end{align*}
    Substituting the definition of $\mathrm{LHS}$ gives
    \begin{align*}
        \Phi_t(A^\top u_{t-1})-\Phi_t(x_t^\star)
        &\le
        (1-\teta c\mu)
        \bigl(
        \Phi_t(A^\top u_t)-\Phi_t(x_t^\star)
        \bigr)
        +\frac{\teta L^2}{2}\|v_t\|^2\\
        &=
        (1-\teta c\mu)\Delta_t
        +\frac{\teta L^2}{2}\|v_t\|^2.
    \end{align*}

    Further,
    \begin{align*}
        &\Delta_{t-1}
        -
        \bigl(
        \Phi_t(A^\top u_{t-1})-\Phi_t(x_t^\star)
        \bigr)\\
        &=
        \bigl(
        \Phi_{t-1}(A^\top u_{t-1})
        -\Phi_{t-1}(x_{t-1}^\star)
        \bigr)
        -
        \bigl(
        \Phi_t(A^\top u_{t-1})
        -\Phi_t(x_t^\star)
        \bigr)\\
        &=
        \underbrace{
        \Phi_{t-1}(A^\top u_{t-1})
        -\Phi_t(A^\top u_{t-1})
        }_{\le 0}
        -
        \bigl(
        \Phi_{t-1}(x_{t-1}^\star)-\Phi_t(x_t^\star)
        \bigr)\\
        &\le
        \Phi_t(x_{t-1}^\star)-\Phi_{t-1}(x_{t-1}^\star)\\
        &=
        \frac{(\sigma_t^2-\sigma_{t-1}^2)\eta^{-1}}{2}
        \|Ax_{t-1}^\star\|_{\Sigma^{-1}}^2\\
        &\le
        \frac{
        (\sigma_t^2-\sigma_{t-1}^2)\eta^{-1}D^2
        }{2}.
    \end{align*}
    Combining the preceding two inequalities yields
    \begin{align*}
        \Delta_{t-1}
        &\le
        (1-\teta c\mu)\Delta_t
        +\frac{\teta L^2}{2}\|v_t\|^2
        +\frac{
        (\sigma_t^2-\sigma_{t-1}^2)\eta^{-1}D^2
        }{2}.
    \end{align*}
\end{proof}

\end{lemma}

\begin{lemma}\label{lemma:e_t}
    Define 
    \begin{align*}
        e_t:= \frac{\eta L^2}{2(1-\eta\mu c)} \|v_t\|^2,
    \end{align*}
    where $v_t$ is defined as in \eqref{eq:def-u-v}, then under the conditions specified in Theorem \ref{theorem: convergence-guided-diffusion} and for $\teta$ defined as in \eqref{eq:subspace-dynamics}, we have that 
    \begin{align*}
        e_{t-1} \le  (1-\teta \mu c)e_t -  \frac{\teta L^2}{2}\|v_t\|^2
    \end{align*}
    \begin{proof} For $\|v_t\| = 0$ the statement is trivial. For $\|v_t\|>0$ we have
        \begin{align*}
            &e_{t-1} \le  (1-\teta \mu c)e_t -  \frac{\teta L^2}{2}\|v_t\|^2\\
            \stackrel{\textup{Lemma \ref{lemma:guided-diffusion-u-v}}}{\Longleftrightarrow}\quad& \frac{\eta L^2}{2(1-\eta\mu c)} \frac{\sigma_{t-1}^2}{\sigma_t^2} \le (1-\teta \mu c)\frac{\eta L^2}{2(1-\eta\mu c)} -  \frac{\teta L^2}{2}\\
            \Longleftrightarrow\quad & \frac{\sigma_{t-1}^2}{\sigma_t^2}\le (1-\teta \mu c) - \frac{\teta (1-\eta\mu c)}{\eta}\\
            \Longleftrightarrow\quad & (1-\epsilon)^2=\frac{\sigma_{t-1}^2}{\sigma_t^2}\le (1-\teta \mu c) - \epsilon (1-\eta\mu c) = 1-\epsilon\eta\mu c -\epsilon + \epsilon\eta\mu c = 1-\epsilon,
        \end{align*}
        given that $\epsilon \in (0,1)$, it completes the proof.
    \end{proof}
\end{lemma}

\begin{lemma}\label{lemma:w_t}
Define
    \begin{align*}
        w_t := \frac{\eta^{-1}D^2}{2(1-\eta\mu c)}\sigma_t^2,
    \end{align*}
    then under the conditions specified in Theorem \ref{theorem: convergence-guided-diffusion} and for $\teta$ defined as in \eqref{eq:subspace-dynamics}, we have that 
    \begin{align*}
         w_{t-1} \le (1-\teta \mu c)w_t - \frac{(\sigma_t^2 - \sigma_{t-1}^2)\eta^{-1}D^2}{2}.
    \end{align*}
    \begin{proof}
        For simplicity, denote $C:= \frac{\eta^{-1}D^2}{2}$, hence
        \begin{align*}
            &w_{t-1} \le (1-\teta \mu c)w_t - \frac{(\sigma_t^2 - \sigma_{t-1}^2)\eta^{-1}D^2}{2}\\
            \Longleftrightarrow\quad & \frac{C}{1-\eta\mu c} \sigma_{t-1}^2 \le  (1-\teta \mu c)\frac{C}{1-\eta\mu c} \sigma_{t}^2 - C(\sigma_t^2 - \sigma_{t-1}^2)\\
           \Longleftrightarrow\quad  & \left(\frac{1}{1-\eta\mu c}-1\right)\sigma_{t-1}^2 \le \left(\frac{1}{1-\eta\mu c}-1\right)\sigma_{t}^2 - \frac{\teta\mu c }{1-\eta\mu c}\sigma_{t}^2\\
           \Longleftrightarrow\quad  & \frac{\eta\mu c}{1-\eta\mu c} \frac{\sigma_{t-1}^2}{\sigma_t^2} \le \frac{\eta\mu c}{1-\eta\mu c}  -  \frac{\teta\mu c }{1-\eta\mu c}\\
           \Longleftrightarrow\quad &\frac{\teta\mu c }{1-\eta\mu c}  \le \frac{\eta\mu c}{1-\eta\mu c}\left(1- \frac{\sigma_{t-1}^2}{\sigma_t^2}\right)\\
           \Longleftrightarrow\quad &\epsilon \le \left(1- \frac{\sigma_{t-1}^2}{\sigma_t^2}\right) = (1-(1-\epsilon)^2) = 2\epsilon - \epsilon^2
        \end{align*}
        Given that $\epsilon < 1$, we have proved the inequality holds.
    \end{proof}
\end{lemma}

\section{Convergence Analysis for the Convex Region Setting}\label{app:convex}
\subsection{Proof of Lemma \ref{lem:stein-convex-projection}}

\begin{proof}[Proof of Lemma \ref{lem:stein-convex-projection}]
Fix $x\in\bR^n$ and let
\begin{align*}
    p:=\pi_{\cX}(x).
\end{align*}
By the Tweedie--Miyasawa identity,
\begin{align*}
    \widehat\pi_\sigma(x)
    =
    \bE[X\mid Y=x].
\end{align*}
Therefore, the conditional law of $X$ given $Y=x$ has density on
$\cX$ proportional to
\begin{align*}
    q
    \longmapsto
    \exp\left(
        -\frac{\|x-q\|^2}{2\sigma^2}
    \right)\vartheta(q)
\end{align*}
with respect to $d\mathcal H^d(q)$. Hence
\begin{align*}
    \widehat\pi_\sigma(x)-p
    =
    \frac{
        \displaystyle
        \int_{\cX}
        (q-p)
        \exp\left(
            -\frac{\|x-q\|^2-\|x-p\|^2}{2\sigma^2}
        \right)
        \vartheta(q)\,d\mathcal H^d(q)
    }{
        \displaystyle
        \int_{\cX}
        \exp\left(
            -\frac{\|x-q\|^2-\|x-p\|^2}{2\sigma^2}
        \right)
        \vartheta(q)\,d\mathcal H^d(q)
    }.
\end{align*}

Choose an affine isometry
\begin{align*}
    F:\bR^d\longrightarrow\cA
\end{align*}
and define
\begin{align*}
    K:=F^{-1}(\cX),
    \qquad
    y:=F^{-1}\bigl(\pi_{\cA}(x)\bigr),
    \qquad
    u:=F^{-1}(p).
\end{align*}
Since $\cX\subseteq\cA$, for every $q\in\cX$ one has
\begin{align*}
    \|x-q\|^2
    =
    \|x-\pi_{\cA}(x)\|^2
    +
    \|\pi_{\cA}(x)-q\|^2.
\end{align*}
The first term is independent of $q$, so $p$ is also the Euclidean
projection of $\pi_{\cA}(x)$ onto $\cX$. Since $F$ is an isometry,
it follows that
\begin{align*}
    u=\pi_K(y).
\end{align*}
Moreover, for every $\xi\in K$,
\begin{align*}
    \|x-F(\xi)\|^2-\|x-p\|^2
    =
    \|y-\xi\|^2-\|y-u\|^2.
\end{align*}

Set
\begin{align*}
    D:=K-u,
    \qquad
    v:=u-y.
\end{align*}
Then $D\subseteq\bR^d$ is compact and convex and contains the origin.
By the variational characterization of the Euclidean projection,
\begin{align*}
    \langle y-u,\xi-u\rangle
    \le
    0,
    \qquad
    \xi\in K.
\end{align*}
Equivalently,
\begin{align*}
    \langle v,h\rangle
    \ge
    0,
    \qquad
    h\in D.
\end{align*}
Furthermore,
\begin{align*}
    \|y-(u+h)\|^2-\|y-u\|^2
    =
    \|h\|^2+2\langle v,h\rangle,
    \qquad
    h\in D.
\end{align*}

For every $\omega\in\mathbb S^{d-1}$, define
\begin{align*}
    A(\omega)
    :=
    \sup\left\{
        r\ge0:r\omega\in D
    \right\}.
\end{align*}
Since $D$ is compact, convex, and contains the origin,
\begin{align*}
    D\cap\{r\omega:r\ge0\}
    =
    \{r\omega:0\le r\le A(\omega)\}.
\end{align*}
Let
\begin{align*}
    \Omega
    :=
    \left\{
        \omega\in\mathbb S^{d-1}:A(\omega)>0
    \right\}.
\end{align*}
For every $\omega\in\Omega$, choose any
$r\in(0,A(\omega)]$. Since $r\omega\in D$, the projection inequality
gives
\begin{align*}
    0
    \le
    \langle v,r\omega\rangle
    =
    r\langle v,\omega\rangle,
\end{align*}
and therefore
\begin{align*}
    \langle v,\omega\rangle\ge0.
\end{align*}

Define
\begin{align*}
    Z
    &:=
    \int_D
    \exp\left(
        -\frac{\|h\|^2+2\langle v,h\rangle}{2\sigma^2}
    \right)\,dh,\\
    M
    &:=
    \int_D
    \|h\|
    \exp\left(
        -\frac{\|h\|^2+2\langle v,h\rangle}{2\sigma^2}
    \right)\,dh.
\end{align*}
Using polar coordinates on $D$, followed by the change of variables
$r=\sigma t$, yields
\begin{align*}
    Z
    &=
    \int_{\Omega}
    \int_0^{A(\omega)}
    r^{d-1}
    \exp\left(
        -\frac{r^2}{2\sigma^2}
        -
        \frac{r\langle v,\omega\rangle}{\sigma^2}
    \right)
    \,dr\,d\omega\\
    &=
    \sigma^d
    \int_{\Omega}
    \int_0^{A(\omega)/\sigma}
    t^{d-1}
    e^{-t^2/2-\beta(\omega)t}
    \,dt\,d\omega,
\end{align*}
and
\begin{align*}
    M
    &=
    \int_{\Omega}
    \int_0^{A(\omega)}
    r^d
    \exp\left(
        -\frac{r^2}{2\sigma^2}
        -
        \frac{r\langle v,\omega\rangle}{\sigma^2}
    \right)
    \,dr\,d\omega\\
    &=
    \sigma^{d+1}
    \int_{\Omega}
    \int_0^{A(\omega)/\sigma}
    t^d
    e^{-t^2/2-\beta(\omega)t}
    \,dt\,d\omega,
\end{align*}
where
\begin{align*}
    \beta(\omega)
    :=
    \frac{\langle v,\omega\rangle}{\sigma}
    \ge
    0.
\end{align*}
By Lemma~\ref{lem:radial-moment-bound}, for every
$\omega\in\Omega$,
\begin{align*}
    \int_0^{A(\omega)/\sigma}
    t^d e^{-t^2/2-\beta(\omega)t}\,dt
    \le
    C_d
    \int_0^{A(\omega)/\sigma}
    t^{d-1}e^{-t^2/2-\beta(\omega)t}\,dt.
\end{align*}
Integrating this inequality over $\Omega$ gives
\begin{align*}
    M
    \le
    C_d\sigma Z.
\end{align*}

Define
\begin{align*}
    \widetilde\vartheta(h)
    :=
    \vartheta\bigl(F(u+h)\bigr),
    \qquad
    h\in D.
\end{align*}
Since $F$ is an isometry, it preserves the $d$-dimensional Hausdorff
measure. Therefore,
\begin{align*}
    \bigl\|
        \widehat\pi_\sigma(x)-p
    \bigr\|
    &\le
    \frac{
        \displaystyle
        \int_D
        \|h\|
        \exp\left(
            -\frac{\|h\|^2+2\langle v,h\rangle}{2\sigma^2}
        \right)
        \widetilde\vartheta(h)\,dh
    }{
        \displaystyle
        \int_D
        \exp\left(
            -\frac{\|h\|^2+2\langle v,h\rangle}{2\sigma^2}
        \right)
        \widetilde\vartheta(h)\,dh
    }\\
    &\le
    \frac{\vartheta_+}{\vartheta_-}
    \frac{
        \displaystyle
        \int_D
        \|h\|
        \exp\left(
            -\frac{\|h\|^2+2\langle v,h\rangle}{2\sigma^2}
        \right)\,dh
    }{
        \displaystyle
        \int_D
        \exp\left(
            -\frac{\|h\|^2+2\langle v,h\rangle}{2\sigma^2}
        \right)\,dh
    }\\
    &=
    \frac{\vartheta_+}{\vartheta_-}
    \frac{M}{Z}\\
    &\le
    C_d\frac{\vartheta_+}{\vartheta_-}\sigma.
\end{align*}
Thus the desired estimate holds with
\begin{align*}
    C_{\cX}
    :=
    C_d\frac{\vartheta_+}{\vartheta_-}.
\end{align*}

Finally, the conditional law of $X$ given $Y=x$ is supported on
$\cX$. Since $\cX$ is compact and convex, its conditional expectation
belongs to $\cX$. Hence
\begin{align*}
    \widehat\pi_\sigma(x)
    =
    \bE[X\mid Y=x]
    \in
    \cX.
\end{align*}
This completes the proof.
\end{proof}

\subsection{Proof of lemma \ref{lemma:one-step-descent-convex}}
We first introduce and proof the following auxiliary lemma.
\begin{lemma}[Exact feasible backward Lyapunov descent]\label{lemma:one-step-descent-convex-exact}
Let $\mathcal X\subseteq\mathbb R^n$ be nonempty, closed, and convex. Suppose that $f:\mathbb R^n\to\mathbb R$ is $\mu$-strongly convex and $L$-smooth on a convex neighborhood containing $\mathcal X$. Let
\begin{align*}
    x^\star:=\arg\min_{x\in\mathcal X} f(x).
\end{align*}
Fix $\eta>0$ and $\epsilon\in(0,1]$ such that
\begin{align*}
    \epsilon(1+L\eta)\le 2.
\end{align*}
For any $p_t\in\mathcal X$, define the exact feasible backward update
\begin{align*}
    y_t:=\pi_{\mathcal X}(p_t-\eta\nabla f(p_t)),
    \qquad
    \bar p_{t-1}:=(1-\epsilon)p_t+\epsilon y_t.
\end{align*}
Define
\begin{align*}
    \mathcal V(p):=f(p)-f(x^\star)+\frac{1}{2\eta}\|p-x^\star\|^2.
\end{align*}
Then
\begin{align*}
    \mathcal V(\bar p_{t-1})\le \rho\mathcal V(p_t),
\end{align*}
where
\begin{align*}
    \rho:=1-\epsilon\min\{1,\eta\mu\}.
\end{align*}
Moreover, if $\eta\mu<1$, then $\rho=1-\epsilon\eta\mu$.
\end{lemma}

\begin{proof}
Let
\begin{align*}
    d_t:=y_t-p_t,
    \qquad
    e_t:=p_t-x^\star.
\end{align*}
Then
\begin{align*}
    \bar p_{t-1}=p_t+\epsilon d_t.
\end{align*}
Since $y_t=\pi_{\mathcal X}(p_t-\eta\nabla f(p_t))$, the projection optimality condition gives
\begin{align*}
    \left\langle y_t-\bigl(p_t-\eta\nabla f(p_t)\bigr),x-y_t\right\rangle\ge 0,
    \qquad \forall x\in\mathcal X.
\end{align*}
Taking $x=x^\star$ and using $d_t=y_t-p_t$, we obtain
\begin{align*}
    \left\langle d_t+\eta\nabla f(p_t),x^\star-y_t\right\rangle\ge 0.
\end{align*}
Since $x^\star-y_t=x^\star-p_t-d_t=-e_t-d_t$, this implies
\begin{align*}
    \left\langle d_t+\eta\nabla f(p_t),e_t+d_t\right\rangle\le 0.
\end{align*}
Expanding the inner product gives
\begin{align*}
    \langle d_t,e_t\rangle+\|d_t\|^2+\eta\langle\nabla f(p_t),e_t\rangle+\eta\langle\nabla f(p_t),d_t\rangle\le 0.
\end{align*}
Therefore,
\begin{align*}
    \langle\nabla f(p_t),d_t\rangle\le -\langle\nabla f(p_t),e_t\rangle-\frac{1}{\eta}\langle d_t,e_t\rangle-\frac{1}{\eta}\|d_t\|^2.
\end{align*}
By $\mu$-strong convexity of $f$, we have
\begin{align*}
    f(x^\star)\ge f(p_t)+\langle\nabla f(p_t),x^\star-p_t\rangle+\frac{\mu}{2}\|p_t-x^\star\|^2.
\end{align*}
Equivalently,
\begin{align*}
    \langle\nabla f(p_t),e_t\rangle\ge f(p_t)-f(x^\star)+\frac{\mu}{2}\|e_t\|^2.
\end{align*}
Hence
\begin{align*}
    \langle\nabla f(p_t),d_t\rangle\le -\bigl(f(p_t)-f(x^\star)\bigr)-\frac{\mu}{2}\|e_t\|^2-\frac{1}{\eta}\langle d_t,e_t\rangle-\frac{1}{\eta}\|d_t\|^2.
\end{align*}
Since $f$ is $L$-smooth and $\bar p_{t-1}=p_t+\epsilon d_t$, we have
\begin{align*}
    f(\bar p_{t-1})\le f(p_t)+\epsilon\langle\nabla f(p_t),d_t\rangle+\frac{L\epsilon^2}{2}\|d_t\|^2.
\end{align*}
Subtracting $f(x^\star)$ and using the previous bound gives
\begin{align*}
    f(\bar p_{t-1})-f(x^\star)
    &\le (1-\epsilon)\bigl(f(p_t)-f(x^\star)\bigr)-\frac{\epsilon\mu}{2}\|e_t\|^2-\frac{\epsilon}{\eta}\langle d_t,e_t\rangle+\left(-\frac{\epsilon}{\eta}+\frac{L\epsilon^2}{2}\right)\|d_t\|^2.
\end{align*}
Moreover, since $\bar p_{t-1}-x^\star=e_t+\epsilon d_t$, we have
\begin{align*}
    \frac{1}{2\eta}\|\bar p_{t-1}-x^\star\|^2=\frac{1}{2\eta}\|e_t\|^2+\frac{\epsilon}{\eta}\langle e_t,d_t\rangle+\frac{\epsilon^2}{2\eta}\|d_t\|^2.
\end{align*}
Adding the previous two estimates gives
\begin{align*}
    \mathcal V(\bar p_{t-1})
    &\le (1-\epsilon)\bigl(f(p_t)-f(x^\star)\bigr)+\left(\frac{1}{2\eta}-\frac{\epsilon\mu}{2}\right)\|e_t\|^2 \\
    &\quad+\left(-\frac{\epsilon}{\eta}+\frac{L\epsilon^2}{2}+\frac{\epsilon^2}{2\eta}\right)\|d_t\|^2.
\end{align*}
The coefficient of $\|d_t\|^2$ satisfies
\begin{align*}
    -\frac{\epsilon}{\eta}+\frac{L\epsilon^2}{2}+\frac{\epsilon^2}{2\eta}
    =\frac{\epsilon}{2\eta}\left(-2+\epsilon L\eta+\epsilon\right)\le 0,
\end{align*}
where the last inequality follows from $\epsilon(1+L\eta)\le 2$. Therefore,
\begin{align*}
    \mathcal V(\bar p_{t-1})
    &\le (1-\epsilon)\bigl(f(p_t)-f(x^\star)\bigr)+\left(1-\epsilon\eta\mu\right)\frac{1}{2\eta}\|p_t-x^\star\|^2.
\end{align*}
Since $p_t\in\mathcal X$ and $x^\star$ minimizes $f$ over $\mathcal X$, we have $f(p_t)-f(x^\star)\ge 0$. Also, $\frac{1}{2\eta}\|p_t-x^\star\|^2\ge 0$. Hence
\begin{align*}
    \mathcal V(\bar p_{t-1})
    &\le \max\{1-\epsilon,1-\epsilon\eta\mu\}\left(f(p_t)-f(x^\star)+\frac{1}{2\eta}\|p_t-x^\star\|^2\right) \\
    &=\left(1-\epsilon\min\{1,\eta\mu\}\right)\mathcal V(p_t).
\end{align*}
This proves the claim.
\end{proof}

\begin{proof}[Proof of Lemma \ref{lemma:one-step-descent-convex}]
Define
\begin{align*}
    p_{t}:=\pi_{\mathcal X}(x_{t}),
    \qquad
    n_{t}:=x_{t}-p_{t}.
\end{align*}
Following the argument in Lemma \ref{lemma:one-step-descent-convex-exact}, with
\begin{align*}
    y_t:=\pi_{\mathcal X}(p_t-\eta\nabla f(p_t)),
    \qquad
    \bar p_{t-1}:=(1-\epsilon)p_t+\epsilon y_t,
\end{align*}
one obtains
\begin{align*}
    f(\bar p_{t-1})-f(x^\star)+\frac{1}{2\eta}\|\bar p_{t-1}-x^\star\|^2\le \rho\mathcal V_t.
\end{align*}
Indeed, writing $d_t:=y_t-p_t$ and $e_t:=p_t-x^\star$, the projection optimality condition gives
\begin{align*}
    \langle\nabla f(p_t),d_t\rangle\le -\bigl(f(p_t)-f(x^\star)\bigr)-\frac{\mu}{2}\|e_t\|^2-\frac{1}{\eta}\langle d_t,e_t\rangle-\frac{1}{\eta}\|d_t\|^2.
\end{align*}
Combining this with $L$-smoothness and the identity $\bar p_{t-1}-x^\star=e_t+\epsilon d_t$ yields
\begin{align*}
    f(\bar p_{t-1})-f(x^\star)+\frac{1}{2\eta}\|\bar p_{t-1}-x^\star\|^2
    &\le (1-\epsilon)\bigl(f(p_t)-f(x^\star)\bigr)+\left(1-\epsilon\eta\mu\right)\frac{1}{2\eta}\|p_t-x^\star\|^2 \\
    &\quad+\frac{\epsilon}{2\eta}\left(-2+\epsilon L\eta+\epsilon\right)\|d_t\|^2.
\end{align*}
Since $\epsilon(1+L\eta)\le 2$, the last term is nonpositive. Since $\eta\mu<1$, we have $1-\epsilon<\rho=1-\epsilon\eta\mu$, and hence the desired reference estimate follows.

Next define
\begin{align*}
    \widehat y_t:=\widehat\pi_t(x_t-\eta\nabla f(x_t)),
    \qquad
    q_{t-1}:=(1-\epsilon)p_t+\epsilon\widehat y_t.
\end{align*}
Since $p_t,\widehat y_t\in\mathcal X$ and $\mathcal X$ is convex, $q_{t-1}\in\mathcal X$. Moreover,
\begin{align*}
    x_{t-1}=(1-\epsilon)x_t+\epsilon\widehat y_t=q_{t-1}+(1-\epsilon)n_t.
\end{align*}
Thus
\begin{align*}
    \|p_{t-1}-q_{t-1}\|\le \|x_{t-1}-q_{t-1}\|=(1-\epsilon)\|n_t\|.
\end{align*}
Also, by nonexpansiveness of projection and $L$-smoothness,
\begin{align*}
    &\quad \|q_{t-1}-\bar p_{t-1}\| = \|((1-\epsilon p_t + \epsilon\widehat y_t))) - ((1-\epsilon)p_t +\epsilon y_t)\| = \epsilon \|\widehat y_t - y_t\|\\
    & = \epsilon\|\widehat\pi_t(x_t-\eta\nabla f(x_t)) - \pi_\cX(p_t - \eta \nabla f(p_t))\|\\
    &\le \epsilon\left(\|\widehat\pi_t(x_t-\eta\nabla f(x_t)) -\pi_\cX(x_t-\eta\nabla f(x_t))\| +\|\pi_\cX(x_t-\eta\nabla f(x_t)) -\pi_\cX(p_t - \eta \nabla f(p_t))\|\right)\\
    &\le \epsilon E_t+\epsilon\|\pi_{\mathcal X}(x_t-\eta\nabla f(x_t))-\pi_{\mathcal X}(p_t-\eta\nabla f(p_t))\| \\
    &\le \epsilon E_t+\epsilon(1+\eta L)\|n_t\|.
\end{align*}
Therefore,
\begin{align*}
    \|p_{t-1}-\bar p_{t-1}\|\le \|q_{t-1}-\bar p_{t-1}\| + \|p_{t-1}-q_{t-1}\|\le A_{\eta,\epsilon}\|n_t\|+\epsilon E_t.
\end{align*}
Let $r_t:=p_{t-1}-\bar p_{t-1}$. Since $\bar p_{t-1},x^\star\in\mathcal X$, we have
\begin{align*}
    &\quad \left\|\nabla f(\bar p_{t-1})+\frac{1}{\eta}(\bar p_{t-1}-x^\star)\right\|\\
    &\le \|\nabla f(x^\star)\| + \left\|\nabla f(\bar p_{t-1}) - \nabla f(x^\star)+\frac{1}{\eta}(\bar p_{t-1}-x^\star)\right\|\\&
    \le G_\star + \left(L+\frac{1}{\eta}\right)D_\cX = M_\eta.
\end{align*}
By $L$-smoothness and expanding the quadratic term,
\begin{align*}
    \mathcal V_{t-1}\le \rho\mathcal V_t+M_\eta\|r_t\|+C_\eta\|r_t\|^2.
\end{align*}
Using $\|r_t\|\le A_{\eta,\epsilon}\|n_t\|+\epsilon E_t$ and $(a+b)^2\le 2a^2+2b^2$, we get
\begin{align*}
    \mathcal V_{t-1}\le \rho\mathcal V_t+B_1\|n_t\|+B_2\|n_t\|^2+\epsilon M_\eta E_t+2\epsilon^2C_\eta E_t^2.
\end{align*}
The normal component satisfies
\begin{align*}
    \|n_{t-1}\|=\operatorname{dist}(x_{t-1},\mathcal X)\le \|x_{t-1}-q_{t-1}\|=(1-\epsilon)\|n_t\|.
\end{align*}
This completes the proof.
\end{proof}

% \gz{%One typo that changes a constant: in the display bounding $\|\nabla f(\bar p_{t-1})+\frac1\eta(\bar p_{t-1}-x^\star)\|$, the last line reads $G_\star + (1+\frac1\eta)D_\cX$; using $L$-smoothness it should be $G_\star + (L+\frac1\eta)D_\cX$, which is exactly the $M_\eta$ defined in the lemma statement --- so the statement is right and only the proof line has the typo ($1\to L$). 
% Also: the triangle-inequality step ``$\|q_{t-1}-\bar p_{t-1}\| \le \epsilon E_t + \epsilon\|\pi_\cX(x_t-\eta\nabla f(x_t)) - \pi_\cX(p_t-\eta\nabla f(p_t))\|$'' inserts/removes $\pi_\cX(x_t-\eta\nabla f(x_t))$; spell out $\widehat y_t - y_t = [\widehat\pi_t(z_t)-\pi_\cX(z_t)] + [\pi_\cX(z_t)-\pi_\cX(\bar z_t)]$ with $z_t, \bar z_t$ defined, for readability.}

\subsection{Proof of Theorem \ref{thm:convergence-convex}}

\begin{proof}[Proof of Theorem \ref{thm:convergence-convex}]
For each $t=T,T-1,\ldots,1$, Lemma~\ref{lem:stein-convex-projection} gives
\begin{align*}
    \|\widehat\pi_t(z)-\pi_{\mathcal X}(z)\|
    \le
    C_{\mathcal X}\sigma_t
    =
    C_{\mathcal X}\sigma_T(1-\epsilon)^{T-t},
    \qquad
    \forall z\in\mathbb R^n.
\end{align*}
Therefore, Lemma~\ref{lemma:one-step-descent-convex} applies with
\begin{align*}
    E_t=C_{\mathcal X}\sigma_T(1-\epsilon)^{T-t}.
\end{align*}
Moreover, the normal-component contraction in Lemma~\ref{lemma:one-step-descent-convex} gives
\begin{align*}
    \|n_t\|\le (1-\epsilon)^{T-t}\|n_T\|,
    \qquad
    t=1,\ldots,T.
\end{align*}
Hence,
\begin{align*}
    \mathcal V_{t-1}
    &\le
    \rho\mathcal V_t
    +
    \left(B_1\|n_T\|+\epsilon M_\eta C_{\mathcal X}\sigma_T\right)
    (1-\epsilon)^{T-t}\\
    &\quad+
    \left(B_2\|n_T\|^2+2\epsilon^2C_\eta C_{\mathcal X}^2\sigma_T^2\right)
    (1-\epsilon)^{2(T-t)}.
\end{align*}
Iterating this recursion from $T$ down to $0$ yields
\begin{align*}
    \mathcal V_0
    &\le
    \rho^T\mathcal V_T
    +
    \left(B_1\|n_T\|+\epsilon M_\eta C_{\mathcal X}\sigma_T\right)
    \sum_{t=1}^{T}\rho^{t-1}(1-\epsilon)^{T-t}\\
    &\quad+
    \left(B_2\|n_T\|^2+2\epsilon^2C_\eta C_{\mathcal X}^2\sigma_T^2\right)
    \sum_{t=1}^{T}\rho^{t-1}(1-\epsilon)^{2(T-t)}.
\end{align*}
Since $\eta\mu<1$, we have
\begin{align*}
    1-\epsilon<\rho=1-\epsilon\eta\mu.
\end{align*}
Therefore,
\begin{align*}
    \sum_{t=1}^{T}\rho^{t-1}(1-\epsilon)^{T-t}
    &=
    \frac{\rho^T-(1-\epsilon)^T}{\rho-(1-\epsilon)},\\
    \sum_{t=1}^{T}\rho^{t-1}(1-\epsilon)^{2(T-t)}
    &=
    \frac{\rho^T-(1-\epsilon)^{2T}}{\rho-(1-\epsilon)^2}.
\end{align*}
Substituting these identities gives
\begin{align*}
    \mathcal V_0
    &\le
    \rho^T\mathcal V_T
    +
    \left(B_1\|n_T\|+\epsilon M_\eta C_{\mathcal X}\sigma_T\right)
    \frac{\rho^T-(1-\epsilon)^T}{\rho-(1-\epsilon)}\\
    &\quad+
    \left(B_2\|n_T\|^2+2\epsilon^2C_\eta C_{\mathcal X}^2\sigma_T^2\right)
    \frac{\rho^T-(1-\epsilon)^{2T}}{\rho-(1-\epsilon)^2}.
\end{align*}
This proves the bound on $\mathcal V_0$.

The normal-component estimate follows directly from Lemma~\ref{lemma:one-step-descent-convex}:
\begin{align*}
    \|n_{t-1}\|\le (1-\epsilon)\|n_t\|.
\end{align*}
Iterating from $T$ down to $0$ gives
\begin{align*}
    \|n_0\|\le (1-\epsilon)^T\|n_T\|.
\end{align*}
Finally, since
\begin{align*}
    \mathcal V_0
    =
    f(p_0)-f(x^\star)
    +
    \frac{1}{2\eta}\|p_0-x^\star\|^2,
\end{align*}
and both terms are nonnegative, we have
\begin{align*}
    f(p_0)-f(x^\star)\le\mathcal V_0,
    \qquad
    \|p_0-x^\star\|^2\le 2\eta\mathcal V_0.
\end{align*}
The stated bounds follow by substituting the bound on $\mathcal V_0$.
\end{proof}

\subsection{Proof of Lemma \ref{lem:one-step-descent-convex-nonconvex} and Theorem \ref{thm:backward-stein-nonconvex}}
\begin{proof}[Proof of Lemma \ref{lem:one-step-descent-convex-nonconvex}]
Define the exact feasible reference update
\begin{align*}
    y_t:=\pi_{\cX}(p_t-\eta\nabla f(p_t)),
    \qquad
    \bar p_{t-1}:=(1-\epsilon)p_t+\epsilon y_t.
\end{align*}
Let
\begin{align*}
    d_t:=y_t-p_t.
\end{align*}
Then
\begin{align*}
    \bar p_{t-1}=p_t+\epsilon d_t,
    \qquad
    \mathcal G_\eta(p_t)=-\frac{1}{\eta}d_t.
\end{align*}
Since $y_t=\pi_{\cX}(p_t-\eta\nabla f(p_t))$, the projection optimality condition gives
\begin{align*}
    \left\langle y_t-\bigl(p_t-\eta\nabla f(p_t)\bigr),x-y_t\right\rangle\ge 0,
    \qquad \forall x\in\cX.
\end{align*}
Taking $x=p_t$ gives
\begin{align*}
    \left\langle d_t+\eta\nabla f(p_t),-d_t\right\rangle\ge 0.
\end{align*}
Equivalently,
\begin{align*}
    \langle\nabla f(p_t),d_t\rangle\le -\frac{1}{\eta}\|d_t\|^2.
\end{align*}
By $L$-smoothness and $\bar p_{t-1}=p_t+\epsilon d_t$,
\begin{align*}
    f(\bar p_{t-1})
    &\le f(p_t)+\epsilon\langle\nabla f(p_t),d_t\rangle+\frac{L\epsilon^2}{2}\|d_t\|^2 \\
    &\le f(p_t)-\left(\frac{\epsilon}{\eta}-\frac{L\epsilon^2}{2}\right)\|d_t\|^2 \\
    &=f(p_t)-\epsilon\eta\left(1-\frac{\epsilon\eta L}{2}\right)\|\mathcal G_\eta(p_t)\|^2 \\
    &=f(p_t)-\alpha_{\eta,\epsilon}\|\mathcal G_\eta(p_t)\|^2.
\end{align*}

Next compare the actual projected point $p_{t-1}$ with $\bar p_{t-1}$. Define
\begin{align*}
    \widehat y_t:=\widehat\pi_t(x_t-\eta\nabla f(x_t)),
    \qquad
    q_{t-1}:=(1-\epsilon)p_t+\epsilon\widehat y_t.
\end{align*}
Since $p_t,\widehat y_t\in\cX$ and $\cX$ is convex, $q_{t-1}\in\cX$. Moreover,
\begin{align*}
    x_{t-1}=(1-\epsilon)x_t+\epsilon\widehat y_t=q_{t-1}+(1-\epsilon)n_t.
\end{align*}
Thus, by nonexpansiveness of projection,
\begin{align*}
    \|p_{t-1}-q_{t-1}\|
    =
    \|\pi_{\cX}(x_{t-1})-\pi_{\cX}(q_{t-1})\|
    \le
    \|x_{t-1}-q_{t-1}\|
    =
    (1-\epsilon)\|n_t\|.
\end{align*}
Also,
\begin{align*}
    \|q_{t-1}-\bar p_{t-1}\|
    &=
    \epsilon\|\widehat y_t-y_t\| \\
    &\le
    \epsilon\|\widehat\pi_t(x_t-\eta\nabla f(x_t))-\pi_{\cX}(x_t-\eta\nabla f(x_t))\| \\
    &\quad+
    \epsilon\|\pi_{\cX}(x_t-\eta\nabla f(x_t))-\pi_{\cX}(p_t-\eta\nabla f(p_t))\| \\
    &\le
    \epsilon E_t+\epsilon\|x_t-\eta\nabla f(x_t)-p_t+\eta\nabla f(p_t)\| \\
    &\le
    \epsilon E_t+\epsilon(1+\eta L)\|n_t\|.
\end{align*}
Therefore,
\begin{align*}
    \|p_{t-1}-\bar p_{t-1}\|
    &\le \|p_{t-1}-q_{t-1}\|+\|q_{t-1}-\bar p_{t-1}\| \\
    &\le (1+\epsilon\eta L)\|n_t\|+\epsilon E_t \\
    &=A_{\eta,\epsilon}\|n_t\|+\epsilon E_t.
\end{align*}
Let
\begin{align*}
    r_t:=p_{t-1}-\bar p_{t-1}.
\end{align*}
Then
\begin{align*}
    \|r_t\|\le A_{\eta,\epsilon}\|n_t\|+\epsilon E_t.
\end{align*}

Since $p_{t-1},\bar p_{t-1}\in\cX$ and $G:=\sup_{x\in\cX}\|\nabla f(x)\|$, $L$-smoothness gives
\begin{align*}
    f(p_{t-1})
    &\le f(\bar p_{t-1})+\langle\nabla f(\bar p_{t-1}),r_t\rangle+\frac{L}{2}\|r_t\|^2 \\
    &\le f(\bar p_{t-1})+G\|r_t\|+\frac{L}{2}\|r_t\|^2.
\end{align*}
Using the bound on $\|r_t\|$ and $(a+b)^2\le 2a^2+2b^2$, we get
\begin{align*}
    f(p_{t-1})
    &\le f(\bar p_{t-1})
    +GA_{\eta,\epsilon}\|n_t\|
    +LA_{\eta,\epsilon}^2\|n_t\|^2
    +\epsilon GE_t
    +\epsilon^2LE_t^2.
\end{align*}
Combining this with the exact feasible descent estimate gives
\begin{align*}
    f(p_{t-1})
    &\le f(p_t)-\alpha_{\eta,\epsilon}\|\mathcal G_\eta(p_t)\|^2
    +\Gamma_1\|n_t\|
    +\Gamma_2\|n_t\|^2
    +\epsilon GE_t
    +\epsilon^2LE_t^2.
\end{align*}
It remains to control the normal component. Since $q_{t-1}\in\cX$ and $x_{t-1}=q_{t-1}+(1-\epsilon)n_t$, we have
\begin{align*}
\|n_{t-1}\|
=
\operatorname{dist}(x_{t-1},\cX)
\le
\|x_{t-1}-q_{t-1}\|
=
(1-\epsilon)\|n_t\|.
\end{align*}
This completes the proof.
\end{proof}

\begin{proof}[Proof of Theorem \ref{thm:backward-stein-nonconvex}]
By Lemma~\ref{lem:stein-convex-projection}, for every $z\in\bR^n$ and every $t$,
\begin{align*}
    \|\widehat\pi_t(z)-\pi_{\cX}(z)\|
    &\le
    C_{\cX}\sigma_T(1-\epsilon)^{T-t}.
\end{align*}
Therefore, Lemma~\ref{lem:one-step-descent-convex-nonconvex} applies with
\begin{align*}
    E_t
    &=
    C_{\cX}\sigma_T(1-\epsilon)^{T-t}.
\end{align*}
Moreover, by the normal-component contraction in Lemma~\ref{lem:one-step-descent-convex-nonconvex}, for every $t=1,\ldots,T$,
\begin{align*}
    \|n_t\|
    &\le
    (1-\epsilon)^{T-t}\|n_T\|.
\end{align*}
Hence, for every $t=T,T-1,\ldots,1$,
\begin{align*}
    \alpha_{\eta,\epsilon}\|\mathcal G_\eta(p_t)\|^2
    &\le
    f(p_t)-f(p_{t-1})
    +
    \Gamma_1(1-\epsilon)^{T-t}\|n_T\|
    +
    \Gamma_2(1-\epsilon)^{2(T-t)}\|n_T\|^2\\
    &\quad+
    \epsilon GC_{\cX}\sigma_T(1-\epsilon)^{T-t}
    +
    \epsilon^2LC_{\cX}^2\sigma_T^2(1-\epsilon)^{2(T-t)}.
\end{align*}
Summing from $t=1$ to $T$ yields
\begin{align*}
    \alpha_{\eta,\epsilon}
    \sum_{t=1}^{T}\|\mathcal G_\eta(p_t)\|^2
    &\le
    \sum_{t=1}^{T}\left(f(p_t)-f(p_{t-1})\right)
    +
    \Gamma_1\|n_T\|
    \sum_{t=1}^{T}(1-\epsilon)^{T-t}\\
    &\quad+
    \Gamma_2\|n_T\|^2
    \sum_{t=1}^{T}(1-\epsilon)^{2(T-t)}
    +
    \epsilon GC_{\cX}\sigma_T
    \sum_{t=1}^{T}(1-\epsilon)^{T-t}\\
    &\quad+
    \epsilon^2LC_{\cX}^2\sigma_T^2
    \sum_{t=1}^{T}(1-\epsilon)^{2(T-t)}.
\end{align*}
The first sum telescopes:
\begin{align*}
    \sum_{t=1}^{T}\left(f(p_t)-f(p_{t-1})\right)
    &=
    f(p_T)-f(p_0).
\end{align*}
Moreover,
\begin{align*}
    \sum_{t=1}^{T}(1-\epsilon)^{T-t}
    &=
    \frac{1-(1-\epsilon)^T}{\epsilon},
\end{align*}
and
\begin{align*}
    \sum_{t=1}^{T}(1-\epsilon)^{2(T-t)}
    &=
    \frac{1-(1-\epsilon)^{2T}}
    {1-(1-\epsilon)^2}.
\end{align*}
Therefore,
\begin{align*}
    \alpha_{\eta,\epsilon}
    \sum_{t=1}^{T}\|\mathcal G_\eta(p_t)\|^2
    &\le
    f(p_T)-f(p_0)
    +
    \frac{\Gamma_1\|n_T\|}{\epsilon}
    \left(1-(1-\epsilon)^T\right)\\
    &\quad+
    \frac{\Gamma_2\|n_T\|^2}
    {1-(1-\epsilon)^2}
    \left(1-(1-\epsilon)^{2T}\right)\\
    &\quad+
    GC_{\cX}\sigma_T
    \left(1-(1-\epsilon)^T\right)\\
    &\quad+
    \frac{\epsilon^2LC_{\cX}^2\sigma_T^2}
    {1-(1-\epsilon)^2}
    \left(1-(1-\epsilon)^{2T}\right).
\end{align*}
Since $0<\epsilon\le 1$, one has
\begin{align*}
    1-(1-\epsilon)^T
    &\le
    1,\\
    1-(1-\epsilon)^{2T}
    &\le
    1,\\
    1-(1-\epsilon)^2
    &=
    \epsilon(2-\epsilon)
    \ge
    \epsilon.
\end{align*}
Together with
\begin{align*}
    f(p_T)-f(p_0)
    \le
    \max_{x\in\cX}f(x)-\min_{x\in\cX}f(x),
\end{align*}
this gives
\begin{align*}
    \alpha_{\eta,\epsilon}
    \sum_{t=1}^{T}\|\mathcal G_\eta(p_t)\|^2
    &\le
    \max_{x\in\cX}f(x)-\min_{x\in\cX}f(x)
    +
    \frac{\Gamma_1\|n_T\|}{\epsilon}\\
    &\quad+
    \frac{\Gamma_2\|n_T\|^2}{\epsilon}
    +
    GC_{\cX}\sigma_T
    +
    \epsilon LC_{\cX}^2\sigma_T^2\\
    &=
    \Delta_T.
\end{align*}
This proves the cumulative bound.

Dividing by $\alpha_{\eta,\epsilon}T$ and using
\begin{align*}
    \min_{1\le t\le T}\|\mathcal G_\eta(p_t)\|^2
    &\le
    \frac{1}{T}
    \sum_{t=1}^{T}\|\mathcal G_\eta(p_t)\|^2,
\end{align*}
gives
\begin{align*}
    \min_{1\le t\le T}\|\mathcal G_\eta(p_t)\|^2
    &\le
    \frac{\Delta_T}{\alpha_{\eta,\epsilon}T}.
\end{align*}
This proves the stationarity bound.

The normal-component estimate follows directly from Lemma~\ref{lem:one-step-descent-convex-nonconvex}:
\begin{align*}
    \|n_{t-1}\|
    &\le
    (1-\epsilon)\|n_t\|.
\end{align*}
Iterating from $T$ down to $0$ gives
\begin{align*}
    \|n_0\|
    &\le
    (1-\epsilon)^T\|n_T\|.
\end{align*}
This completes the proof.
\end{proof}

\section{Convergence Analysis for the Riemann Manifold Setting}\label{app:riemann}

\subsection{Proof of Lemma \ref{lem:stein-projection-c2}}
\begin{proof}[Proof of Lemma \ref{lem:stein-projection-c2}]
Fix $x\in\cU$ and set
\begin{align*}
    p:=\pi_\cM(x),\qquad n:=x-p.
\end{align*}
By the Tweedie/Miyasawa identity,
\begin{align*}
    \widehat\pi_\sigma(x)=\bE[X\mid Y=x].
\end{align*}
Therefore the conditional law of $X$ given $Y=x$ has density on $\cM$ proportional to
\begin{align*}
    q\mapsto
    \exp\left(-\frac{\|x-q\|^2}{2\sigma^2}\right)\vartheta(q)
\end{align*}
with respect to $d\vol_{\cM}(q)$. Hence
\begin{align*}
    \widehat\pi_\sigma(x)-p
    =
    \frac{
    \int_{\cM}(q-p)
    \exp\left(-\frac{\|x-q\|^2-\|n\|^2}{2\sigma^2}\right)
    \vartheta(q)\,d\vol_{\cM}(q)
    }{
    \int_{\cM}
    \exp\left(-\frac{\|x-q\|^2-\|n\|^2}{2\sigma^2}\right)
    \vartheta(q)\,d\vol_{\cM}(q)
    }.
\end{align*}

Let $q=\Psi_p(\alpha)$ be the local chart centered at $p$. Define
\begin{align*}
    \Phi_{p,n}(\alpha)
    :=
    \frac12\left(\|x-\Psi_p(\alpha)\|^2-\|n\|^2\right).
\end{align*}
Writing $n=V_p^\top\beta$, one has
\begin{align*}
    \Phi_{p,n}(\alpha)
    =
    \frac12\|\alpha\|^2
    -
    \langle \beta,\phi_p(\alpha)\rangle
    +
    \frac12\|\phi_p(\alpha)\|^2 .
\end{align*}
Since $x\in\cU$, one has
\begin{align*}
|n|
=
\dist(x,\cM)
\le
\frac{1}{6\kappa}.
\end{align*}
Moreover, since $\|D^2\phi_p\|\le\kappa$, Taylor's theorem gives, uniformly over $p$,
\begin{align*}
\|\phi_p(\alpha)\|
\le
\frac{\kappa}{2}\|\alpha\|^2,
\qquad
\|\alpha\|\le\rho_0.
\end{align*}
It follows that
\begin{align*}
\Phi_{p,n}(\alpha)
\ge
\frac12\|\alpha\|^2
-
\|\beta\|\|\phi_p(\alpha)\|\ge
\left(
\frac12-\frac{\kappa|n|}{2}
\right)|\alpha|^2\
\ge
\frac{5}{12}|\alpha|^2,
\end{align*}
where we used $|\beta|=|n|$. On the other hand,
\begin{align*}
\Phi_{p,n}(\alpha)
\le
\frac12\|\alpha\|^2
+
\|\beta\|\|\phi_p(\alpha)\|
+
\frac12\|\phi_p(\alpha)\|^2\
\le
\left(
\frac12
+
\frac{\kappa\|n\|}{2}
+
\frac{\kappa^2}{8}\|\alpha\|^2
\right)\|\alpha\|^2\
\le
\left(
\frac{7}{12}
+
\frac{\kappa^2\rho_0^2}{8}
\right)|\alpha|^2.
\end{align*}
Therefore, setting
\begin{align*}
c:=\frac{5}{12},
\qquad
C_\Phi
:=
\frac{7}{12}
+
\frac{\kappa^2\rho_0^2}{8},
\end{align*}
one has
\begin{align*}
c|\alpha|^2
\le
\Phi_{p,n}(\alpha)
\le
C_\Phi|\alpha|^2,
\qquad
|\alpha|\le\rho_0,
\end{align*}
uniformly over $x\in\cU$.

Let $\rho_0$ be the uniform chart radius given by Assumption~\ref{assump:uniform-graph}. By compactness and the tubular-neighborhood assumption, there exists $\eta_0>0$ such that
\begin{align*}
\|x-q\|^2-\|n\|^2\ge 2\eta_0
\end{align*}
for all $x\in\cU$ and all
\begin{align*}
q\in\cM\setminus\Psi_p(B_{\rho_0}(0)).
\end{align*}
Here, $\eta_0$ depends only on the uniform chart radius $\rho_0$ and the fixed tubular-neighborhood bound $\kappa\|n\|\le\frac16.$
Thus the contribution outside the chart is exponentially small uniformly over $x\in\cU$.

Let
\begin{align*}
    a_p(\alpha)
    :=
    \vartheta(\Psi_p(\alpha))
    \sqrt{\det\left(I_d+D\phi_p(\alpha)^\top D\phi_p(\alpha)\right)} .
\end{align*}
By Assumption~\ref{assump:data} and Assumption~\ref{assump:uniform-graph}, there exist constants $0<a_-\le a_+<\infty$ such that
\begin{align*}
    a_-\le a_p(\alpha)\le a_+,
    \qquad
    \|\alpha\|\le\rho_0,
\end{align*}
uniformly over $p\in\cM$.

Set
\begin{align*}
\alpha_\sigma:=\sigma|\log\sigma|.
\end{align*}
Since
\begin{align*}
\lim_{\sigma\downarrow0}\alpha_\sigma=0,
\end{align*}
then there exists a $\bar\sigma$ such that
\begin{align*}
\alpha_\sigma\le\rho_0,
\qquad
\forall 0<\sigma\le\bar\sigma.
\end{align*}
Thus $B_{\alpha_\sigma}(0)\subseteq B_{\rho_0}(0)$, so the near region is contained in the domain of the fixed local chart. Split the posterior numerator and denominator into the near region $B_{\alpha_\sigma}(0)$ and its complement.

First, the local denominator is bounded below by the near-region contribution:
\begin{align*}
    Z_\sigma
    &\ge
    \int_{B_{\alpha_\sigma}(0)}
    e^{-\Phi_{p,n}(\alpha)/\sigma^2}a_p(\alpha)\,d\alpha\\
    &\ge
    a_-
    \int_{B_\sigma(0)}
    e^{-C_\Phi\|\alpha\|^2/\sigma^2}\,d\alpha\\
    &\ge
    c_Z\sigma^d
\end{align*}
where
\begin{align*}
c_Z
:=
a_-\int_{B_1(0)}
e^{-C_\Phi\|u\|^2}du
>
0.
\end{align*}

Next, on the near region $B_{\alpha_\sigma}(0)$, the chart displacement satisfies
\begin{align*}
    \|\Psi_p(\alpha)-p\|
    &\le
    \|\alpha\|+\|\phi_p(\alpha)\|\\
    &\le
    C_\Psi\|\alpha\|
    \le
    C_\Psi\alpha_\sigma .
\end{align*}
Therefore,
\begin{align*}
    &\left\|
    \int_{B_{\alpha_\sigma}(0)}
    \left(\Psi_p(\alpha)-p\right)
    e^{-\Phi_{p,n}(\alpha)/\sigma^2}a_p(\alpha)\,d\alpha
    \right\|\\
    &\qquad\le
    C_\Psi\alpha_\sigma
    \int_{B_{\alpha_\sigma}(0)}
    e^{-\Phi_{p,n}(\alpha)/\sigma^2}a_p(\alpha)\,d\alpha .
\end{align*}

Now consider the local tail $B_{\rho_0}(0)\setminus B_{\alpha_\sigma}(0)$. Since $\Phi_{p,n}(\alpha)\ge c\|\alpha\|^2$,
\begin{align*}
    \int_{B_{\rho_0}(0)\setminus B_{\alpha_\sigma}(0)}
    e^{-\Phi_{p,n}(\alpha)/\sigma^2}a_p(\alpha)\,d\alpha
    &\le
    a_+
    \int_{\|\alpha\|\ge\alpha_\sigma}
    e^{-c\|\alpha\|^2/\sigma^2}\,d\alpha\\
    &\le
    C\sigma^d e^{-c'|\log\sigma|^2}.
\end{align*}
Since $\cM$ is compact, $\|\Psi_p(\alpha)-p\|\le\diam(\cM)$, and hence the corresponding numerator tail is bounded by
\begin{align*}
    C\sigma^d e^{-c'|\log\sigma|^2}.
\end{align*}

The contribution outside the chart $\Psi_p(B_{\rho_0}(0))$ is also exponentially small:
\begin{align*}
    \int_{\cM\setminus\Psi_p(B_{\rho_0}(0))}
    \exp\left(-\frac{\|x-q\|^2-\|n\|^2}{2\sigma^2}\right)
    \vartheta(q)\,d\vol_{\cM}(q)
    \le
    e^{-\eta_0/\sigma^2}.
\end{align*}
The corresponding contribution to the centered numerator satisfies
\begin{align*}
&\left|
\int_{\cM\setminus\Psi_p(B_{\rho_0}(0))}
(q-p)
\exp\left(-\frac{|x-q|^2-|n|^2}{2\sigma^2}\right)
\vartheta(q)d\vol_{\cM}(q)
\right|\\
&\le
\diam(\cM)
\int_{\cM\setminus\Psi_p(B_{\rho_0}(0))}
\exp\left(-\frac{|x-q|^2-|n|^2}{2\sigma^2}\right)
\vartheta(q)d\vol_{\cM}(q)\le
\diam(\cM)e^{-\eta_0/\sigma^2}.
\end{align*}

Combining these estimates gives
\begin{align*}
|\widehat\pi_\sigma(x)-p|
&\le
C_\Psi\alpha_\sigma
+
\frac{
C\sigma^d e^{-c'|\log\sigma|^2}
+
C e^{-\eta_0/\sigma^2}
}{
c_Z\sigma^d
}\
&\le
C_\Psi\sigma|\log\sigma|
+
C e^{-c'|\log\sigma|^2}
+
C\sigma^{-d}e^{-\eta_0/\sigma^2},
\end{align*}
where the constant $C$ absorbs, in particular, the factor $\diam(\cM)$ appearing in the outside-chart numerator estimate. Since
\begin{align*}
\frac{
e^{-c'|\log\sigma|^2}
}{
\sigma|\log\sigma|
}
\longrightarrow 0,
\qquad
\frac{
\sigma^{-d}e^{-\eta_0/\sigma^2}
}{
\sigma|\log\sigma|
}
\longrightarrow 0
\end{align*}
as $\sigma\downarrow0$, after decreasing $\bar\sigma$ if necessary, assume that
\begin{align*}
C e^{-c'|\log\sigma|^2}
+
C\sigma^{-d}e^{-\eta_0/\sigma^2}
\le
\sigma|\log\sigma|,
\qquad
0<\sigma\le\bar\sigma.
\end{align*}
Therefore, for all $0<\sigma\le\bar\sigma$ and all $x\in\cU$,
\begin{align*}
|\widehat\pi_\sigma(x)-p|
\le
(C_\Psi+1)\sigma|\log\sigma|
\le
C\sigma|\log\sigma|.
\end{align*}
Since $p=\pi_\cM(x)$, this proves
\begin{align*}
\big|\widehat\pi_\sigma(x)-\pi_\cM(x)\big|
\le
C\sigma|\log\sigma|.
\end{align*}

\end{proof}

\subsection{Proof of Lemma \ref{lemma:Lyapunov-descent-riemann}}
We first state two auxiliary lemmas that is necessary for the proof (Lemma \ref{lemma:projection-decomposition} and \ref{lem:riemann-tube-invariance})
\begin{lemma}\label{lemma:projection-decomposition}
Let $x\in\cU, x+\Delta \in \cU$, where $\cU$ is defined as in \eqref{eq:cU}, and define $p:=\pi_\cM(x)$. Suppose that $q:=\pi_\cM(x+\Delta)$ lies in the local graph chart around $p$, so that
$q=\Psi_p(\alpha)=p+U_p^\top\alpha+V_p^\top\phi_p(\alpha)$ for some $\alpha\in\cA_p$. Then
\begin{align*}
    \alpha = M U_p\Delta,
\end{align*}
where $M\in\bR^{d\times d}$ satisfies
\begin{align*}
    \|M-I\|\le \frac13 .
\end{align*}
\end{lemma}

\begin{proof}
Let $z:=x+\Delta$ and $q:=\pi_\cM(z)$. Since $p=\pi_\cM(x)$, the projection residual $x-p$ is normal to $\cM$ at $p$. Hence there exists $\beta\in\bR^{n-d}$ such that $x-p=V_p^\top\beta$. Define
\begin{align*}
    u:=U_p\Delta,
    \qquad
    v:=V_p\Delta .
\end{align*}
Using the local representation $q=\Psi_p(\alpha)$, we can write
\begin{align*}
    z-q&= x+\Delta-\left(p+U_p^\top\alpha+V_p^\top\phi_p(\alpha)\right) \\
    &=    U_p^\top(u-\alpha)+V_p^\top\left(\beta+v-\phi_p(\alpha)\right).
\end{align*}
Since $q=\pi_\cM(z)$ is the nearest point on $\cM$ to $z$, the residual $z-q$ is orthogonal to $T_q\cM$. Equivalently,
\begin{align*}
    D\Psi_p(\alpha)^\top (z-\Psi_p(\alpha))=0 .
\end{align*}
Because $D\Psi_p(\alpha)=U_p^\top+V_p^\top D\phi_p(\alpha)$, the above condition gives
\begin{align*}
    0
    &=
    (u-\alpha)
    +
    D\phi_p(\alpha)^\top
    \left(\beta+v-\phi_p(\alpha)\right).
\end{align*}
Let $s:=\beta+v-\phi_p(\alpha)$. Since $D\phi_p(0)=0$, the curvature bound implies that there exists a matrix $B(\alpha,s)\in\bR^{d\times d}$ such that
\begin{align*}
    D\phi_p(\alpha)^\top s = B(\alpha,s)\alpha,
    \qquad
    \|B(\alpha,s)\|\le \kappa \|s\|.
\end{align*}
Indeed, for any $w\in\bR^d$,
\begin{align*}
    w^\top D\phi_p(\alpha)^\top s
    =
    s^\top D\phi_p(\alpha)w 
    =
    \int_0^1 s^\top D^2\phi_p(t\alpha)[\alpha,w]\,dt,
\end{align*}
which gives $\|B(\alpha,s)\|\le \kappa\|s\|$. For the sake of simplicity we abbrievate $B(\alpha,s)$ as $B$ in the following.

Moreover, since $s$ is the normal-coordinate component of the projection residual $z-q$, from Assumption \ref{assump:uniform-graph} we have
\begin{align*}
    \|s\|\le\|z-q\|=\dist(z,\cM)\le \frac{1}{6\kappa}\le\frac{1}{4\kappa}.
\end{align*}
Therefore $\|B\|\le 1/4$. The first-order condition becomes
\begin{align*}
    0=u-\alpha+B\alpha,
\end{align*}
or equivalently,
\begin{align*}
    (I-B)\alpha=u.
\end{align*}
Since $\|B\|\le 1/4<1$, the matrix $I-B$ is invertible. Thus
\begin{align*}
    \alpha
    =
    (I-B)^{-1}u
    =
    M U_p\Delta,
\end{align*}
where $M:=(I-B)^{-1}$. Finally,
\begin{align*}
    \|M-I\|
    =
    \|(I-B)^{-1}-I\| =
    \|(I-B)^{-1}B\| 
    \le
    \frac{\|B\|}{1-\|B\|} \le
    \frac{1/4}{1-1/4}=
    \frac13 .
\end{align*}
This proves the claim.
\end{proof}

We first establish a one-step estimate for the normal distance to the
manifold.

\begin{lemma}[One-step normal-distance recursion]
\label{lem:riemann-normal-recursion}
Suppose Assumption~\ref{assump:uniform-graph} holds. Let
\begin{align*}
    p:=\pi_\cM(x),
    \qquad
    n:=x-p,
\end{align*}
and suppose that, for some $a\in B_{\r}(0)$,
\begin{align*}
    y
    =
    \pi_\cM(z)
    =
    \Psi_p(a).
\end{align*}
Let $\widehat y$ satisfy
\begin{align*}
    \|\widehat y-y\|\le E,
\end{align*}
and consider the relaxed update
\begin{align*}
    x^+
    =
    (1-\epsilon)x+\epsilon\widehat y.
\end{align*}
If
\begin{align*}
    \epsilon\bigl(\|a\|+E\bigr)\le \r,
\end{align*}
then
\begin{align*}
    \dist(x^+,\cM)
    \le
    (1-\epsilon)\|n\|
    +
    \kappa\epsilon\|a\|^2
    +
    \frac{7}{6}\epsilon E.
\end{align*}
\end{lemma}

\begin{proof}
For simplicity, write
\begin{align*}
    U:=U_p,
    \qquad
    V:=V_p,
    \qquad
    \phi:=\phi_p.
\end{align*}
Since $n=x-p\in N_p\cM$, there exists
$\beta\in\bR^{n-d}$ such that
\begin{align*}
    n=V^\top\beta,
    \qquad
    \|\beta\|=\|n\|.
\end{align*}
Decompose the projection error according to
$\bR^n=T_p\cM\oplus N_p\cM$:
\begin{align*}
    \widehat y-y
    =
    U^\top\xi+V^\top\zeta.
\end{align*}
Since $U$ and $V$ are row-orthonormal,
\begin{align*}
    \|\xi\|\le E,
    \qquad
    \|\zeta\|\le E.
\end{align*}
Using
\begin{align*}
    y
    =
    p+U^\top a+V^\top\phi(a),
\end{align*}
we obtain
\begin{align*}
    x^+
    =
    p
    +
    U^\top\epsilon(a+\xi)
    +
    V^\top\left(
        (1-\epsilon)\beta
        +
        \epsilon\phi(a)
        +
        \epsilon\zeta
    \right).
\end{align*}
Moreover,
\begin{align*}
    \epsilon\|a+\xi\|
    \le
    \epsilon\bigl(\|a\|+E\bigr)
    \le
    \r.
\end{align*}
Therefore,
\begin{align*}
    \Psi_p\bigl(\epsilon(a+\xi)\bigr)
    \in
    \cM
\end{align*}
is a valid candidate point. It follows that
\begin{align*}
    \dist(x^+,\cM)
    &\le
    \left\|
    x^+
    -
    \Psi_p\bigl(\epsilon(a+\xi)\bigr)
    \right\|\\
    &\le
    (1-\epsilon)\|n\|
    +
    \left\|
    \epsilon\phi(a)
    -
    \phi\bigl(\epsilon(a+\xi)\bigr)
    \right\|
    +
    \epsilon E.
\end{align*}
We split the middle term as
\begin{align*}
    \epsilon\phi(a)-\phi\bigl(\epsilon(a+\xi)\bigr)
    =
    \epsilon\phi(a)-\phi(\epsilon a)
    +
    \phi(\epsilon a)-\phi\bigl(\epsilon(a+\xi)\bigr).
\end{align*}
By the curvature bound, $\phi(0)=0$, and $D\phi(0)=0$,
\begin{align*}
    \|\epsilon\phi(a)-\phi(\epsilon a)\|
    &\le
    \epsilon\|\phi(a)\|+\|\phi(\epsilon a)\|\\
    &\le
    \frac{\kappa}{2}
    \left(
        \epsilon+\epsilon^2
    \right)
    \|a\|^2\\
    &\le
    \kappa\epsilon\|a\|^2.
\end{align*}
Furthermore, the segment between $\epsilon a$ and
$\epsilon(a+\xi)$ lies in $B_{\r}(0)$. Since
$\r\le 1/(6\kappa)$, Assumption~\ref{assump:uniform-graph}
gives $\|D\phi\|\le 1/6$ along this segment. Hence,
\begin{align*}
    \|\phi(\epsilon a)-\phi\bigl(\epsilon(a+\xi)\bigr)\|
    \le
    \frac{\epsilon}{6}\|\xi\|
    \le
    \frac{\epsilon}{6}E.
\end{align*}
Combining these estimates gives
\begin{align*}
    \dist(x^+,\cM)
    \le
    (1-\epsilon)\|n\|
    +
    \kappa\epsilon\|a\|^2
    +
    \frac{7}{6}\epsilon E.
\end{align*}
\end{proof}
Combining Lemmas~\ref{lemma:projection-decomposition}
and~\ref{lem:riemann-normal-recursion}, we obtain the following
invariance property of the update.

\begin{lemma}[Invariance of the local tubular neighborhood]
\label{lem:riemann-tube-invariance}
Suppose Assumption \ref{assump:uniform-graph} and \ref{assump:smoothness} hold. Consider the update \eqref{eq:projected-gradient-guidance-relaxed}. 
Let $z_t:=x_t-\eta\nabla f(x_t), p_t:=\pi_\cM(x_t),$ and suppose that
   $ \|\widehat\pi_t(z_t)-\pi_\cM(z_t)\|\le E_t .$
   Let $ $
then under the condition that
\begin{align*}
 \textstyle   \dist(x_{t_0},\cM)
    \le
    \frac{\r}{2},
    \qquad
    \eta 
    \le
    \frac{\r}{4G},
    \qquad
    \epsilon
    \le
    \frac{\r}{4\sup_{1\le t\le t_0} E_t}.
\end{align*}
the backward iterates generated by
\eqref{eq:projected-gradient-guidance-relaxed} satisfy
\begin{align*}
   \textstyle \dist(x_t,\cM)
    \le
    \frac{\r}{2},
    \qquad
    z_t
    \in
    \cU,
    \qquad
    0\le t\le t_0.
\end{align*}
Moreover, for every $1\le t\le t_0$, there exist unique
$a_t,\gamma_t\in B_{\r}(0)$ such that
    $\pi_\cM(z_t)
    =
    \Psi_{p_t}(a_t),
    \pi_\cM(x_{t-1})
    =
    \Psi_{p_t}(\gamma_t).$
\end{lemma}

\begin{proof}
Let
\begin{align*}
    p_t:=\pi_\cM(x_t),
    \qquad
    n_t:=x_t-p_t,
    \qquad
    z_t:=x_t-\eta\nabla f(x_t).
\end{align*}
We argue by backward induction. Suppose that
$\|n_t\|\le \r/2$. Then
\begin{align*}
    \dist(z_t,\cM)
    \le
    \|n_t\|+\eta\|\nabla f(x_t)\|
    \le
    \frac{3\r}{4},
\end{align*}
and hence $x_t,z_t\in\cU$. By
Lemma~\ref{lemma:projection-decomposition}, there exists a unique
$a_t\in B_{\rho_0}(0)$ such that
\begin{align*}
    \pi_\cM(z_t)=\Psi_{p_t}(a_t),
\end{align*}
and
\begin{align*}
    \|a_t\|
    \le
    \frac{4}{3}\eta G
    \le
    \frac{\r}{3}
    <
    \r.
\end{align*}
Therefore, Lemma~\ref{lem:riemann-normal-recursion} gives
\begin{align*}
    \|n_{t-1}\|
    &=
    \dist(x_{t-1},\cM)\\
    &\le
    (1-\epsilon)\|n_t\|
    +
    \kappa\epsilon\|a_t\|^2
    +
    \frac{7}{6}\epsilon E_t\\
    &\le
    (1-\epsilon)\frac{\r}{2}
    +
    \epsilon\left(
        \frac{\kappa\r^2}{9}
        +
        \frac{7\r}{24}
    \right)\\
    &\le
    \frac{\r}{2},
\end{align*}
where we used $\r\le 1/(6\kappa)$. This proves the
inductive claim.

It remains to verify that the second local coordinate used in the
projection decomposition also stays inside the chart. Define
$\gamma_t\in B_{\rho_0}(0)$ by
\begin{align*}
    p_{t-1}
    =
    \pi_\cM(x_{t-1})
    =
    \Psi_{p_t}(\gamma_t).
\end{align*}
Lemma~\ref{lemma:projection-decomposition} gives
\begin{align*}
    \|\gamma_t\|
    \le
    \frac{4}{3}\epsilon
    \left(
        \|a_t\|+E_t
    \right)
    \le
    \frac{4}{9}\epsilon \r
        +\frac{\r}{3}
    \le
    \r.
\end{align*}
Thus, for every $1\le t\le t_0$,
\begin{align*}
    a_t,\gamma_t\in B_{\r}(0)
    \subseteq B_{\rho_0}(0),
\end{align*}
so all local graph representations invoked in the proof are
well-defined.
\end{proof}
\begin{proof}[Proof of Lemma \ref{lemma:Lyapunov-descent-riemann}]
The statement of $\dist(x_{t-1},\cM)\le \frac{r}{2}$ is immediate from Lemma \ref{lem:riemann-tube-invariance}. We now prove the rest.

Denote $g_t:=U_{p_t}\nabla f(p_t)$.
For simplicity, write $p=p_t$, $U=U_p$, $V=V_p$, $\phi=\phi_p$, $x=x_t$, $n=n_t$, and $g=g_t$. Since $p=\pi_\cM(x)$, we have $n=x-p\in N_p\cM$. Thus there exists $\beta\in\bR^{n-d}$ such that $n=V^\top\beta$ and $\|\beta\|=\|n\|$. Define
\begin{align*}
    h:=U\nabla f(x).
\end{align*}
By $L$-smoothness,
\begin{align*}
    \|h-g\|
    =
    \|U(\nabla f(x)-\nabla f(p))\|
    \le
    L\|x-p\|
    =
    L\|n\|.
\end{align*}

Let
\begin{align*}
    y:=\pi_\cM(x-\eta\nabla f(x)),
    \qquad
    \widehat y:=\widehat\pi_t(x-\eta\nabla f(x)).
\end{align*}
By Lemma \ref{lem:riemann-tube-invariance}, $y$ lies in the local chart and hence by Lemma~\ref{lemma:projection-decomposition}, 
\begin{align*}
    y
    =
    p+U^\top a+V^\top\phi(a),
    \qquad
    a=-\eta M_t h,
\end{align*}
where $\|M_t-I\|\le 1/3$. Let
\begin{align*}
    \widehat y-y
    =
    U^\top r_t+V^\top s_t
\end{align*}
be the decomposition of the inexact projection error in the splitting
$\bR^n=T_p\cM\oplus N_p\cM$. Since $U$ and $V$ are row-orthonormal,
\begin{align*}
    \|r_t\|\le E_t,
    \qquad
    \|s_t\|\le E_t.
\end{align*}
Therefore
\begin{align*}
    \widehat y
    =
    p
    +
    U^\top(a+r_t)
    +
    V^\top(\phi(a)+s_t).
\end{align*}
The update can be written as
\begin{align*}
    x_{t-1}
    &=
    (1-\epsilon)x+\epsilon\widehat y \\
    &=
    p
    +
    U^\top\epsilon(a+r_t)
    +
    V^\top\left((1-\epsilon)\beta+\epsilon\phi(a)+\epsilon s_t\right).
\end{align*}

Applying Lemma~\ref{lemma:projection-decomposition} again to $p_{t-1}=\pi_\cM(x_{t-1})$ (from Lemma \ref{lem:riemann-tube-invariance} we know that it is in the local chart around $p$), we obtain
\begin{align*}
    p_{t-1}
    =
    p+U^\top\gamma+V^\top\phi(\gamma),
    \qquad
    \gamma
    =
    \epsilon M_t^+(a+r_t),
\end{align*}
where $\|M_t^+-I\|\le 1/3$. Hence, with $a=-\eta M_t h$,
\begin{align*}
    \gamma
    =
    -\epsilon\eta M_t^+M_t h
    +
    \epsilon M_t^+r_t .
\end{align*}
Moreover,
\begin{align*}
    \|M_t^+M_t\|\le \frac{16}{9},
    \qquad
    \|M_t^+M_t-I\|\le \frac79,
    \qquad
    \|M_t^+\|\le \frac43.
\end{align*}

We first prove the descent estimate for $f(p_{t-1})$. By $L$-smoothness,
\begin{align*}
    f(p_{t-1})-f(p)
    &\le
    \nabla f(p)^\top(p_{t-1}-p)
    +
    \frac{L}{2}\|p_{t-1}-p\|^2 \\
    &=
    g^\top\gamma
    +
    (V\nabla f(p))^\top\phi(\gamma)
    +
    \frac{L}{2}\|p_{t-1}-p\|^2 .
\end{align*}
The first term satisfies
\begin{align*}
    g^\top\gamma
    &=
    -\epsilon\eta g^\top M_t^+M_t h
    +
    \epsilon g^\top M_t^+r_t .
\end{align*}
The first term on the right hand side can be bounded by
\begin{align*}
    &\quad g^\top M_t^+M_t h = g^\top M_t^+M_t g + g^\top M_t^+M_t (h-g)\\
    &\ge \|g\|^2 - g^\top (I- M_t^+M_t) g - \frac{16}{9}\|g\|\|h-g\|\\
    &\ge \frac{2}{9}\|g\|^2 - \frac{16}{9} L \|g\|\|n\|\ge \frac{2}{9}\|g\|^2 - \frac{1}{18}\|g^2\| - \frac{128}{9}L^2\|n\|^2\\
    &=\frac16\|g\|^2
    -
    \frac{128}{9}L^2\|n\|^2 .
\end{align*}
Also, the second term can be bounded by
\begin{align*}
    \epsilon g^\top M_t^+r_t
    \le
    \frac43\epsilon\|g\|E_t
    \le
    \frac{1}{12}\epsilon\eta\|g\|^2
    +
    \frac{16}{3}\frac{\epsilon}{\eta}E_t^2 .
\end{align*}
Therefore
\begin{align*}
    g^\top\gamma
    &\le
    -
    \frac{1}{12}\epsilon\eta\|g\|^2
    +
    \frac{128}{9}\epsilon\eta L^2\|n\|^2
    +
    \frac{16}{3}\frac{\epsilon}{\eta}E_t^2 .
\end{align*}

Next, since
\begin{align*}
    \|\gamma\|
    &\le
    \frac{16}{9}\epsilon\eta\|h\|
    +
    \frac43\epsilon E_t,
\end{align*}
we have
\begin{align*}
    \|\gamma\|^2
    \le 8(\epsilon\eta)^2\|h\|^2
    +
    4\epsilon^2E_t^2 \le
    16(\epsilon\eta)^2\left(\|g\|^2+L^2\|n\|^2\right)
    +
    4\epsilon^2E_t^2. \qquad (\|h\|^2 \le 2(\|g\|^2+L^2\|n\|^2))
\end{align*}
From Assumption \ref{assump:uniform-graph},  $\|D^2\phi_p(\alpha)[\gamma,\gamma]\|\le \kappa\|\gamma\|^2$ for all $\|\alpha\|\le \rho_0$, and further $\phi_p(0) = 0, D\phi_p(0) = 0,$ hence using Lagrange remainder we get $\|\phi(\gamma)\|\le \frac{\kappa}{2}\|\gamma\|^2$. Thus, we get
\begin{align*}
    |(V\nabla f(p))^\top\phi(\gamma)|
    +
    \frac{L}{2}\|p_{t-1}-p\|^2
    &\le
    (G\kappa+L)\|\gamma\|^2 \\
    &\le
    16(G\kappa+L)(\epsilon\eta)^2
    \left(\|g\|^2+L^2\|n\|^2\right) \\
    &\quad
    +
    4(G\kappa+L)\epsilon^2E_t^2.
\end{align*}
Combining the previous estimates yields
\begin{align*}
    f(p_{t-1})-f(p_t)
    &\le
    -
    \left(
    \frac{1}{12}
    -
    16(G\kappa+L)\epsilon\eta
    \right)
    \epsilon\eta\|g_t\|^2 \\
    &\quad
    +
    \left(
    \frac{128}{9}
    +
    16(G\kappa+L)\epsilon\eta
    \right)
    \epsilon\eta L^2\|n_t\|^2 \\
    &\quad
    +
    \epsilon
    \left(
    \frac{16}{3\eta}
    +
    4(G\kappa+L)\epsilon
    \right)
    E_t^2 .
\end{align*}

We now bound the normal distance. By
Lemma~\ref{lem:riemann-normal-recursion},
\begin{align*}
    \|n_{t-1}\|
    \le
    (1-\epsilon)\|n_t\|
    +
    \kappa\epsilon\|a\|^2
    +
    \frac{7}{6}\epsilon E_t.
\end{align*}
Using $((1-\epsilon)a+b)^2\le (1-\epsilon)a^2+b^2/\epsilon$ for $a,b\ge 0$, we obtain
\begin{align*}
    \|n_{t-1}\|^2
    &\le
    (1-\epsilon)\|n\|^2
    +
    \epsilon
    \left(
    \kappa\|a\|^2+\frac76E_t
    \right)^2 \\
    &\le
    (1-\epsilon)\|n\|^2
    +
    2\epsilon\kappa^2\|a\|^4
    +
    \frac{49}{18}\epsilon E_t^2 .
\end{align*}
Since $\|a\|\le 1/(6\kappa)$, we have $2\kappa^2\|a\|^4\le \|a\|^2$. Also, $\|a\|\le \frac43\eta\|h\|$ and $\|h\|\le \|g\|+L\|n\|$. Thus
\begin{align*}
    \|a\|^2
    &\le
    \frac{16}{9}\eta^2\|h\|^2 \\
    &\le
    \frac{32}{9}\eta^2\left(\|g\|^2+L^2\|n\|^2\right).
\end{align*}
Consequently,
\begin{align*}
    \|n_{t-1}\|^2
    &\le
    \left(1-\epsilon+8\epsilon\eta^2L^2\right)\|n_t\|^2
    +
    8\epsilon\eta^2\|g_t\|^2
    +
    3\epsilon E_t^2 .
\end{align*}
Equivalently,
\begin{align*}
    \|n_{t-1}\|^2-\|n_t\|^2
    &\le
    -
    \epsilon\left(1-8\eta^2L^2\right)\|n_t\|^2
    +
    8\epsilon\eta^2\|g_t\|^2
    +
    3\epsilon E_t^2 .
\end{align*}

Combining the function decrease and the normal-distance estimate gives
\begin{align*}
    V(x_{t-1})-V(x_t)
    &=
    f(p_{t-1})-f(p_t)
    +
    \frac{1}{160\eta}
    \left(
    \|n_{t-1}\|^2-\|n_t\|^2
    \right) \\
    &\le
    -
    \epsilon
    \left[
    \left(
    \frac{1}{12}
    -
    16(G\kappa+L)\epsilon\eta
    \right)\eta
    -
    \frac{\eta}{20}
    \right]
    \|g_t\|^2 \\
    &\quad
    -
    \epsilon
    \left[
    \frac{1}{160\eta}
    \left(1-8\eta^2L^2\right)
    -
    \left(
    \frac{128}{9}
    +
    16(G\kappa+L)\epsilon\eta
    \right)
    \eta L^2
    \right]
    \|n_t\|^2 \\
    &\quad
    +
    \epsilon
    \left[
    \frac{16}{3\eta}
    +
    4(G\kappa+L)\epsilon
    +
    \frac{3}{160\eta}
    \right]
    E_t^2 .
\end{align*}

It remains to simplify the constants. Since $\epsilon\eta\le 1/(512(G\kappa+L))$, we have
\begin{align*}
    16(G\kappa+L)\epsilon\eta
    \le
    \frac{1}{32}.
\end{align*}
Hence
\begin{align*}
    \left(
    \frac{1}{12}
    -
    16(G\kappa+L)\epsilon\eta
    \right)\eta
    -
    \frac{\eta}{20}
    &\ge
    \left(
    \frac{1}{12}
    -
    \frac{1}{32}
    -
    \frac{1}{20}
    \right)\eta \\
    &=
    \frac{\eta}{480}.
\end{align*}
Similarly, since $\eta\le 1/(50L)$,
\begin{align*}
    1-8\eta^2L^2
    \ge
    \frac{2492}{2500}.
\end{align*}
Together with $16(G\kappa+L)\epsilon\eta\le 1/32$, this implies
\begin{align*}
    &\frac{1}{160\eta}
    \left(1-8\eta^2L^2\right)
    -
    \left(
    \frac{128}{9}
    +
    16(G\kappa+L)\epsilon\eta
    \right)
    \eta L^2 \\
    &\qquad\ge
    \frac{1}{\eta}
    \left[
    \frac{2492}{400000}
    -
    \left(
    \frac{128}{9}
    +
    \frac{1}{32}
    \right)
    \frac{1}{50^2}
    \right] \\
    &\qquad\ge
    \frac{1}{2000\eta}.
\end{align*}
Finally,
\begin{align*}
    4(G\kappa+L)\epsilon
    \le
    \frac{1}{128\eta},
\end{align*}
and therefore
\begin{align*}
    \frac{16}{3\eta}
    +
    4(G\kappa+L)\epsilon
    +
    \frac{3}{160\eta}
    &\le
    \frac{16}{3\eta}
    +
    \frac{1}{128\eta}
    +
    \frac{3}{160\eta} \\
    &\le
    \frac{6}{\eta}.
\end{align*}
Substituting these bounds gives
\begin{align*}
    V(x_{t-1})-V(x_t)
    \le
    -
    \frac{\epsilon\eta}{480}\|g_t\|^2
    -
    \frac{\epsilon}{2000\eta}\|n_t\|^2
    +
    \frac{6\epsilon}{\eta}E_t^2.
\end{align*}
Note that $g_t := U_{p_t}\nabla f(p_t)$, hence $$\|g_t\|^2 = \nabla f(p_t)^\top U_{p_t}^\top U_{p_t}\nabla f(p_t)= \nabla f(p_t)^\top U_{p_t}^\top U_{p_t}U_{p_t}^\top U_{p_t}\nabla f(p_t) = \|\grad_\cM f(p_t)\|^2,$$ since $U_{p_t}^\top U_{p_t}$ is the orthogonal tangent projector. And hence this proves the claim
\end{proof}

% \gz{Remaining gaps to close: 
% (a) as noted in the main text, the inequality $g^\top M_t^+M_t h \ge \frac16\|g\|^2 - \frac{128}{9}L^2\|n\|^2$ is cited from a nonexistent ``exact-projection argument''. Include the two-line derivation. 
% (b) The step $\|\gamma\|\le\frac{16}{9}\epsilon\eta\|h\| + \frac43\epsilon E_t$ then ``$\|\gamma\|^2 \le 16(\epsilon\eta)^2(\|g\|^2+L^2\|n\|^2)+4\epsilon^2E_t^2$'' additionally uses $\|h\|^2\le2\|g\|^2+2L^2\|n\|^2$; make that explicit. 
% (c) The Lyapunov coefficient $\frac{1}{160\eta}$: nothing before the proof motivates $160$; see main-text comment about parametrizing the weight. 
% (d) Note $g_t := U_{p_t}\nabla f(p_t)$ has $\|g_t\| = \|\grad_\cM f(p_t)\|$ since $U^\top U$ is the orthogonal tangent projector. One line, currently implicit in the final sentence. 
% (e) Where you use $\|\phi(\gamma)\|\le\frac\kappa2\|\gamma\|^2$ you also need $\|\gamma\|\le\rho_0$; this is exactly the ``tangent coordinates have norm at most $1/(6\kappa)$'' assumption in the lemma statement --- cross-reference it at the point of use.}

\subsection{Proof of Theorem \ref{thm:riemann-backward-convergence}}
\begin{proof}[Proof of Theorem \ref{thm:riemann-backward-convergence}]
For each $t=1,\ldots,t_0$, Lemma~\ref{lem:stein-projection-c2} gives
\begin{align*}
    \|\widehat\pi_t(z)-\pi_\cM(z)\|
    &\le
    C\sigma_t|\log(\sigma_t)|\\
    &\le
    C\sigma_{t_0}(1-\epsilon)^{t_0-t}
    \left|
        \log\left(
            \sigma_{t_0}(1-\epsilon)^{t_0-t}
        \right)
    \right|,
    \qquad
    \forall z\in\cU.
\end{align*}
Thus, applying Lemma~\ref{lemma:Lyapunov-descent-riemann} in backward form with
\begin{align*}
    E_t
    =
    C\sigma_{t_0}(1-\epsilon)^{t_0-t}
    \left|
        \log\left(
            \sigma_{t_0}(1-\epsilon)^{t_0-t}
        \right)
    \right|,
\end{align*}
we obtain
\begin{align*}
    V(x_{t-1})
    \le
    V(x_t)
    &-
    \frac{\epsilon\eta}{480}\|\grad_\cM f(p_t)\|^2
    -
    \frac{\epsilon}{2000\eta}\|n_t\|^2\\
    &+
    \frac{6\epsilon C^2}{\eta}
    \sigma_{t_0}^2(1-\epsilon)^{2(t_0-t)}
    \left|
        \log\left(
            \sigma_{t_0}(1-\epsilon)^{t_0-t}
        \right)
    \right|^2.
\end{align*}
Rearranging gives
\begin{align*}
    \frac{\epsilon\eta}{480}\|\grad_\cM f(p_t)\|^2
    +
    \frac{\epsilon}{2000\eta}\|n_t\|^2
    \le
    V(x_t)-V(x_{t-1})
    &+
    \frac{6\epsilon C^2}{\eta}
    \sigma_{t_0}^2(1-\epsilon)^{2(t_0-t)}\\
    &\quad\times
    \left|
        \log\left(
            \sigma_{t_0}(1-\epsilon)^{t_0-t}
        \right)
    \right|^2.
\end{align*}
Summing from $t=1$ to $t_0$ yields
\begin{align*}
    \sum_{t=1}^{t_0}
    \left(
        \frac{\epsilon\eta}{480}\|\grad_\cM f(p_t)\|^2
        +
        \frac{\epsilon}{2000\eta}\|n_t\|^2
    \right)
    &\le
    \sum_{t=1}^{t_0}
    \bigl(V(x_t)-V(x_{t-1})\bigr)\\
    &\quad+
    \frac{6\epsilon C^2}{\eta}
    \sum_{t=1}^{t_0}
    \sigma_{t_0}^2(1-\epsilon)^{2(t_0-t)}
    \left|
        \log\left(
            \sigma_{t_0}(1-\epsilon)^{t_0-t}
        \right)
    \right|^2.
\end{align*}
The first sum telescopes:
\begin{align*}
    \sum_{t=1}^{t_0}
    \bigl(V(x_t)-V(x_{t-1})\bigr)
    =
    V(x_{t_0})-V(x_0)
    \le
    V(x_{t_0})-\min_{x\in\cM}f(x).
\end{align*}
For the second sum, since $0<\sigma_{t_0}\le 1$, setting $j:=t_0-t$ gives
\begin{align*}
    &\sum_{t=1}^{t_0}
    \sigma_{t_0}^2(1-\epsilon)^{2(t_0-t)}
    \left|
        \log\left(
            \sigma_{t_0}(1-\epsilon)^{t_0-t}
        \right)
    \right|^2\\
    &\qquad=
    \sigma_{t_0}^2
    \sum_{j=0}^{t_0-1}
    (1-\epsilon)^{2j}
    \left(
        |\log\sigma_{t_0}|
        +
        j|\log(1-\epsilon)|
    \right)^2\\
    &\qquad\le
    \sigma_{t_0}^2
    \sum_{j=0}^{\infty}
    (1-\epsilon)^{2j}
    \left(
        |\log\sigma_{t_0}|
        +
        j|\log(1-\epsilon)|
    \right)^2.
\end{align*}
Using
\begin{align*}
    \sum_{j=0}^{\infty}r^j
    =
    \frac{1}{1-r},
    \qquad
    \sum_{j=0}^{\infty}jr^j
    =
    \frac{r}{(1-r)^2},
    \qquad
    \sum_{j=0}^{\infty}j^2r^j
    =
    \frac{r(1+r)}{(1-r)^3},
\end{align*}
with $r=(1-\epsilon)^2$, we obtain
\begin{align*}
    &\sum_{t=1}^{t_0}
    \sigma_{t_0}^2(1-\epsilon)^{2(t_0-t)}
    \left|
        \log\left(
            \sigma_{t_0}(1-\epsilon)^{t_0-t}
        \right)
    \right|^2\\
    &\qquad\le
    \sigma_{t_0}^2
    \left[
        \frac{|\log\sigma_{t_0}|^2}
        {1-(1-\epsilon)^2}
        +
        \frac{
            2(1-\epsilon)^2
            |\log\sigma_{t_0}|
            |\log(1-\epsilon)|
        }{
            \left(1-(1-\epsilon)^2\right)^2
        }\right.\\
    &\hspace{7cm}\left.
        +
        \frac{
            (1-\epsilon)^2
            \left(1+(1-\epsilon)^2\right)
            |\log(1-\epsilon)|^2
        }{
            \left(1-(1-\epsilon)^2\right)^3
        }
    \right].
\end{align*}
Since $0<\epsilon\le\frac12$, one has
\begin{align*}
    1-(1-\epsilon)^2
    =
    \epsilon(2-\epsilon)
    \ge
    \epsilon,
    \qquad
    (1-\epsilon)^2
    \le
    1,
    \qquad
    |\log(1-\epsilon)|
    \le
    2\epsilon.
\end{align*}
Therefore,
\begin{align*}
    &\sum_{t=1}^{t_0}
    \sigma_{t_0}^2(1-\epsilon)^{2(t_0-t)}
    \left|
        \log\left(
            \sigma_{t_0}(1-\epsilon)^{t_0-t}
        \right)
    \right|^2\\
    &\qquad\le
    \frac{\sigma_{t_0}^2}{\epsilon}
    \left(
        |\log\sigma_{t_0}|^2
        +
        4|\log\sigma_{t_0}|
        +
        8
    \right).
\end{align*}
Consequently,
\begin{align*}
    \sum_{t=1}^{t_0}
    \left(
        \frac{\epsilon\eta}{480}\|\grad_\cM f(p_t)\|^2
        +
        \frac{\epsilon}{2000\eta}\|n_t\|^2
    \right)
    &\le
    V(x_{t_0})-\min_{x\in\cM}f(x)\\
    &\quad+
    \frac{6C^2\sigma_{t_0}^2}{\eta}
    \left(
        |\log\sigma_{t_0}|^2
        +
        4|\log\sigma_{t_0}|
        +
        8
    \right)\\
    &=
    \Delta_{t_0}.
\end{align*}
Hence,
\begin{align*}
    \min_{1\le t\le t_0}
    \left(
        \frac{\epsilon\eta}{480}\|\grad_\cM f(p_t)\|^2
        +
        \frac{\epsilon}{2000\eta}\|n_t\|^2
    \right)
    \le
    \frac{\Delta_{t_0}}{t_0}.
\end{align*}
Let $t_\star$ be an index attaining the minimum. Since both terms are nonnegative,
\begin{align*}
    \frac{\epsilon\eta}{480}
    \|\grad_\cM f(p_{t_\star})\|^2
    \le
    \frac{\Delta_{t_0}}{t_0},
    \qquad
    \frac{\epsilon}{2000\eta}
    \|n_{t_\star}\|^2
    \le
    \frac{\Delta_{t_0}}{t_0}.
\end{align*}
This completes the proof.
\end{proof}

\section{Auxiliaries}
\begin{lemma}[Uniform radial moment bound]\label{lem:radial-moment-bound}
For every $n\ge 1$, there exists a constant $C_n>0$ such that, for all $\beta\ge 0$ and all $A\in(0,\infty]$,
\begin{align*}
    \frac{
    \int_0^A t^n e^{-t^2/2-\beta t}\,dt
    }{
    \int_0^A t^{n-1}e^{-t^2/2-\beta t}\,dt
    }
    \le C_n .
\end{align*}
\end{lemma}

\begin{proof}
Let $\nu_0$ be the probability measure on $(0,\infty)$ with density
\begin{align*}
    d\nu_0(t)
    =
    \frac{t^{n-1}e^{-t^2/2}}
    {\int_0^\infty s^{n-1}e^{-s^2/2}\,ds}\,dt .
\end{align*}
Fix $\beta\ge 0$ and $A\in(0,\infty]$, and define
\begin{align*}
    g(t):=e^{-\beta t}\mathbf 1_{\{0<t<A\}} .
\end{align*}
Then $g$ is nonnegative and nonincreasing. Moreover,
\begin{align*}
    \int g(t)\,d\nu_0(t)>0 .
\end{align*}
Let $\varphi:(0,\infty)\to\bR$ be any bounded nondecreasing function. If $S,T$ are independent random variables with law $\nu_0$, then
\begin{align*}
    \operatorname{Cov}_{\nu_0}(\varphi,g)
    &=
    \frac12
    \bE\left[
    \left(\varphi(T)-\varphi(S)\right)
    \left(g(T)-g(S)\right)
    \right]
    \le 0,
\end{align*}
because $\varphi$ is nondecreasing and $g$ is nonincreasing. Therefore,
\begin{align*}
    \frac{\int \varphi(t)g(t)\,d\nu_0(t)}
    {\int g(t)\,d\nu_0(t)}
    \le
    \int \varphi(t)\,d\nu_0(t).
\end{align*}
Taking $\varphi_M(t):=\min\{t,M\}$ and letting $M\to\infty$ by monotone convergence gives
\begin{align*}
    \frac{\int t\,g(t)\,d\nu_0(t)}
    {\int g(t)\,d\nu_0(t)}
    \le
    \int t\,d\nu_0(t).
\end{align*}
Writing this inequality in terms of the density of $\nu_0$, we obtain
\begin{align*}
    \frac{
    \int_0^A t^n e^{-t^2/2-\beta t}\,dt
    }{
    \int_0^A t^{n-1}e^{-t^2/2-\beta t}\,dt
    }
    \le
    \frac{
    \int_0^\infty t^n e^{-t^2/2}\,dt
    }{
    \int_0^\infty t^{n-1}e^{-t^2/2}\,dt
    }.
\end{align*}
Thus the claim holds with
\begin{align*}
    C_n:=
    \frac{
    \int_0^\infty t^n e^{-t^2/2}\,dt
    }{
    \int_0^\infty t^{n-1}e^{-t^2/2}\,dt
    }<\infty .
\end{align*}
\end{proof}
\end{document}